\documentclass{article}
\usepackage{lmodern}
\usepackage[T1]{fontenc}
\usepackage[utf8]{inputenc}
\usepackage[all,cmtip]{xy}
\usepackage[hidelinks]{hyperref}
\usepackage{graphicx}
\usepackage{float}
\usepackage{subfig}
\usepackage{amsmath,amsthm,amssymb,mathtools}
\usepackage{mathrsfs}
\usepackage[mathcal]{euscript}

\usepackage{enumitem}

\usepackage{tikz}
\usetikzlibrary{shapes.geometric,arrows,decorations.pathmorphing,backgrounds,positioning,fit,petri}

\usepackage{arydshln}
\usepackage{todonotes}

\usepackage{algorithm}
\usepackage{setspace}
\usepackage{algpseudocode}
\usepackage{bm}

\theoremstyle{definition}
\newtheorem{theorem}{Theorem}
\newtheorem{corollary}[theorem]{Corollary}
\newtheorem{lemma}[theorem]{Lemma}
\newtheorem{prop}[theorem]{Proposition}
\newtheorem{defn}{Definition}
\newtheorem{remark}{Remark}
\newtheorem{example}{Example}
\newtheorem{assumption}{Assumption}

\newcommand {\bigO}{\mathcal{O}}
\newcommand {\St}{\mathcal{S}}
\newcommand {\Obs}{\mathcal{O}}

\newcommand {\A}{\mathcal{A}}
\newcommand{\D}{\mathcal{D}}
\newcommand {\Hist}{\mathcal{H}}
\newcommand {\Dist}{\Delta}
\newcommand {\Reward}{\mathcal{R}}

\newcommand {\R}{\mathbb{R}}
\newcommand {\E}{\mathbb{E}}
\newcommand {\N}{\mathbb{N}}
\newcommand {\Alt}{\mathbb{A}}

\newcommand{\overbar}[1]{\mathbf{#1}}
\newcommand{\ite}[3]{\mathit{if}\;#1\;\mathit{then}\;#2\;\mathit{else}\;#3}

\newcommand{\nocontentsline}[3]{}
\let\origcontentsline\addcontentsline
\newcommand\stoptoc{\let\addcontentsline\nocontentsline}
\newcommand\resumetoc{\let\addcontentsline\origcontentsline}

\title{The Problem of Social Cost in Multi-Agent General Reinforcement Learning: \\ Survey and Synthesis}
\author{Kee Siong Ng \and Samuel Yang-Zhao \and Timothy Cadogan-Cowper \vspace{1em} \\
 The Australian National University
 }

\begin{document}

\maketitle

\begin{abstract}
The AI safety literature is full of examples of powerful AI agents that, in blindly pursuing a specific and usually narrow objective, ends up with unacceptable and even catastrophic collateral damage to others. In this paper, we consider the problem of social harms that can result from actions taken by learning and utility-maximising agents in a multi-agent environment. The problem of measuring social harms or impacts in such multi-agent settings, especially when the agents are artificial generally intelligent (AGI) agents, was listed as an open problem in Everitt et al, 2018. We attempt a partial answer to that open problem in the form of market-based mechanisms to quantify and control the cost of such social harms. The proposed setup captures many well-studied special cases and is more general than existing formulations of multi-agent reinforcement learning with mechanism design in two ways: (i) the underlying environment is a history-based general reinforcement learning environment like in AIXI; (ii) the reinforcement-learning agents participating in the environment can have different learning strategies and planning horizons. To demonstrate the practicality of the proposed setup, we survey some key classes of learning algorithms and present a few applications, including a discussion of the Paperclips problem and pollution control with a cap-and-trade system.
\end{abstract}

\newpage
\tableofcontents \newpage

\section{Introduction}
\label{sec:introduction}

The AI safety literature is full of examples of powerful AI agents that, in blindly pursuing a specific and usually narrow objective, ends up with unacceptable collateral damage to others, including destroying humankind and the world in extreme cases; see, for example, \cite{russell19, Bostrom14}.
Many of these examples are effectively variations of the tragedy of the commons phenomenon, which has been studied extensively in economics \cite{hardin68, ostrom90, cornes1996theory}. 
Tragedy of the commons typically occur because of an externality that arises when a utility-maximising economic agent does not pay an appropriate cost when making a decision to take an action that provides private benefit but incurs social harm to others, in particular those that are not a party to the decision-making process.
Pollution, overfishing, traffic are all classic examples of multi-agent economic systems that exhibit externalities.

In this paper, we consider the problem of social harms that can result from actions taken by learning and utility-maximising agents in a multi-agent environment.
The problem of measuring social harms in such multi-agent settings, especially when the agents are artificial generally intelligent (AGI) agents, was listed as an open problem in \cite{everittLH18}.
We provide a partial answer to that open problem in the form of market-based mechanisms to quantify and control the cost of such social harms in \S~\ref{sec:action cost}.
The key proposal is the control protocol and agent valuation functions described in \S~\ref{subsec:general case}, which captures many existing and well-studied special cases. 
Our proposed setup is more general than existing formulations of multi-agent reinforcement learning with mechanism design in two ways: (i) the underlying environment is a history-based general reinforcement learning environment like in AIXI \cite{hutter2024introduction}; (ii) the reinforcement-learning agents participating in the environment can have different horizons and algorithms.

To demonstrate the practicality of the proposed setup, we survey some learning algorithms in \S~\ref{sec:learning and planning}, including a description of a Bayesian reinforcement learning agent \cite{yang2024dynamic} that provides the current best direct approximation of AIXI and the conditions under which a collection of such agents will converge to a Nash equilibrium. 
A few applications, including the Paperclip problem and a simple cap-and-trade carbon trading scheme, are discussed in \S~\ref{sec:applications}.

As the background literature is rich, both in breadth and depth, we have taken the liberty to take an expository approach in writing the paper and will introduce key topics and concepts as they are required, 
starting with General Reinforcement Learning in \S~\ref{sec:grl} and Mechanism Design in \S~\ref{sec:mechanism design}.
Readers familiar with these topics can skip the two sections without issue.

\section{General Reinforcement Learning}
\label{sec:grl}

\subsection{Single Agent Setting}\label{subsec:single agent}
We consider finite action, observation and reward spaces denoted by $\A, \Obs, \Reward$ respectively. 
The agent interacts with the environment in cycles: at any time, the agent chooses an action from $\A$ and the environment returns an observation and reward from $\Obs$ and $\Reward$. 
Frequently we will be considering observations and rewards together, and will denote $x \in \Obs \times \Reward$ as a percept $x$ from the percept space $\Obs \times \Reward$. We will denote a string $x_1x_2\ldots x_n$ of length $n$ by $x_{1:n}$ and its length $n-1$ prefix as $x_{<n}$. An action, observation and reward from the same time step will be denoted $aor_t$. After interacting for $n$ cycles, the interaction string $a_1o_1r_1\ldots a_no_nr_n$ (denoted $aor_{1:n}$ from here on) is generated. 
We define the history space $\Hist$ to be an interaction string of any length.

\begin{defn}
	A history $h$ is an element of the space $\Hist \coloneqq (\A \times \Obs \times \Reward)^* $. 
    The history at time $t$ is denoted $h_t = aor_{1:t}$.
\end{defn}

The set of probability distributions over a (finite) set $X$ is denoted $\Dist(X)$.
An environment is a process which generates the percepts given actions. It is defined to be a sequence of probability distributions over percept sequences conditioned on the actions taken by the agent.

\begin{defn}
	\label{defn:Env}
	An environment $\rho$ is a sequence of probability distributions $\{\rho_0, \rho_1, \rho_2, \ldots \}$, where $\rho_n: \A^{n} \to \Dist((\Obs \times \Reward)^n)$, that satisfies
	\begin{align}
		\forall a_{1:n}~ \forall or_{<n} \; \rho_{n-1}(or_{<n} \,|\, a_{<n}) = \sum_{or \in \Obs \times \Reward} \rho_n(or_{1:n} \,|\, a_{1:n}). \label{eqn_chrono}
	\end{align}
	In the base case, we have $\rho_0 (\epsilon \,|\, \epsilon) = 1$. 
\end{defn}
Equation (\ref{eqn_chrono}) captures the natural constraint that actions in the future do not affect past percepts and is known as the chronological condition \cite{Hutter:04uaibook}. 
We will drop the subscript on $\rho_n$ when the context is clear.

The predictive probability of the next percept given history and a current action is given by
 \begin{align*}
	\rho(or_n \,|\, aor_{<n}, a_n) = \rho(or_n \,|\, h_{n-1} a_n) \coloneqq \frac{\rho(or_{1:n} \,|\, a_{1:n})}{\rho(or_{<n} \,|\, a_{<n})}
 \end{align*}
for all $aor_{1:n}$ such that $\rho(or_{<n} \,|\, a_{<n}) > 0$. 
This allows us to write 
\[ \rho( or_{1:n} \,|\, a_{1:n}) = \rho( or_1 \,|\, a_1) \rho(or_2 \,|\, aor_1a_2) \ldots \rho(or_n \,|\, aor_{<n}a_n). \]

The general reinforcement learning problem is for the agent to learn a \textit{policy} $\pi: \Hist \to \Dist(A)$ mapping histories to a distribution on possible actions that will allow it to maximise its expected cumulative future rewards.  
%
%
\begin{defn}    
Given an environment $\mu$, at time $t$, given history $h_{t-1}$ and an action $a_{t}$, the expected future cumulative rewards up to finite horizon $m \in \N$ is given by the value function 
\begin{multline}
	V_{t,m}^{\mu,*}(h_{t-1}, a_{t}) = \sum_{or_{t}} \mu(or_{t} \,|\, h_{t-1} a_{t}) \max_{a_{t+1}} \sum_{or_{t+1}} \mu(or_{t+1} \,|\, h_{t-1} aor_{t} a_{t+1}) \,\cdots \\ \max_{a_{t+m}} \sum_{or_{t+m}} \mu(or_{t+m} \,|\, h_{t-1} aor_{t+1: t+m-1} a_{t+m}) \left[  \sum_{i=t}^{t+m} r_i  \right], \label{eq:V single agent}
\end{multline}
which can also be written in this recursive form
\begin{equation}\label{eq:Bellman single agent}
 V_{t,m}^{\mu,*}(h_{t-1}, a_{t}) = \sum_{or_{t}} \mu( or_{t} \,|\, h_{t-1}a_{t}) \left[ r_{t} + \max_{a_{t+1}} V_{t+1,m}^{\mu,*}(h_{t-1} aor_{t}, a_{t+1}) \right], 
\end{equation}
where $V_{m+1,m}^{\mu,*}(\cdot,\cdot) = 0$.
\end{defn}

If the environment $\mu$ is known, the optimal action $a_{t}^*$ to take at time $t$ is given by
\begin{align*}
a_{t}^* = \arg\max_{a_{t}} V^{\mu,*}_{t,m} (h_{t-1}, a_{t}).
\end{align*}
In practice, $\mu$ is of course unknown and needs to be learned from data and background knowledge.
The AIXI agent \cite{Hutter:04uaibook} is a mathematical solution to the general reinforcement learning, obtained by estimating the unknown environment $\mu$ in (\ref{eq:V single agent}) using Solomonoff Induction \cite{solomonoff1964formal-part1}. 
At time $t$, the AIXI agent chooses action $a^*_t$ according to
\begin{align}
	a^*_t = \arg\max_{a_t} \sum_{or_t} \ldots \max_{a_{t+m}} \sum_{or_{t+m}} \left[ \sum_{j=t}^{t+m} r_j \right] \sum_{\rho \in \mathcal{M}_U} 2^{-K(\rho)} \rho(or_{1:{t+m}} \,|\, a_{1:t+m}), \label{eqn_K_AIXI}
\end{align}
where $m \in \N$ is a finite lookahead horizon, $\mathcal{M}_U$ is the set of all enumerable chronological semimeasures \cite{Hutter:04uaibook}, $\rho(or_{1:{t+m}} | a_{1:t+m})$ is the probability of observing $or_{1:t+m}$ given the action sequence $a_{1:t+m}$, and $K(\rho)$ denotes the Kolmogorov complexity \cite{LV-kolmogorov} of $\rho$. 
The performance of AIXI relies heavily on the next result.
 
\begin{defn} 
	Given a countable model class $\mathcal{M} \coloneqq \{ \rho_1, \rho_2, \ldots \}$ and a prior weight $w^\rho_0 > 0$ for each $\rho \in \mathcal{M}$ such that $\sum_{\rho \in \mathcal{M}} w^\rho_0 = 1$, the \textit{Bayesian mixture model} with respect to $\mathcal{M}$ is given by $\xi_{\mathcal{M}}(or_{1:n} | a_{1:n}) = \sum_{\rho \in \mathcal{M}} w^\rho_0 \rho(or_{1:n} | a_{1:n})$. 
\end{defn}

A Bayesian mixture model enjoys the property that it converges rapidly to the true environment if there exists a `good' model in the model class. 

\begin{theorem}\cite{Hutter:04uaibook}
	\label{thm:MEMConvergence}
	Let $\mu$ be the true environment and $\xi$ be the Bayesian mixture model over a model class $\mathcal{M}$. 
	For all $n \in \N$ and for all $a_{1:n}$,
	\begin{multline}
		\sum_{j=1}^{n} \sum_{or_{1:j}} \mu(or_{<j} | a_{<j}) ( \mu(or_j | aor_{<j} a_j) - \xi(or_j | aor_{<j}a_j) )^2  \\ 
  \leq \min_{\rho \in \mathcal{M}} \left\{ \ln \frac{1}{w^\rho_0} + D_n(\mu || \rho) \right\}, \label{eqn:MEMBound}
	\end{multline}
	where $D_n(\mu || \rho)$ is the KL divergence of $\mu(\cdot | a_{1:n})$ and $\rho(\cdot | a_{1:n})$ defined by \[ D_n(\mu || \rho) \coloneqq \sum_{or_{1:n}} \mu(or_{1:n} | a_{1:n}) \ln \frac{\mu(or_{1:n} | a_{1:n})}{\rho(or_{1:n} | a_{1:n})}. \] 
\end{theorem}

To see the rapid convergence of $\xi$ to $\mu$, 
take the limit $n \to \infty$ on the l.h.s of (\ref{eqn:MEMBound}) and observe that in the case where $\min_{\rho \in \mathcal{M}} \sup_{n} D_{n}(\mu || \rho)$ is bounded, the l.h.s. can only be finite if $\xi(or_k | aor_{<k} a_k)$ converges sufficiently fast to $\mu(or_k | aor_{<k}a_k)$.


AIXI is known to be incomputable.
In practice, we will consider Bayesian reinforcement learning agents \cite{ghavamzadeh2015bayesian} that make use of Bayesian mixture models of different kinds and approximate the expectimax operation in (\ref{eqn_K_AIXI}) with algorithms like monte-carlo tree search \cite{ks06} or other reinforcement learning algorithms; examples of such agents include approximations of AIXI like \cite{veness09, yang2022direct, yang2024dynamic}.
These are all model-based techniques; model-free techniques like Temporal Difference learning \cite{Sutton18} that exploits the functional form of (\ref{eq:Bellman single agent}) and universal function approximators like deep neural networks can also be considered.
Indeed, there is an alternative formulation of a universal Bayesian agent called Self-AIXI \cite{catt2023self} that uses a Bayesian mixture over policies to self-predict and maximise over the agent's actions in place of expectimax-style planning.

\subsection{Multi-Agent Setting}\label{subsec:multi agent}

In the multi-agent setup, we assume there are $k > 1$ agents, each with its own action and observation spaces $\A_i$ and $\Obs_i$, $i \in [1\ldots k]$.
At time $t$, the $k$ agents take a joint action \[ \overbar{a_t} = (a_{t,1}, \ldots, a_{t,k}) \in \A_1 \times \cdots \times \A_k = \overbar{\A} \] 
and receive a joint percept 
\[ \overbar{or_t} = (or_{t,1}, \ldots, or_{t,k}) \in (\Obs_1 \times \Reward) \times \cdots \times (\Obs_k \times \Reward) = \overbar{\Obs \times \Reward}. \]
The joint history up to time $t$ is denoted $\overbar{h_t} = \overbar{aor_{1:t}} = \overbar{a_1}\overbar{or_1}\overbar{a_2}\overbar{or_2} \ldots \overbar{a_t}\overbar{or_t}$.

\begin{defn}
	\label{defn:multi-agent env}
	A multi-agent environment $\varrho$ is a sequence of probability distributions 
    $\{ \varrho_0, \varrho_1, \varrho_2, \ldots \}$, where $\varrho_n: (\overbar{\A})^{n} \to \Dist (\overbar{\Obs \times \Reward})^n$, that satisfies
	\begin{align}
		\forall \overbar{a_{1:n}}~ \forall \overbar{or_{<n}} \; \varrho_{n-1}(\overbar{or_{<n}} \,|\, \overbar{a_{<n}}) = \sum_{\overbar{or_n} \in \overbar{\Obs \times \Reward}} \varrho_n( \overbar{or_{1:n}} \,|\, \overbar{a_{1:n}}). 
	\end{align}
	In the base case, we have $\varrho_0 (\epsilon \,|\, \epsilon) = 1$. 
\end{defn}

Multi-agent environments can have different (non-exclusive) properties, some of which are listed here: 
\begin{enumerate}\itemsep1mm\parskip0mm
    \item Mutually exclusive actions, where only one of the actions in $\overbar{a_t}$ chosen by the agents can be executed by the environment at time $t$ -- the default is that the actions are not mutually exclusive;
    \item Zero-sum rewards, where the agents' rewards sum to 0 at every time step so they are competing against each other, like in \cite{littman1994markov};
    \item Identical rewards, where the agents get the exact same rewards at every time step so they are playing cooperatively with each other;
    \item Mixed-sum rewards, which cover all the settings that combine elements of both cooperation and competition; 
    \item Existence of a common resource pool, which can be represented by an `agent' with null action space and whose reward is a function of other agents' consumption of the resource pool.
\end{enumerate}
Comprehensive surveys of formalisations and some key challenges and results in a few of these topics can be found in \cite{shoham2008multiagent, zhang2021multi}.

Each agent's goal in a multi-agent environment is to learn the optimal policy to achieve its own maximum expected cumulative future rewards.
The celebrated result of \cite{kalai1993rational} shows that a group of agents that each (i) uses Bayesian mixture and updating to keep track of other agents' strategies, and (ii) produces a best response policy to the mixture of strategy profiles of the other agents, will converge to an $\epsilon$-Nash equilibrium in repeated games as long as there is a ``grain of truth'' in their beliefs, i.e. each possible opponent strategy is assigned a non-zero probability.
The grain-of-truth condition is not satisfied for a group of AIXI agents because each AIXI agent is not computable, and a Bayesian mixture over such other agents is thus also not computable and therefore not in the model class (of all computable functions).
A technical and general solution to the grain-of-truth problem that significantly extends the result of  \cite{kalai1993rational} to general reinforcement learning is in \cite{leike2016formal}. 
The solution uses Reflective Oracles \cite{fallenstein2015reflective} and a variant of AIXI that uses Thompson sampling to pick policies \cite{leike2016thompson}.
We will consider primarily the behaviour of a collection of (computable) Bayesian reinforcement learning agents in this paper, in both cooperative and competitive multi-agent systems.



\section{Mechanism Design}
\label{sec:mechanism design}

\subsection{Tragedy of the Commons} 
\label{sec:tragedy}

The key to solving tragedy of the commons issues is to work out a way to `internalise' the externality in the design of the multi-agent economic system of interest.
There are two primary approaches: price regulation through a central authority, and a market-based cap-and-trade system.
The former is sometimes referred to as Pigouvian tax after \cite{pigou02}, and it requires a central authority to (i) have enough information to quite accurately determine the unit price of the externality or social harm; and (ii) enforce its payment by agents that cause the externality, thereby internalising it.
In contrast, the cap-and-trade system is motivated by the idea of Coasean bargaining \cite{coase60}, whereby the maximum amount of the externality or social harm allowed is capped through the issuance of a fixed number of permits, each of which allows an agent to produce a unit of externality, and the agents are allowed to determine for themselves whether to use their permits to produce externality, or trade the permits among themselves for profit.
The idea is that the cap-and-trade system will allow the agents that are most efficient in generating private benefit while minimising social harm to win because they can afford to pay a higher price for the permits.
Indeed, the Coase `Theorem' says that as long as the permits are completely allocated and there is no transaction cost involved in trading, then the agents will collectively end up with a Pareto efficient solution.
So a market made up of utility-maximising agents, under the right conditions, is capable of determining the right price for the externality; there is no need for an informative and powerful central authority to set the price.

In the following sections, we will look at some concrete protocols from the field of Mechanism Design for implementing Coasean bargaining and trading in multi-agent environments.
In keeping with the intended spirit of \cite{coase2012firm}, we will largely avoid the term externality from here onwards and favour, instead, the term `social harm'.

\subsection{The VCG Mechanism}\label{subsec:vcg}

In a multi-agent environment, the different agents participating in it can be given different goals and preferences, either competing or cooperative, and the algorithms behind those agents can exhibit different behaviour, including differing abilities in learning and planning for the long term.
Mechanism design \cite{nisan2007} is the study of protocols that can take the usually dispersed and private information and preferences of multiple agents and aggregating them into an appropriate social choice, usually a decision among alternatives, that maximises the welfare of all involved.

Let $\Alt$ be a set of alternatives for a set of $k$ agents.
The preference of agent $i$ is given by a valuation function $v_i : \Alt \to \R$, where $v_i(a)$ denotes the value that agent $i$ assigns to alternative $a$ being chosen.
Here, $v_i \in V_i$, where $V_i \subseteq \R^{\Alt}$ is the set of possible valuation functions for agent $i$. 
We will use the notation $V_{-i} = V_1 \times \cdots \times V_{i-1} \times V_{i+1} \times \cdots V_k$ in the following.

\begin{defn}\label{def:mechanism}
A 
mechanism is a tuple $(f,p_1,\ldots,p_k)$ made up of a social choice function $f : V_1 \times \cdots \times V_k \to \Alt$ and payment functions $p_1,\ldots,p_k$, where $p_i : V_1 \times \cdots \times V_k \to \R$ is the amount that agent $i$ pays to the mechanism.    
\end{defn}
\noindent Given a mechanism $(f,p_1,\ldots,p_k)$ and $k$ agents with value functions $v_1,\ldots,v_k$, the utility of agent $i$ from participating in the mechanism is given by 
\begin{equation}\label{eq:individual utility} 
 u_i(v_1,\ldots,v_k) := v_i(f(v_1,\ldots,v_k)) - p_i(v_1,\ldots,v_k). 
\end{equation}

\begin{defn}\label{def:incentive compatible}
A mechanism $(f,p_1,\ldots,p_k)$ is called incentive compatible if for every agent $i$ with valuation function $v_i \in V_i$, for every $v_i' \in V_i$, and every $v_{-i} \in V_{-i}$, we have
\begin{equation}\label{eq:incentive compatible} 
v_i(f(v_i,v_{-i})) - p_i(v_i,v_{-i}) \geq v_i(f(v_i',v_{-i})) - p_i(v_i',v_{-i}). \end{equation}
\end{defn}
Thus, in an incentive compatible mechanism, each agent $i$ would maximise its utility by being truthful in revealing its valuation function $v_i$ to the mechanism, rather than needing to worry about obtaining an advantage by presenting a possibly false / misleading $v_i'$.

\begin{defn}\label{def:individually rational}
A mechanism $(f,p_1,\ldots,p_k)$ is individually rational if for every agent $i$ with valuation function $v_i \in V_i$ and every $v_{-i} \in V_{-i}$, we have
\begin{equation}\label{eq:IR} v_i(f(v_i,v_{-i})) - p_i(v_i,v_{-i}) \geq 0. \end{equation}
\end{defn}
In other words, the utility of each agent is always non-negative, assuming the agent reports truthfully.

Definitions~\ref{def:mechanism}, \ref{def:incentive compatible} and \ref{def:individually rational} can be generalised to allow the social choice function $f$ and the payment functions $p_i$'s to be randomised functions, in which case we will work with the expectation version of (\ref{eq:individual utility}),  (\ref{eq:incentive compatible}) and (\ref{eq:IR}).

\begin{defn}\label{defn:vcg mechanism}
A mechanism $(f, p_1, \ldots, p_k)$ is called a Vickrey-Clarke-Groves (VCG) mechanism if
\begin{enumerate}
    \item $f(v_1,\ldots,v_k) \in \arg \max_{a \in \Alt} \sum_i v_i(a)$; that is the social choice function $f$ maximises the social welfare, and
    \item there exists functions $h_1,\ldots,h_k$, where $h_i : V_{-i} \to \R$, such that for all $v_1 \in V_1, \ldots, v_k \in V_k$, we have \[ p_i(v_1,\ldots,v_k) = h_i(v_{-i}) - \sum_{j \neq i} v_j(f(v_1,\ldots,v_k)). \] 
\end{enumerate}
\end{defn}

Here is a classical result from mechanism design.

\begin{theorem}\label{thm:vcg incentive compatible}
Every Vickrey-Clarke-Groves mechanism is incentive compatible.    
\end{theorem}

What should the $h_i$ functions in VCG mechanisms be?
A good choice is the Clark pivot rule.
\begin{defn}
The Clark pivot payment function for a VCG mechanism is given by $h_i(v_{-i}) := \max_{b\in\Alt} \sum_{j\neq i} v_j(b)$ for agent $i$.    
\end{defn}
Under this choice of $h_i$, the payment for agent $i$ is 
\[ p_i(v_1,\ldots,v_k) = \max_b \sum_{j\neq i} v_j(b) - \sum_{j\neq i} v_j(f(v_1,\ldots,v_k)), \] which is the difference between the collective social welfare of the others with and without $i$'s participation in the system.
So each agent makes the payment that internalises the exact social harm it causes other agents.
The utility of agent $i$ is
\[ u_i(v_1,\ldots,v_k) = \sum_j v_j(f(v_1,\ldots,v_k)) - \max_b \sum_{j \neq i} v_j(b). \]

\begin{theorem}\label{thm:vcg is individually rational}
The VCG mechanism with Clark pivot payment function is individually rational if the agent valuation functions are all non-negative.
\end{theorem}

\begin{example}\label{ex: second price auction}
Consider an auction where $\Alt = \{ 1,\ldots,k \}$ -- so one and only one of the agents win -- and where, for agent $i$, $v_i(i) = p_i$ and $\forall j\neq i, v_i(j) = 0$.
Vickrey's Second Price auction, in which each agent $i$ bids the highest price $p_i$ it is willing to pay for the auction item and where the winner $i^* = \arg \max_j p_j$ pays the second highest bid price $p^* = \max_{j\neq i^*} p_j$ and every one else pays $0$ is a VCG mechanism with the Clark pivot payment function. 
\end{example}

\begin{example}
Consider the design of a mechanism to allow two agents, a buyer $B$ and a seller $S$, to engage in bilateral trade for a good owned by the seller.
There are two possible outcomes: no-trade or trade, which we model numerically as $\Alt = \{ 0, 1 \}$.
The buyer values the good at $\theta_B \geq 0$ so its valuation function is
\[ v_B(d) := \ite{d=1}{\theta_B}{0}. \]
The seller values the good at $\theta_S \geq 0$ so its valuation function is 
\[ v_S(d) := \ite{d = 1}{-\theta_S}{0} \]
because the seller loses the good in the case of a trade.
Suppose we use the VCG mechanism with the Clark pivot payment function as the mechanism, we will end up with the social choice 
\begin{align*}
d^* = f(v_B,v_S) &= \arg \max_{d \in \Alt} (v_B(d) + v_S(d)) \\
   &= \arg \max_{d \in \Alt} \,(\ite{d=1}{\theta_B - \theta_S}{0}) \\
   &= \ite{\theta_B - \theta_S > 0}{1}{0},
\end{align*}
which means there is a trade iff the buyer attaches a higher value to the good than the seller.
Here are the payment functions:
\begin{align*}
p_B(v_B,v_S) &= \max_{d \in \Alt} v_S(d) - v_S(d^*) = \ite{\theta_B - \theta_S > 0}{\theta_S}{0} \\
p_S(v_B,v_S) &= \max_{d \in \Alt} v_B(d) - v_B(d^*) = \ite{\theta_B - \theta_S > 0}{0}{\theta_B}.
\end{align*}
So the buyer pays the mechanism $\theta_S$ if there is a trade and the seller pays the mechanism $\theta_B$ if there is no trade.  
The latter is slightly odd and results in the problematic issue of the seller always having negative utility in participating in the mechanism:
\begin{align*}
    u_S(v_B,v_S) &= v_B(d^*) + v_S(d^*) - \max_{d\in\Alt} v_B(d) \\
    &= \ite{ \theta_B - \theta_S > 0}{-\theta_S}{-\theta_B}.
\end{align*}
The problem comes down to the asymmetric position of the buyer and seller and the $p_i(\cdot) \geq 0$ condition enforced by the Clark pivot payment function, where the seller is forced to make a payment to maintain its ownership of the good (no trade), even though the status quo is that the seller already owns the good.
There are at least two solutions.
The first solution is to insist that the buyer and seller pays nothing if there is no trade, so we end up with the constraints 
\begin{align*}
p_B(v_B,v_S) = h_B(v_S) - v_S(0) = 0 \\
p_S(v_B,v_S) = h_S(v_B) - v_B(0) = 0 
\end{align*}
that imply $h_B(v_S) = v_S(0)$ and $h_S(v_B) = v_B(0)$.
In this scenario, when trade happens, we have
\begin{align*}
    p_B(v_B,v_S) &= h_B(v_S) - v_S(1) = \theta_S \\
    p_S(v_B,v_S) &= h_S(v_B) - v_B(1) = -\theta_B,
\end{align*}
which means the buyer pays the mechanism $\theta_S$ and the seller is paid $\theta_B$ by the mechanism, with both agents obtaining utility $\theta_B - \theta_S > 0$.
Note that mechanism ends up having to subsidise the trade.
The second solution is to remove the ownership asymmetry between the two agents, by first making the mechanism pay an amount $\theta \geq \theta_S$ to the seller to transfer the good to the mechanism and then running the VCG mechanism with Clark pivot rule to determine new ownership of the good. 
Under this setup, the valuation function of the buyer stays the same, but the valuation function of the seller becomes 
\[ v_S(d) := \ite{d = 1}{0}{\theta_S}, \]
and we end up with the Vickrey second-price auction setup.
The utility of the buyer stays the same, and that of the seller becomes
\begin{align*}
    u_S(v_B,v_S) &= \theta + (v_S(d^*) + v_B(d^*) - \max_{d \in \Alt} v_B(d) \\
      &= \theta - \theta_B + (\ite{\theta_B - \theta_S > 0}{\theta_B}{\theta_S}) \\
      &= \ite{(\theta_B - \theta_S > 0)}{\theta}{(\theta + \theta_S - \theta_B)} \\
      &\geq 0.
\end{align*}
As with the first solution, the mechanism ends up with a negative value, which is the cost of subsidising the trade.
\end{example}

\subsection{The Exponential VCG Mechanism}\label{subsec:exponential mechanism}
We have shown in \S~\ref{subsec:vcg} that VCG mechanisms are incentive compatible and individually rational, which means agents are incentivised to participate and be truthful.
It turns out that VCG mechanisms can be made privacy-preserving too.
The exponential mechanism \cite{mcsherry2007expmechanism}, a key technique in differential privacy \cite{dwork2014algorithmic}, has been shown in \cite{huang2012exponential} to be a generalisation of the VCG mechanism that is differentially private, incentive compatible and nearly optimal for maximising social welfare.
We now briefly describe this key result and furnish the proofs, which are rather instructive.

\begin{defn}
    \label{defn:DP}
    A randomized algorithm $\mathcal{M}: V_1 \times \cdots \times V_k \to \Alt$ is $(\epsilon, \delta)$-differentially private if for any $v \in V_1 \times \cdots \times V_k$ and for any subset $\Omega \subseteq \Alt$
    \begin{align*}
        P(\mathcal{M}(v) \in \Omega) \leq e^{\epsilon} P(\mathcal{M}(v') \in \Omega) + \delta,
    \end{align*}
    for all $v'$ such that $|v - v'|_1 \leq 1$ (i.e. there exists at most one $i \in [n]$ such that $v_i \neq v'_i$).
\end{defn}

\begin{defn}
Given a quality function $q : V_1 \times \cdots \times V_k \times \Alt \to \R$ and a $v \in V_1 \times \cdots \times V_k$, the Exponential DP Mechanism $\mathcal{M}_{q}^{\epsilon}(v)$ samples and outputs an element $r \in \Alt$ with probability proportional to $\exp(\frac{\epsilon}{2\Delta_q}q(v,r))$, where \[ \Delta_q = \max_{r \in \Alt} \max_{v_1,v_2: ||v_1 - v_2||_1 \leq 1} |q(v_1,r) - q(v_2,r)|. \]
\end{defn}

\begin{theorem}
The Exponential DP Mechanism is $(\epsilon,0)$-differentially private.   
\end{theorem}

\begin{defn}
The Exponential VCG Mechanism is defined by $(\mathcal{M}_{q}^{\epsilon},p_1,\ldots,p_k)$ where
\begin{align*}
 & q(v,r) = \sum_{i} v_i(r) \\
 & p_i(v) = \mathop{\mathbb{E}}_{r \sim \mathcal{M}_q^{\epsilon}(v)} \left[-\sum_{j \neq i} v_j(r) \right] - \frac{2}{\epsilon} H(\mathcal{M}_q^{\epsilon}(v)) + \frac{2}{\epsilon}\ln\left( \sum_{r \in \Alt} \exp\left(\frac{\epsilon}{2} \sum_{j \neq i} v_j(r)\right) \right)
\end{align*}
and $H(\cdot)$ is the Shannon entropy function.
\end{defn}

Note that as $\epsilon$ increases, $\mathcal{M}_q^{\epsilon}$ will sample $r^* = \arg \max_{r \in \Alt} \sum_{i} v_i(r)$ with probability rapidly approaching 1, and the payment function also satisfies the form given in Definition~\ref{defn:vcg mechanism}.
In that sense, the exponential VCG mechanism can be considered a generalisation of the VCG mechanism.

\begin{lemma}
Given $\epsilon \in \R$ and valuation functions $v = v_1,\ldots,v_k$ where each $v_i : \Alt \to [0,1]$, the Gibbs social welfare defined by
\[ \mathbb{E}_{r \sim \xi} \left[ \sum_i v_i(r) \right] + \frac{2}{\epsilon}H(\xi) \]
is maximised when $\xi = \mathcal{M}^{\epsilon}_q(v)$ for $q(v,r) = \sum_i v_i(r)$.
\end{lemma}
\begin{proof}
The first term in the Gibbs social welfare can be rewritten as follows:
\begin{align}
& \sum_{r \in \Alt} \xi(r) q(v,r) \nonumber \\
&= \frac{2}{\epsilon} \sum_{r \in \Alt} \xi(r) \frac{\epsilon}{2} q(v,r) \nonumber \\  
&= \frac{2}{\epsilon} \sum_{r \in \Alt} \xi(r) \ln \left(\exp \left(\frac{\epsilon q(v,r)}{2} \right) \right)  \nonumber \\
&= \frac{2}{\epsilon} \sum_{r \in \Alt} \xi(r) \ln \left(\frac{\exp (\epsilon q(v,r)/2 )}{\sum_{a \in \Alt} \exp (\epsilon q(v,a)/2 )} \right) + \frac{2}{\epsilon} \ln \left( \sum_{a \in \Alt} \exp (\epsilon q(v,a)/2 ) \right). \label{eqn:gibbs}
\end{align}
Adding (\ref{eqn:gibbs}) to the second entropy term of the Gibbs social welfare and noting that $\Delta_q = 1$, we get
\begin{align}
& \mathbb{E}_{r \sim \xi} \left[ \sum_i v_i(r) \right] + \frac{2}{\epsilon}H(\xi)  \nonumber \\
&= \frac{2}{\epsilon} \sum_{r \in \Alt} \xi(r) \ln (\mathcal{M}_q^{\epsilon}(v)(r) ) + \frac{2}{\epsilon} \ln \left( \sum_{a \in \Alt} \exp (\epsilon q(v,a)/2 ) \right) - \frac{2}{\epsilon} \sum_{r \in \Alt} \xi(r) \ln (\xi(r)) \nonumber \\
&= -\frac{2}{\epsilon} D_{\mathit{KL}}(\xi \,||\, \mathcal{M}_q^{\epsilon}(v)) + \frac{2}{\epsilon} \ln \left( \sum_{a \in \Alt} \exp (\epsilon q(v,a)/2 ) \right). \label{eqn:gibbs 2}
\end{align}
By Gibb's inequality, (\ref{eqn:gibbs 2}) is maximised when $\xi = \mathcal{M}_q^{\epsilon}(v)$.
\end{proof}

\begin{theorem}(\cite{huang2012exponential})
The Exponential VCG Mechanism is incentive compatible and individually rational.    
\end{theorem}
\begin{proof}
We first show the incentive compatible property.
Consider an agent $i$ with valuation function $v_i$ and fix the bids $v_{-i}$ of the other agents.
Let 
\[ h(\{ v_j \}_{j} ) = \frac{2}{\epsilon}\ln\left( \sum_{r \in \Alt} \exp\left(\frac{\epsilon}{2} \sum_{j} v_j(r)\right) \right). \]
The expected utility to agent $i$ when bidding $b_i$ is
\begin{align}
  & \mathop{\mathbb{E}}_{r \sim \mathcal{M}_q^{\epsilon}(b_i,v_{-i})} [ v_i(r) ] - p_i(b_i, v_{-i}) \nonumber \\ 
  &= \mathop{\mathbb{E}}_{r \sim \mathcal{M}_q^{\epsilon}(b_i,v_{-i})} [ v_i(r) ] + \mathop{\mathbb{E}}_{r \sim \mathcal{M}_q^{\epsilon}(b_i,v_{-i})} \left[\sum_{j \neq i} v_j(r) \right] + \frac{2}{\epsilon} H(\mathcal{M}_q^{\epsilon}(b_i,v_{-i})) - h(v_{-i}) \nonumber \\
  &= \mathop{\mathbb{E}}_{r \sim \mathcal{M}_q^{\epsilon}(b_i,v_{-i})} \left[\sum_{j} v_j(r) \right] + \frac{2}{\epsilon} H(\mathcal{M}_q^{\epsilon}(b_i,v_{-i})) - h(v_{-i}) \nonumber \\
  &= -\frac{2}{\epsilon} D_{\mathit{KL}}( \mathcal{M}_q^\epsilon(b_i,v_{-i}) \,||\, \mathcal{M}_q^\epsilon(v_i,v_{-i})) 
  + h(v_i,v_{-i}) - h(v_{-i}), \label{eq:exp vcg utility} 
\end{align}
where the last step follows from (\ref{eqn:gibbs 2}) and is maximised when $b_i = v_i$.

To show the individually rational property, note that when $b_i = v_i$, the expression (\ref{eq:exp vcg utility}) reduces to $h(v_i,v_{-i}) - h(v_{-i})$, which is non-negative and equals zero when $v_i$ is the zero function.
\end{proof}

As shown in \cite{huang2012exponential}, it is also advisable to add differential privacy noise to the payment functions given it contains information about all the agents' valuation functions, which may need to be kept private.

\section{The Social Cost of Actions}
\label{sec:action cost}

\subsection{General Case}\label{subsec:general case}
Suppose we have multiple Bayesian reinforcement learning agents operating within an environment.
These agents are concrete realisations of the concept of perfectly rational utility-maximising agents commonly assumed in economics.
We have seen that, in the absence of some control mechanism, such multi-agent environments can exhibit bad equilibrium.
To avoid tragedy of the commons-type issues, we need to impose a cost on each agent's actions, commensurate with the social harm they are causing other agents with that action, and we will see in this section how augmenting a multi-agent environment with, for example, VCG mechanisms can address such issues.

\subsubsection*{Protocol for Controlled Multi-Agent Environment}
Given a multi-agent environment $\phi$ with $k$ agents and a VCG mechanism $M = (f,p_1,\ldots,p_k)$, we denote by $M \triangleright \phi$ the following interaction protocol between the agents and the environment.
Let $\overbar{h_{t-1}}$ be the joint history up till time $t$.
At time $t$, each agent $i$ submits a valuation function $v_{t,i} \in \mathbb{R}^{\A_i}$, which are collectively denoted as $\overbar{v_t} = ( v_{t,1}, v_{t,2}, \ldots, v_{t,k} )$.
We then use the VCG mechanism $M$ to determine the joint action the agents should take to maximise social welfare via
\begin{equation}\label{eq:max welfare} 
\overbar{a^*_t} := f(\overbar{v_t}) = \arg \max_{\overbar{a} \in \Alt} \sum_{i} v_{t,i}(\overbar{a}). 
\end{equation}
A joint percept $\overbar{or_t}$ is then sampled from $\phi( \overbar{or_t} \,|\, \overbar{h_{t-1}}\overbar{a^*_t})$ and
each agent $i$ receives the percept $or_{t,i}$ and is charged the following payment amount by the mechanism $M$:
\begin{equation}\label{eq:payment function} 
p_i(\overbar{v_t}) := \max_{\overbar{a_t}} \sum_{j \neq i} v_{t,j}(\overbar{a_t}) - \sum_{j \neq i} v_{t,j}(\overbar{a_t^*}), 
\end{equation}
which can be thought of as the social cost that agent $i$ incurs from the joint action $\overbar{a_t^*}$. 
The instantaneous utility of agent $i$ at time $t$ is then given by 
\begin{equation}\label{eq:instantaenous utility} 
r_{t,i} - p_i(\overbar{v_t}). 
\end{equation}
The goal of each agent $i$ is to submit a sequence of valuation functions to maximise its cumulative utility \[ \sum_{t} (r_{t,i} - p_i(\overbar{v_t})), \] which is a random variable dependent on $M \triangleright \phi$ and the submitted valuation functions of the other agents.
The total social welfare obtained from each run of the protocol is the sum of all the agents' cumulative utilities: \[ \sum_{i=1}^k \sum_{t} (r_{t,i} - p_i(\overbar{v_t})). \]

It is worth noting that, under the usual VCG mechanism convention, the utility of agent $i$ for the chosen action $\overbar{a_t^*}$ would be $v_{t,i}(\overbar{a_t^*}) - p_i(\overbar{v_t})$, rather than (\ref{eq:instantaenous utility}).
As we will shall see, with the right choice of $v_{t,i}$, the formula $v_{t,i}(\overbar{a_t^*}) - p_i(\overbar{v_t})$ captures the expected cumulative utility for the agent, where the expectation is with respect to the randomness of the underlying environment $\phi$ and possible randomness in the strategies of the other agents.

\subsubsection*{What should the agent valuation functions be?}
A key question in defining the agent valuation function is to determine whether each agent needs to explicitly consider the other agents operating in the same environment.
We will start by looking at a simple scenario. 

\begin{example}\label{ex:protocol cases}
Consider two agents $A_1$ and $A_2$ that are at the entrance to a long narrow tunnel that leads to a treasure that is guarded by a sleeping dragon.
The treasure is valued by $A_1$ at 100 and by $A_2$ at 90. 
Assume further that the tunnel can only fit one agent, and who ever enters the tunnel first will end up claiming the treasure, assuming they do not wake the dragon along the way. 
This is, in a sense, just a simple extension of Vickrey second price auction but with a planning component and possibly uncertain outcomes.
The first point to make is that each agent's horizon has to be sufficiently long to see that there is a treasure at the end of the tunnel; if it takes 100 steps to reach the treasure but an agent can only see 50 steps ahead, then it will value the action of entering the tunnel at 0. 
The agents should also account for the probability of waking the dragon, and thus getting killed, in their valuation functions.
Should the agents take each other's presence into account? Here are three scenarios, ignoring the probability of being killed by the dragon for now:
\begin{itemize}\itemsep1mm\parskip0mm
 \item If the agents are oblivious of each other in forming their valuation functions, then $A_1$ will value the action of entering the tunnel at 100, and $A_2$ will value the same action at 90, in which case the VCG mechanism will make the right social choice of allowing $A_1$ to enter the tunnel, with a payment of 90 and net gain of 10.

 \item Suppose $A_1$ takes $A_2$ into account but $A_2$ is oblivious to $A_1$. Then $A_1$ will simulate the VCG mechanism and come to the conclusion that it should value the action of entering the tunnel at \$10, which comes from its gain of 100 from claiming the treasure, minus the 90 it has to pay. $A_2$, being oblivious of $A_1$, will value the action of entering the tunnel at 90.
 With these two valuation functions, the VCG mechanism will make the incorrect social choice of picking $A_2$ to enter the tunnel, with a payment of 10 and a net gain of 80.

 \item Suppose both agents take each other into account. Then $A_1$ will value the action of entering the tunnel at 10 as before, and $A_2$ will value the same action at 0, which comes from its simulation that the VCG mechanism will always pick $A_1$ to enter the tunnel. With these two valuation functions, the VCG mechanism will make the right social choice of picking $A_1$ to enter the tunnel, with 0 payment and a net gain of 10.
\end{itemize}
The scenarios suggest the agents need to be synchronised in whether they take each other into account in their valuation functions, but we cannot tell from this example whether they should or should not take each other into account, given the first and third scenarios produce the same net outcome.
The first scenario fits better with the standard framing of mechanism design; e.g. in a Vickrey second price auction, there is no need for each bidder to model what the others might bid. 
The third scenario, however, offers arguably some payment efficiency, but we need proper accounting of payments, given valuation functions are used by the VCG mechanism to determine the actual payments.
\end{example}

In Example~\ref{ex:protocol full example alt self-rational} in Appendix~\ref{app:agent valuation function}, we give a simple scenario to show that an interaction protocol $M \triangleright \phi$ can yield suboptimal results if the agents ignore the others in the environment in forming their valuation functions. 
This leads us to the following proposed agent valuation function, under the assumption that all the agents have full knowledge of the underlying environment $\phi$. Once we have established the optimal agent valuation function under full knowledge, we will examine in \S~\ref{sec:learning and planning} how agents with only partial knowledge of the environment can learn approximations of the optimal valuation function from interaction data.

\begin{defn}\label{defn:agent valuation function}
Given full knowledge of the environment $\phi$ and the mechanism $(f,p_1,\ldots,p_k)$, for each agent $i$ with horizon $m_i$ at time $t$ having seen history $\overbar{h_{t-1}}$, the \emph{rational q-function} $q_{t,i}$, \emph{social cost function} $c_{t,i}$, and \emph{valuation function} $v_{t,i}$ are defined inductively as follows
\begin{align}
  q_{t,i}(\overbar{h_{t-1},\overbar{a_t}}) &=
   \begin{cases}
       0 & t > m_i \\
       \sum_{ \overbar{or_t}} \phi(\overbar{or_t} \,|\, \overbar{h_{t-1}a_t}) [r_{t,i} + \overline{q}_{t+1,i}(\overbar{h_{t-1}}\overbar{a_t}\overbar{or_t})]  & t \leq m_i 
   \end{cases}   \label{eq:agent q function} \\
 \overline{q}_{t,i}(\overbar{h_{t-1}}) &= q_{t,i}(\overbar{h_{t-1}}, f( \overbar{v_{t}}))  \label{eq:agent bar q function} \\ 
 c_{t,i}(\overbar{h_{t-1},\overbar{a_t}}) &= 
   \begin{cases}
       0 & t \geq m_i \\
       \sum_{ \overbar{or_t}} \phi(\overbar{or_t} \,|\, \overbar{h_{t-1}}\overbar{a_t}) [ \overline{c}_{t+1,i}(\overbar{h_{t-1}}\overbar{a_t}\overbar{or_t}) ]  & t < m_i 
   \end{cases}   \label{eq:agent cost function} \\
  \overline{c}_{t,i}(\overbar{h_{t-1}}) &= p_i(\overbar{v_t}) + c_{t,i}(\overbar{h_{t-1}}, f( \overbar{v_{t}}))  \label{eq:agent bar cost function} \\
   v_{t,i}(\overbar{h_{t-1}},\overbar{a_t}) &= q_{t,i}(\overbar{h_{t-1}},\overbar{a_t}) - c_{t,i}(\overbar{h_{t-1}},\overbar{a_t}) \label{eq:alt agent valuation fn} \\
   \overline{v}_{t,i}(\overbar{h_{t-1}}) &= \overline{q}_{t,i}(\overbar{h_{t-1}}) - \overline{c}_{t,i}(\overbar{h_{t-1}}), \label{eq:expected total utility}
\end{align}
where $\overbar{v_t} = (v_{t,1}(\overbar{h_{t-1}},\cdot), \ldots, v_{t,k}(\overbar{h_{t-1}},\cdot))$ and $f(\overbar{v_t}) = \arg \max\limits_{\overbar{a}} \sum\limits_{j} v_{t,j}(\overbar{h_{t-1}},\overbar{a})$.
\end{defn}
Note how (\ref{eq:agent q function}) is similar in form to (\ref{eq:Bellman single agent}), but with the maximising action for each agent replaced by the social choice action at every time step.
To gain some intuition on (\ref{eq:agent q function})-(\ref{eq:expected total utility}), consider the tree in Fig~\ref{fig:game tree} for the simple setup where all the agents have horizon $m =2$, the action space consists only of $\{ \overbar{a}^1, \overbar{a}^2 \}$, and the perception space consists only of $\{ \overbar{or}^1, \overbar{or}^2 \}$.
Each node is indexed by the sequence of symbols in the path from the root to the node, starting with $\overbar{h_0} = \epsilon$ at the root node.
There are two types of alternating nodes: decision nodes (in red) and observation nodes (in green).
Attached to 
\begin{itemize}\itemsep1mm\parskip0mm
    \item each terminal decision node at the bottom of the tree labelled by a percept $\overbar{or}$ is  set of reward values $\{ r_i \}_{i=1\ldots k}$;
    \item each non-terminal decision node indexed by $\overbar{h_{t-1}}$ is a set  $\{ \overline{v}_{t,i}(\overbar{h_{t-1}}) \}_{i=1\ldots k}$ of values corresponding to (\ref{eq:expected total utility}), which requires $\overline{q}_{t,i}(\overbar{h_{t-1}})$ and $\overline{c}_{t,i}(\overbar{h_{t-1}})$; 

    \item each observation node indexed by $\overbar{h_{t-1}a_t}$ is a set  $\{ v_{t,i}(\overbar{h_{t-1}},\overbar{a_t}) \}_{i=1\ldots k}$ of values corresponding to (\ref{eq:alt agent valuation fn}), which requires $q_{t,i}(\overbar{h_{t-1}},\overbar{a_t})$ and $c_{t,i}(\overbar{h_{t-1}},\overbar{a_t})$.
\end{itemize}
Each non-terminal decision node indexed by $\overbar{h_t}$ has a social-welfare maximising action $\arg \max_{\overbar{a}} \sum_{i} v_{t+1,i}(\overbar{h_t},\overbar{a_t})$, which is indicated by a thick arrow.
In the diagram, we assume each of the agents has the same horizon. 
In general, this is not the case and some decision nodes can play the role of both terminal and non-terminal nodes, depending on each agent's horizon.

\begin{figure} 
\begin{center}
\begin{tikzpicture}[
roundnode/.style={circle, draw=green!60, fill=green!5},
squarednode/.style={rectangle, draw=red!60, fill=red!5, minimum width=4mm, text centered, text width=0.6cm},
trinode/.style={isosceles triangle, rotate=90, draw=green!60, fill=green!5}
]
\node[squarednode] (par) {$\epsilon$};
\node[roundnode] (a1) [below left=2.4cm of par] {$a^1$};
\node[roundnode] (a2) [below right=2.4cm of par] {$a^2$};
\node[squarednode] (lor1) [below left=1.5cm of a1] {$or^1$};
\node[squarednode] (lor2) [below =0.95cm of a1] {$or^2$};
\node[trinode] (lor1b) [below = 0.75cm of lor1] {$\cdots$};
\node[trinode] (lor2b) [below = 0.75cm of lor2] {$\cdots$};
\node[squarednode] (ror1) [below left=1.5cm of a2] {$or^1$};
\node[squarednode] (ror2) [below right=1.5cm of a2] {$or^2$};
\node[roundnode] (1a1) [below left=0.8cm of ror1] {$a^1$};
\node[roundnode] (1a2) [below right=0.8cm of ror1] {$a^2$};
\node[roundnode] (2a1) [below left=0.8cm of ror2] {$a^1$};
\node[roundnode] (2a2) [below right=0.8cm of ror2] {$a^2$};
\node[squarednode] (f1or1) [below left=1.3cm of 1a1] {$or^1$};
\node[squarednode] (f1or2) [below =1.6cm of 1a1] {$or^2$};
\node[squarednode] (f2or1) [below left=1.3cm of 1a2] {$or^1$};
\node[squarednode] (f2or2) [below =1.6cm of 1a2] {$or^2$};
\node[squarednode] (f3or1) [below right=1.3cm of 2a1] {$or^1$};
\node[squarednode] (f3or2) [below =1.6cm of 2a1] {$or^2$};
\node[squarednode] (f4or1) [below right=1.3cm of 2a2] {$or^1$};
\node[squarednode] (f4or2) [below =1.6cm of 2a2] {$or^2$};

\draw[->] (par) -- (a1);
\draw[ultra thick, ->] (par) -- (a2);
\draw[->] (a1) -- (lor1);
\draw[->] (a1) -- (lor2);
\draw[->] (a2) -- (ror1);
\draw[->] (a2) -- (ror2);
\draw[ultra thick, ->] (ror1) -- (1a1);
\draw[->] (ror1) -- (1a2);
\draw[->] (ror2) -- (2a1);
\draw[ultra thick, ->] (ror2) -- (2a2);
\draw[->] (1a1) -- (f1or1);
\draw[->] (1a1) -- (f1or2);
\draw[->] (1a2) -- (f2or1);
\draw[->] (1a2) -- (f2or2);
\draw[->] (2a1) -- (f3or1);
\draw[->] (2a1) -- (f3or2);
\draw[->] (2a2) -- (f4or1);
\draw[->] (2a2) -- (f4or2);
\end{tikzpicture}
\caption{A horizon-2 game tree with small action and percept spaces}\label{fig:game tree}
\end{center}
\end{figure}
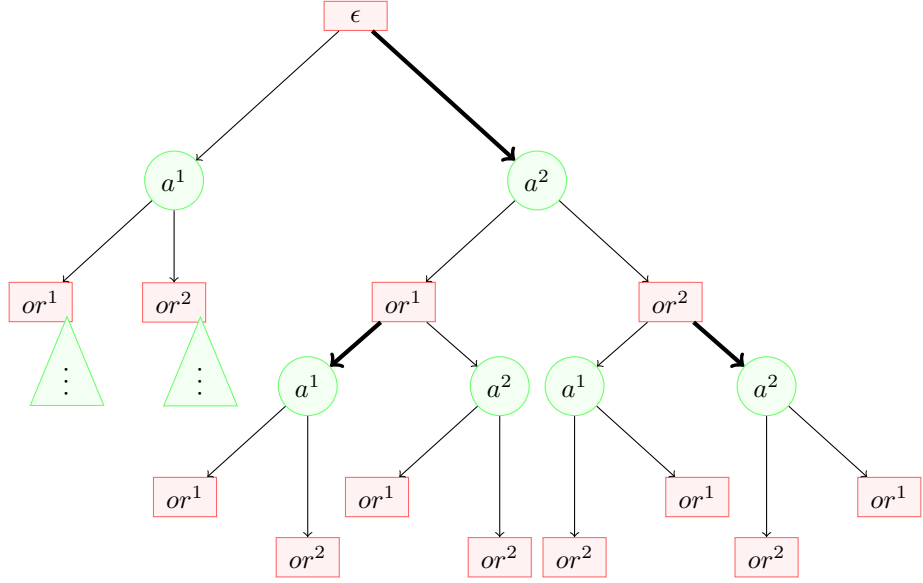

Observe that the quantity $\overline{v}_{1,i}(\epsilon) = \overline{q}_{1,i}(\epsilon) - \overline{c}_{1,i}(\epsilon)$ is the socially optimal expected value of agent $i$'s cumulative utility $\sum_{t} r_{t,i} - p_i(\overbar{q_t})$. (See Appendix~\ref{app:agent valuation function} for details.)
In general, at time $t$ having seen $\overbar{h_{t-1}}$, each agent $i$'s future expected cumulative utility from $t$ onwards is given by $\overline{v}_{t,i}(\overbar{h_{t-1}})$ if all the agents submit their rational valuation functions to the protocol from time $t$ onwards, in which case 
the socially optimal joint action that maximises expected total utility for all the agents, conditioned on the history so far, is taken at every time step. 
We next show that all the agents are incentivised to submit their true rational valuation functions.

\begin{defn}
Let $\phi$ be the environment and $(f,p_1,\ldots,p_k)$ the mechanism.
Given history $\overbar{h_{t-1}}$ at time $t$, 
the agents' submitted valuation functions $\overbar{\tilde{v}_t}$,
social choice action $\overbar{a_t} := f(\overbar{\tilde{v}_t})$, and a percept $\overbar{or_t}$ sampled from $\phi(\cdot \,|\, \overbar{h_{t-1}}\overbar{a_t})$, we define the \emph{realisable cumulative utility} at time $t$ for agent $i$ to be
\begin{equation}\label{eq:realisable return} 
r_{t,i} - p_i(\overbar{\tilde{v}_t}) + \overline{q}_{t+1,i}(\overbar{h_{t-1}}\overbar{a_t}\overbar{or_t}) - \overline{c}_{t+1,i}(\overbar{h_{t-1}}\overbar{a_t}\overbar{or_t}), 
\end{equation}
which is the sum of the agent's instantaneous utility at time $t$ and its expected future cumulative utility from time $t+1$ onwards. 
\end{defn}

The next result is an adaptation of Theorem~\ref{thm:vcg incentive compatible} and shows that, 
assuming all the agents have rational valuation functions, an agent can maximise its expected cumulative utility by always submitting its true rational valuation function if all the other agents also submit their true rational valuation functions, a property called Bayes-Nash Incentive Compatibility, a weaker form of Definition~\ref{def:incentive compatible}.

\begin{corollary}\label{cor:bayes-nash ic}
Let $M \triangleright \phi$ be the interaction protocol.
Suppose each agent's true valuation function is as defined in (\ref{eq:alt agent valuation fn}). 
Then $M \triangleright \phi$ is Bayes-Nash Incentive Compatible with respect to each agent's realisable cumulative utility at every time step.  
\end{corollary}   
\begin{proof}
Let $\overbar{h_{t-1}}$ be the history at time $t$.
Fix an agent $i$ and suppose all the other agents submit their true rational valuation functions $v_{t,-i}$. 
Agent $i$ can choose to submit its true valuation function $v_{t,i}$ or some other arbitrary function $\widetilde{v}_{t,i}$.
If it submits $v_{t,i}$, then the protocol picks action $\overbar{a_t} := \arg \max_{\overbar{a}} \sum_j v_{t,j}(\overbar{h_{t-1}}, \overbar{a})$ and agent $i$'s expected realisable cumulative utility from the protocol is
\begin{align}
& 
  \mathbb{E}_{\overbar{or_t}} \biggl[ r_{t,i} - p_i(v_{t,i},v_{t,-i}) + \overline{q}_{t+1,i}(\overbar{h_{t-1}}\overbar{a_t}\overbar{or_t}) - \overline{c}_{t+1,i}(\overbar{h_{t-1}a_t or_t}) \biggr] \nonumber \\ 
= & 
   \mathbb{E}_{\overbar{or_t}} \biggl[ r_{t,i}  + \overline{q}_{t+1,i}(\overbar{h_{t-1}}\overbar{a_t}\overbar{or_t}) \biggr] - \mathbb{E}_{\overbar{or_t}} [ \overline{c}_{t+1,i}(\overbar{h_{t-1}a_t or_t}) ]  \nonumber \\ 
 & \hspace{9em} - \max_{\overbar{b_t}} \sum_{j \neq i} v_{t,j}(\overbar{h_{t-1}}, \overbar{b_t}) + \sum_{j \neq i} v_{t,j}(\overbar{h_{t-1}},\overbar{a_t}) \nonumber \\
= & \sum_{j} v_{t,j}(\overbar{h_{t-1}},\overbar{a_t}) - \max_{\overbar{b_t}} \sum_{j \neq i} v_{t,j}(\overbar{h_{t-1}}, \overbar{b_t}), \label{eq:truth utility alt}
\end{align}
where (\ref{eq:truth utility alt}) follows from (\ref{eq:alt agent valuation fn}). 
(The argument holds for all $t$; for the $t \geq m_i$ case, the $\overline{q}_{t+1,i}$ and $\overline{c}_{t+1,i}$ terms are both zero and $v_{t,i}(\overbar{h_{t-1}},\overbar{a_t}) = \mathbb{E}_{\overbar{or_t}} r_{t,i}$.)
If agent $i$ submits $\widetilde{v}_{t,i}$, we can similarly show that agent $i$'s expected realisable cumulative utility is
\begin{equation}\label{eq:lie utility alt}
  \sum_{j} v_{t,j}(\overbar{h_{t-1}},\overbar{\tilde{a}_t}) - \max_{\overbar{b_t}} \sum_{j \neq i} v_{t,j}(\overbar{h_{t-1}}, \overbar{b_t}),
\end{equation}
where $\overbar{\tilde{a}_t} := \arg \max_{\overbar{a}} \bigl[ \widetilde{v}_{t,i}(\overbar{a}) + \sum_{j \neq i} v_{t,j}(\overbar{h_{t-1}},\overbar{a}) \bigr]$.
Clearly, (\ref{eq:truth utility alt}) $ \geq $ (\ref{eq:lie utility alt}), by the definition of $\overbar{a_t}$.
\end{proof}

By backward induction 
starting from $t = m_i$ for each agent $i$, we can see that each agent's best response is always to submit its true rational valuation function.
For $m_i \to \infty$, the same argument can be made using the One-Shot Deviation Principle in place of backward induction. 

The next result, which is a simple adaptation of Theorem~\ref{thm:vcg is individually rational}, shows the agents are never worse-off by participating in the protocol.
\begin{corollary}\label{cor:bayes-nash ir}
Let $M \triangleright \phi$ be the interaction protocol.
Suppose each agent's true valuation function is as defined in (\ref{eq:alt agent valuation fn}) and $v_{t,i}(\cdot) \geq 0$ for all $t$ and $i$. 
Then $M \triangleright \phi$ is Individually Rational with respect to each agent's realisable cumulative utility at every time step.  
\end{corollary}   
\begin{proof}
Denote by $\overbar{h_{t-1}}$ the history at time $t$.
Fix an arbitrary agent $i$ 
and let $\overbar{a_t} := \arg \max_{\overbar{a}} \sum_j v_{t,j}(\overbar{h_{t-1}}, \overbar{a})$ and $\overbar{b_t} := \arg \max_{\overbar{b}} \sum_{j \neq i} v_{t,j}(\overbar{h_{t-1}}, \overbar{b})$.
Then agent $i$'s expected realisable cumulative utility is non-negative since 
\begin{align*}
& \mathbb{E}_{\overbar{or_t}} \biggl[ r_{t,i} - p_i(v_{t,i},v_{t,-i}) + \overline{q}_{t+1,i}(\overbar{h_{t-1}}\overbar{a_t}\overbar{or_t}) - \overline{c}_{t+1,i}(\overbar{h_{t-1}a_t or_t}) \biggr] \nonumber \\ 
& =  \sum_{j} v_{t,j}(\overbar{h_{t-1}},\overbar{a_t}) - \sum_{j \neq i} v_{t,j}(\overbar{h_{t-1}}, \overbar{b_t}) \\ 
& \geq  \sum_{j} v_{t,j}(\overbar{h_{t-1}},\overbar{a_t}) - \sum_{j} v_{t,j}(\overbar{h_{t-1}}, \overbar{b_t}) \\
& \geq  0
\end{align*}
by the non-negativity of $v_{t,i}(\overbar{h_{t-1}},\overbar{b_t})$ and definition of $\overbar{a_t}$.
\end{proof}

\begin{example}\label{ex:protocol full example alt}
Suppose a factory can produce one unit of a product in each time step, the product is perishable within one time step, and the factory only has enough raw material to produce two units of the product. 
Assume the product is valued by agents $A_1, A_2$ and $A_3$ at 100, 80, and 60 respectively, and suppose agents $A_1$ and $A_2$ are both already at the factory and $A_3$ is one time step away from arriving.
The action set at each time step is $\A = \{ 1,2,3 \}$, denoting which agent gets to consume the product.
We assume the outcome of each action is deterministic.
%
Suppose each agent has horizon 2. The following shows the rational q-functions of each agent:
    \begin{align*}
     & q_{2,1}(a_1or_1,a_2) := \ite{a_1 = 1}{0}{ \ite{a_2 = 1}{100}{0}  } \\
     & q_{2,2}(a_1or_1,a_2) := \ite{a_1 = 2}{0}{ \ite{a_2 = 2}{80}{0}  } \\
     & q_{2,3}(a_1or_1,a_2) := \ite{a_2 = 3}{60}{0} \\
     & q_{1,1}(a_1) := 100 \\
     & q_{1,2}(a_1) := \ite{a_1 = 3}{0}{80} \\
     & q_{1,3}(a_1) := 0
    \end{align*}
    As formulated in $q_{2,1}$ and $q_{2,2}$, $A_1$ and $A_2$ only attach value to consuming the product once, either at $t=1$ or $t=2$.
    The value of $q_{1,1}$ is 100 because $A_1$ will always be picked by the VCG mechanism to consume the product, either at $t=1$ or $t=2$.
    The value of $q_{1,2}$ is contingent on $a_1$, in that $A_2$ will be able to consume the product, as long as action $3$ is not picked at $t=1$, in which case it has to compete with $A_1$ at $t=2$ and lose.
    The value of $q_{1,3}$ is 0 because $A_3$ can never consume the product; action $3$ at $t=1$ yields 0 value to $A_3$ because it is not yet at the factory, and it will not win against either $A_1$ or $A_2$ at $t=2$.
    
    The following are the rational social cost functions of each agent, which can be obtained mechanically from the payment function $p_i(\cdot)$ -- see (\ref{eq:payment function}) for definition -- acting on the rational q-functions given above.  
    \begin{align*}
      & c_{2,i}(a_1or_1, a_2) := 0 \;\;\;\; \text{for all $i$} \\ 
      & c_{1,1}(a_1) := \ite{a_1 = 1}{0}{ \ite{a_1 = 2}{60}{80}  } \\
      & c_{1,2}(a_1) := \ite{a_1 = 1}{ 60 }{ 0 } \\
      & c_{1,3}(a_1) := 0 
    \end{align*}
    \noindent The agent valuation functions can be derived from $q_{t,i}$ and $c_{t,i}$ to yield
        \begin{align*}
            & v_{2,1}(a_1,a_2) := \ite{a_1 = 1}{0}{ \ite{a_2 = 1}{100}{0}} \\
      & v_{2,2}({a_1,a_2}) := \ite{a_1 = 2}{0}{ \ite{a_2 = 2}{80}{0} } \\
      & v_{2,3}(a_1,a_2) := \ite{a_2 = 3}{60}{0} \\  
      & v_{1,1}(a_1) :=  \ite{a_1 = 1}{ 100 }{ \ite{a_1 = 2}{ 40 }{20}  } \\
      & v_{1,2}(a_1) := \ite{a_1 = 1}{ 20 }{ \ite{a_1 = 2}{80}{0}} \\
      & v_{1,3}(a_1) := 0 
    \end{align*}    
    If all three agents submit $v_{t,i}$ truthfully, then at $t=1$, actions 1 and 2 
    both maximise social utility. 
    Breaking ties randomly, Tables~\ref{tab:random choice 1 alt} and \ref{tab:random choice 2 alt} show the two possible scenarios.

        \begin{table}[!htbp]
        \centering
        \begin{tabular}{c|c|c|c}
           \hline
                  & $a_1^* = 1$ & $a_2^* = 2$ & CU \\
           \hline       
           $A_1$  &  (100, 100, 60)  & (0,0,0) & 40 \\
           $A_2$  &  (20, 0, 0) &   (80,80,60) & 20 \\
           $A_3$  &  (0, 0, 0) &  (0,0,0) & 0 \\
           \hline
        \end{tabular}
        \caption{Scenario when $a_1^*$ is randomly chosen to be 1}
        \label{tab:random choice 1 alt}
    \end{table}
    
    \begin{table}[!htbp]
        \centering
        \begin{tabular}{c|c|c|c}
           \hline
                  & $a_1^* = 2$ & $a_2^* = 1$ & CU \\
           \hline       
           $A_1$  &  (40, 0, 0)  & (100,100,60) & 40 \\
           $A_2$  &  (80, 80, 60) &   (0,0,0) & 20 \\
           $A_3$  &  (0, 0, 0) &  (0,0,0) & 0 \\
           \hline
        \end{tabular}
        \caption{Scenario when $a_1^*$ is randomly chosen to be 2}
        \label{tab:random choice 2 alt}
    \end{table}
    \noindent The numbers in each cell are the realised $v_{t,i}$, $r_{t,i}$, and $p_{t,i}$ values. 
    The last column is the cumulative utility $\sum_{t=1}^2 r_{t,i} - p_{t,i}$ for each agent.
    Note that 
    both $A_1$ and $A_2$ get the same total cumulative utility irrespective of the random choice on the first action. 

\end{example}


Appendix~\ref{app:agent valuation function} explores a few alternative agent valuation functions. 
Each of them is arguably a more natural candidate that assumes less knowledge about the environment and the strategies of other agents, but these alternative valuation functions either do not consistently maximise total social utility, or they do not satisfy Bayes-Nash incentive compatibility (with respect to realisable cumulative utility).
This is perhaps not surprising; Definition~\ref{defn:agent valuation function} has good properties because it makes extreme assumptions, assumptions that we will need to drop in \S~\ref{sec:learning and planning} when designing practical algorithms for learning rational valuation functions.

\subsubsection*{Possible Protocol Variations}
\paragraph{Variations of Total Social Welfare}
In certain formulations of dynamic mechanism design like \cite{bergemann2010dynamic, qiu2022learning}, the mechanism is considered one of the agents with valuation function defined to be the sum of payments received from all the other agents.
The total social welfare in such a setup is then defined to be the sum of the cumulative utilities of all the agents including the mechanism agent, in which case the payments made to and received by the mechanism cancels out in the total social welfare.
In this setup, the payment terms have a neutral net effect on total social welfare, unlike our setup.
Depending on the intended application of the dynamic mechanism design problem, there are also natural setups where the goal is to maximise the total payments made to the mechanism, for example when the payments are a platform company's revenue.

\paragraph{Variations of Agent Valuation Functions}
In the case where the mechanism $M = (f,p_1,\ldots,p_k)$ is probabilistic (e.g. the Exponential VCG Mechanism described in \S~\ref{subsec:exponential mechanism}), the rational q and social cost functions can be generalised to take the expectation over the values of $f(v_{t,1},\ldots,v_{t,k})$.
We can also, if useful for the intended application, add discounting to (\ref{eq:agent q function}) and (\ref{eq:agent cost function}) in the usual manner.

\paragraph{Variations of Payment Functions}
A key attraction of having the Clark pivot term $\max_{\overbar{a_t}} \sum_{j\neq i} v_{t,j}(\overbar{a_t})$ in (\ref{eq:payment function}) is that, in addition to incentive compatibility and individual rationality, the payment function also (trivially) satisfy the no-positive-transfer property, which means no agent is ever paid money by the mechanism.
If the no-positive-transfer property is not important in an intended application, then the Clark pivot term can be dropped from (\ref{eq:payment function}) to yield the Team mechanism of \cite{athey2013efficient}.
If budget balance is important, which means the payments from all the agents sum to 0, then (\ref{eq:payment function}) can also be suitably adapted to implement the balanced Team mechanism from \cite{athey2013efficient}. 
We provide the details of a collusion-proof generalisation of the balanced Team mechanism, called the Guaranteed Utility Mechanism \cite{csoka2024collusion}, in Appendix~\ref{app:gum}. 

The payment function (\ref{eq:payment function}) for each agent $i$ in our proposed protocol appears to only consider the social effect of removing agent $i$ from the current time step, and this is in contrast to other dynamic mechanism design formulations like \cite{bergemann2010dynamic, qiu2022learning} where the payment for agent $i$ at time $t$ captures the social effect of removing agent $i$ from time $t$ onwards.
Our setup is quite natural for modelling the activities of long-lived agents that participate in different tasks and the intermittent social effect of their non-participation in certain time steps; indeed, our setup appears to be closely related to the general setup in \cite{pavan2014dynamic}, where \cite{athey2013efficient, bergemann2010dynamic} are special cases where the so-called impulse responses can be computed efficiently.
In any case, non-participation of an agent in all future time steps can be obtained by setting the agent's valuation function to 0 after a certain time, and this appears to be a natural thing to do when we apply our protocol to a specific problem like sequential allocation as shown in Example~\ref{ex:protocol full example alt}.

\paragraph{Partial Observability}
In practice, an agent participating in a mechanism-controlled environment will not have full knowledge of the environment and likely no full visibility of the joint actions and / or payments. 
Some of the possible configurations are covered in \S~\ref{subsec:special cases}, with more detailed descriptions of learning algorithms covered in \S~\ref{sec:learning and planning}.

\paragraph{Private Information}
In addition to partial observability, we can also explicitly introduce various forms of private information into our setup.
One noteworthy extension is to have reward modelling at the agent level, where each agent has a private reward function $R_{t,i}(h_{t-1}, or_t)$ that replaces the $r_{t,i}$ term in (\ref{eq:agent q function}).
Such a reward function can model the agent's preferences for different possible observations.
It can also be acquired from an agent's interaction with its human owner through reinforcement learning with human feedback techniques \cite{kaufmann2023survey}.

\subsection{Special Cases and Related Settings}\label{subsec:special cases}
Here are some special cases of a mechanism controlled environment $M\triangleright\phi$, all of which have a rich literature behind them.

\paragraph{Single Agent}
When $k = 1$, the formula (\ref{eq:alt agent valuation fn}) simplifies to (\ref{eq:Bellman single agent}) because $f(v) = \arg \max_{a} v(a)$ and the payment terms evaluate to 0, thus reducing the protocol to that of the single-agent general reinforcement learning setting described in \S~\ref{subsec:single agent}.
And, of course, when the actions are null or have no effect on the environment, we recover the sequential prediction setting, which includes Solomonoff Induction \cite{solomonoff1964formal-part1, solomonoff1964formal-part2}, prediction with expert advice \cite{cesa2006prediction} and the closely related zero-sum repeated games.

\paragraph{Two Player Turn-based Games}
When $k=2$ and the agents take turn executing actions (and therefore the actions are never exclusionary), our general setup reduces to two-player turn-based games, covering both perfect information and imperfect information settings, studied in game theory  \cite{vonneumann.morgenstern47, osborne2004introduction}.

\paragraph{Multi-Agent Reinforcement Learning}
If the actions of the $k$ agents are never mutually exclusive, then $\overbar{a_t^*} = f(v_{t,1},\ldots,v_{t,k}) = \arg \max_{\overbar{a} \in \Alt} \sum_i v_{t,i}(\overbar{a})$ is such that $\overbar{a_t^*} = \arg \max_{\overbar{a} \in \Alt} v_{t,i}(\overbar{a})$ for each $i \in [1,\ldots,k]$.
In that case, the mechanism $M$ in $M \triangleright \phi$ never play a role and the protocol becomes that of the multi-agent general reinforcement learning setting described in \S~\ref{subsec:multi agent}.
Comprehensive surveys of key challenges and results in this area can be found in \cite{shoham2008multiagent, nisan07, zhang2021multi}, with a unifying framework in \cite{lanctot2017unified}.

\paragraph{Static Mechanism Design}
When $m=1$ and $\phi$ is fixed, the setup recovers the classical mechanism design settings like auctions and single-decision markets \cite{borgers2015introduction}.
For example, if $\phi(\cdot \,|\, f(v_1,\ldots,v_k))$ assigns probability 1 to the outcome that agent $i^* = \arg \max_{i} v_i(i)$ gets the percept $(i^*,v_{i^*})$ and every other agent gets $(i^*, 0)$, and the payment function is the Clark pivot function, then $M \triangleright \phi$ reduces to the Vickrey second-price auction among $k$ bidders.

\paragraph{Dynamic Mechanism Design}
When $1 < m \leq \infty$ and $k$ and $\phi$ can change over time, then we are in the dynamic mechanism design setting \cite{bergemann2019dynamic}, with special cases like sequential auctions, and market participants that come and go.
Such dynamic mechanisms have been used, for example, in Uber's surge-pricing model
\cite{chen2016dynamic} and congestion control \cite{barrera2014dynamic}.
A general online mechanism design algorithm based on Hedge \cite{freund1997decision} can be found in \cite{huh2024nash}.

\paragraph{Multi-Agent Coordinated Planning}
There are both similarities and important differences between our mechanism-controlled multi-agent environments setup and the literature on online mechanisms for coordinated planning and learning in multi-agent systems \cite{parkesS03, cavallo2006optimal}.
In the latter case, there is usually a top-level Markov Decision Process (MDP) whose transition and reward functions are common knowledge to all the agents, and each of the agents may only be active during certain time periods and they hold private information -- usually something simple like the value attached to winning an item being auctioned, or a hidden state that forms part of the state for the top-level MDP -- that are needed by a central planner to work out the optimal joint policy for all the agents. 
In that setup, the key problem is the design of dynamic VCG-style mechanisms to incentivise all the agents to truthfully reveal their private information to the central planner at each time step.
Also worth noting that the concept of self-interested agents in the multi-agent coordinate planning literature is mostly about an agent who may lie about its own private information in order to gain an advantage, which is a narrow form of the classical economic notion of a self-interested agent that seeks to act in such a way to maximise its own expected future cumulative rewards.

\paragraph{Multi-Agent Coordinated Learning}
Our setup can also be understood game-theoretically as a simultaneous-move sequential game in which there are $k$ agents, the action space for each agent is the set of all valuation functions, and the loss function for agent $i$ given the collective actions (i.e. submitted valuation functions of all the agents) is its negative utility as determined by the VCG mechanism.
In the learning in games literature \cite{fudenberg1998theory}, the underlying game dynamics is usually assumed to be static and the concern of each agent is primarily around learning a best response to the possibly adaptive strategies of other agents.
Key results in such simultaneous-move repeated games can be found in \cite{cesa2006prediction}.

\paragraph{Dynamic Mechanism Design via Reinforcement Learning}
In the multi-agent coordinated planning case, the underlying MDP is assumed to be known.
When this assumption is dropped, the mechanism designer has to use reinforcement learning algorithms to learn the parameters of the VCG mechanism from data, including the payment functions. 
The simpler bandit setting is tackled in \cite{kandasamyGJS23}. 
A reinforcement learning (RL) algorithm using linear function approximation is described in \cite{qiu2022learning}.
It is worth noting that it is the mechanism designer that is using RL to learn the optimal VCG mechanism parameters over time from the environment made up of multiple agents and their instantaneous (or horizon 1) reward functions. 
The RL for dynamic mechanism design framework is silent on where the agents' reward functions come from; each agent's reward function could come from another learning process, for example via the agent learning its owner's preferences.  

\section{Learning in the Presence of Social Cost}\label{sec:learning and planning}

\subsection{Measures of Success}
In \cite{shoham2008multiagent}, when it comes to multi-agent reinforcement learning, the authors ask ``If learning is the answer, what is the question?''.
The answer is not obvious, not only because the underlying environment is non-stationary -- which can already appear in the single-agent reinforcement learning setting -- but also because the agents can adapt to each other's behaviour so each agent can play the dual roles of learner and teacher at the same time.
For example, the storied Tit-for-Tat strategy \cite{axelrod1981emergence, nowak1992tit, nowak1993strategy} is an agent policy that both learn from and `teach' the other agents in iterated Prisoner's Dilemma-type problems.

Several possible and non-unique theories of successful learning, both descriptive and prescriptive, are provided in \cite[\S 7]{shoham2008multiagent}.
In the context of multi-agent reinforcement learning under mechanism design, we are concerned primarily with designing learning algorithms that satisfy some or all of the following properties, noting that the agents do not learn policy functions but only valuation functions that are fed into a VCG mechanism for joint action selection. 
\begin{description}
    \item[Convergence] Starting from no or only partial knowledge of the environment and the other agents, each agent's learned valuation function at any one time should converge to (\ref{eq:alt agent valuation fn}). 

    \item[Rational] An agent's learning algorithm is rational if, whenenever the other agents have settled on a stationary set of valuation functions, it settles on a best response to that stationary set.
    
    \item[No Regret] An agent's learning algorithm minimises regret if,  
    against any set of other agents, it learns a sequence of valuation functions that achieves long-term utility almost as well as what the agent can achieve by picking, with hindsight, the best fixed valuation function for every time step.
    
    \item[Incentive Compatibility] In the case when the parameters of the VCG mechanism are learned through data, we require that the mechanism exhibits approximate incentive compatibility with high probability.
\end{description}
Some of these concepts will be made more precise in the coming subsections.

\subsection{Bayesian Reinforcement Learning Agents}\label{subsec:brl agents}
In practice, an agent $i$ operating in a given controlled multi-agent environment $M \triangleright \phi$ will not actually have enough information to construct $v_{t,i}$ as defined in (\ref{eq:alt agent valuation fn}).
First of all, it does not know what $\phi$ is.
A good solution is to use a Bayesian mixture $\xi_{\mathcal P}$, for a suitable model class $\mathcal{P}$, to learn $\phi$.
So at time $t$ with history $\overbar{h_{t-1}}$, agent $i$ approximates the expression $\phi(\overbar{or_t} \,|\, \overbar{h_{t-1}}\overbar{a_t} )$ by
\begin{equation}\label{eq:mixture for percept}
   \sum_{\rho \in \mathcal{P}} w_0^{\rho} \rho( \overbar{or_t} \,|\, \overbar{h_{t-1}}\overbar{a_t} ).  
\end{equation} 
The quantity $\overbar{p_t} = (p_1,\ldots,p_k)$ can also be estimated directly using, say, another Bayesian mixture $\xi_{\mathcal{Q}}$ 
via
\begin{equation}\label{eq:mixture for payment} 
  \sum_{\rho \in \mathcal{Q}} w_0^{\rho} \rho( \overbar{p_t} \,|\, \overbar{apor_{<t}}\overbar{a_t} ). 
\end{equation}
In general terms, we can think of (\ref{eq:mixture for percept}) as learning the dynamics of the underlying environment $\phi$, and (\ref{eq:mixture for payment}) as learning the preferences and strategies of the other agents in the environment, which determine the payments charged.
By simulating possible futures, the mixture model (\ref{eq:mixture for percept}) can be used to estimate the agent's rational q-function, and the mixture model (\ref{eq:mixture for payment}) can be used to estimate the agent's rational social cost function.

\subsubsection{Online Mixture Learning in Practice}
The optimal model class $\mathcal{P}$ (and $\mathcal{Q}$) for each agent to use is the class of all computable functions, yielding a multi-agent Solomonoff induction setup, where we can see that each agent's model converges at a fast rate to the true environment by virtue of Theorem~\ref{thm:MEMConvergence}.
This proposal has two issues.
First of all, it is an incomputable solution.
Secondly, Theorem~\ref{thm:MEMConvergence} is only applicable when the environment is computable, and this assumption is violated when the environment contains other agents that are themselves incomputable.

In practice, Solomonoff induction can be approximated efficiently using factored, binarised versions of the Context Tree Weighting algorithm \cite{WST95, VH18, veness09, yang2022direct}, or online prediction-with-experts algorithms like Exponential Weights / Hedge \cite{freund1997decision, cesa2006prediction, arora2012multiplicative}, Switch \cite{volf1998switching, venessNHB12} and their many special cases \cite{hoeven2018many}.
While these algorithms have their roots in online convex optimisation, they can be interpreted as Bayesian mixtures when the loss function is log loss 
\cite{koolen2013universal}, which in our setup is $-\log_2 M(or)$ where $M$ is the model for $\phi$ and $or$ is the observed percept.
In particular, the Hedge algorithm is an exact Bayesian mixture model in the case when the loss function is log loss and the learning rate is 1, in which case one can show that the Hedge weight for each expert is its posterior probability \cite[\S 9.2]{cesa2006prediction}. The Prod algorithm with log loss \cite{orseau2017soft} has been shown to be a "soft-bayes" algorithm, which coincides with the exact Bayesian mixture when the learning rate is 1 and can be interpreted as a "slowed-down" version of Bayesian mixture or a Bayesian mixture with "partially sleeping" experts when the learning rate is less than 1.\footnote{When the true environment is contained in the model class, the optimal learning rate is 1 because Bayesian inference is optimal. However, in the agnostic / improper learning setting where the true environment may not be in the model class, the learning rate needs to decrease over time to avoid pathological issues especially in non-convex function classes \cite{grunwald2017inconsistency, van2015fast}. }
More generally, \cite{koolen2013universal} shows that there is a spectrum of Bayesian mixture algorithms over expert sequences with different priors that can be understood as interpolations of Bayesian mixtures over fixed experts and (non-learning) element-wise mixtures. The spectrum includes Fixed-Share \cite{herbster1998tracking} with static Bernoulli switching frequencies, Switching distributions with dynamic slowly decreasing switching frequencies \cite{erven2007catching, erven2012catching} and dynamically learned switching frequencies \cite{volf1998switching}, and more general switching frequencies that are dependent on some ordering on experts \cite{vovk1997derandomizing, koolen2013universal}.
We will look at some specific Hedge-style algorithms shortly in \S~\ref{subsec:dynamic hedge aixi}.

\subsubsection{Monte Carlo Planning} 
The recurrence in each agent's rational q-function and social cost function can be approximated using variants of the Monte Carlo Tree Search (MCTS) algorithm \cite{browne2012survey, vodopivecSS17, swiechowski2023monte}. 
Even though there are no explicit min-nodes and max-nodes in an MCTS tree, it is known that MCTS converges to the (expecti)minimax tree because as the number of roll-out operations goes to infinity, the UCT \cite{ks06} selection policy at each decision node concentrates the vast majority of the roll-outs on the best child nodes so the weighted average back-up rewards converges to the max/min value. 
It is also worth noting that, rather than choosing the action that maximises the value function at the root of the MCTS tree, the agent declares the entire value function to $M\triangleright\phi$ for a joint action to be chosen by the $M$ mechanism.

\subsubsection{Partial Observability}\label{subsec:partial obs}
We have so far assumed each agent can see everything, including the declared $\overbar{v_t}$ valuation functions of other agents, the chosen joint action $\overbar{a_t}$, the full joint percept $\overbar{or_t}$, and all payments $\overbar{p_t}$.
It is possible to relax this assumption, which we will briefly look at now.

\begin{assumption}\label{ass:agent percept}
Each agent $i$ sees, at time $t$, the joint action $\overbar{a_t}$ but only its own percept $\overbar{or_{t|i}} := or_{t,i}$ and the value of its payment $p_i(\overbar{v_t})$. 
Its view of the history $\overbar{h_t}$ up to time $t$ is denoted $\overbar{h_{t|i}} := \overbar{a_1}or_{1,i}\overbar{a_2}or_{2,i}\ldots \overbar{a_t}or_{t,i}$. 
\end{assumption}

\begin{defn}\label{defn:agent view of multi-agent environment}
Given a multi-agent environment $\varrho$, we can define agent $i$'s view of $\varrho$ under Assumption~\ref{ass:agent percept} as 
$\varrho_{|i} = \{ \varrho_{0|i}, \varrho_{1|i}, \varrho_{2|i}, \ldots \}$, where $\varrho_{t|i} : \overbar{\A}^t \to \Dist (\Obs_i \times \Reward)^t$ is defined by
\[ \varrho_{t|i}( o'r'_{1:t} \,|\, \overbar{a_{1:t}}) \,:= \sum_{\substack{\overbar{or_{1:t}} \\ \text{st} \; \overbar{or_{1:t|i}} = o'r'_{1:t}}} \varrho_t( \overbar{or_{1:t}} \,|\, \overbar{a_{1:t}}). \]
\end{defn}

The valuation function (\ref{eq:alt agent valuation fn}) can no longer be approximated directly with the partial observations.
Instead, we will have to use $\phi_{t|i}$ instead of $\phi_t$ in (\ref{eq:alt agent valuation fn}), and learn the rational q-function and social cost functions 
from data obtained from interactions with the environment, possibly using Bayesian mixture estimators. 

\subsubsection*{Markovian State Abstraction} 
To maintain computational tractability and statistical learnability, we will need to approximate estimators like (\ref{eq:mixture for percept}) with 
\[ \sum_{\rho \in \mathcal{P}} w_0^{\rho} \rho( \overbar{or_t} \,|\, \chi(\overbar{h_{t-1}}\overbar{a_t}) ), \]
where $\chi$ is a feature function that maps arbitrarily long history sequences into usually a finite $n$-dimensional feature space that is a strict subset of $\mathbb{R}^n$.
Depending on the application, such a $\chi$ could be hand-coded by domain experts, or picked from a class of possible feature functions using model-selection principles.
The aim is to select a mapping $\chi$ such that the induced process can facilitate learning without severely compromising performance.  
Theoretical approaches to this question have typically focused on providing conditions that minimise the error in the action-value function between the abstract and original process \cite{abel18,li06}. 
The $\Phi$MDP framework \cite{Hutter2009FeatureRL} instead provides an optimisation criteria for ranking candidate mappings $\chi$ based on how well the state-action-reward sequence generated by $\chi$ can be modelled as an MDP whilst still being predictive of the rewards. A good $\chi$ results in a model that is Markovian and predictive of future rewards, facilitating the efficient learning of good actions that maximise the expected long-term reward.
The class of possible feature functions can be defined using formal logic \cite{de2008logical, lloyd03, dzeroskiRD01} or obtained from embedding techniques \cite{Bengio:2003:NPL:944919.944966, cai2018comprehensive, goyal2018graph} and Deep Learning techniques \cite{mnih2015human, arulkumaran2017deep}. 

One such algorithm, motivated by AIXI \cite{Hutter:04uaibook}, is described in \cite{yang2022direct}.
In this case, the formal agent knowledge representation and reasoning language is as described in \cite{lloyd11,lloyd03}. 
In particular, the syntax of the language are the terms of the $\lambda$-calculus \cite{church}, extended with type polymorphism and modalities to increase its expressiveness for modelling agent concepts. 
The semantics of the language follow Henkin \cite{henkin50completeness}. 
The inference engine has two interacting components: an equational-reasoning engine 
and a tableau theorem prover. 
There is also a predicate rewrite system for defining and enumerating a set of predicates.
The choice of formalism is informed by the standard arguments given in \cite{farmer} and the practical convenience of working with (the functional subset of) a language like Python.
(Suitable alternatives to the formalism are studied in the Statistical Relational Artificial Intelligence literature \cite{deraedt16}, including a few that cater specifically for relational reinforcement learning  \cite{dzeroskiRD01} and Symbolic MDP/POMDPs \cite{kerstingOR04, sannerK10}.)
The feature selection problem was addressed through the use of a so-called random forest binary decision diagram algorithm to find a set of features that approximately minimise an adjusted version of the $\Phi$-MDP criteria \cite{NSH11}.

\subsubsection{Dynamic Hedge AIXI}\label{subsec:dynamic hedge aixi}

We now describe a slight modification of the Bayesian reinforcement learning agent first presented in \cite{yang2024dynamic} that can be used in our multi-agent general reinforcement learning with social cost setup.
The agent combines online learning, Monte Carlo planning, state abstraction, and can deal with partial observability. 
As we shall we see, it has attractive convergence and equilibrium properties. 

We work in the single agent general reinforcement learning setup described in \S~\ref{sec:grl}, with environment models that are obtained as single-agent views of general multi-agent environments as given in Definition~\ref{defn:agent view of multi-agent environment}. The agent operates within the interaction protocol as described in \S~\ref{subsec:general case}.
The general algorithmic strategy is to find the best way to approximate AIXI as directly as possible.

\paragraph{AIXI Approximation}
The AIXI agent can be viewed as containing all possible knowledge as its Bayesian mixture is performed over all computable distributions. From this perspective, AIXI's performance does not suffer due to limitations in its modelling capacity.
In contrast, all previous approximations of AIXI are limited to having a finite pre-defined model class containing a subset of computable probability distributions, presenting an irreducible source of error.
To address this issue, we propose to work in a dynamic knowledge injection setting, where an external source is used to provide additional knowledge that is then integrated into new candidate environment models. 
In particular, dynamic knowledge injection can model an external feature-construction process like \cite{yang2022direct} that regularly injects new features into the agent's learning process, or a human-AI teaming constructs where the human can provide additional domain knowledge that the agent can use to model aspects of the environment. 
Once a new environment model is proposed, the central issue is then to determine how it can be incorporated to improve the agent's performance. Utilising a variation of the GrowingHedge algorithm \cite{MM17:growing_expert}, itself an extension of Hedge \cite{cesa2006prediction}, we construct an adaptive \textit{anytime} Bayesian mixture algorithm that incorporates newly arriving models and also allows the removal of existing models.

\paragraph{Prediction with Specialist Advice}
The prediction with expert advice setting is a well-established framework providing theoretically sound strategies on how to aggregate the forecasts provided by many experts in a sequential setting \cite{cesa2006prediction}. This setting is characterised by a game played between a learner and an adversary. Initially, a loss function $\ell: \mathcal{X} \times \mathcal{Y} \to \R$ is provided, where $\mathcal{X}$ is the vector space of predictions and $\mathcal{Y}$ is the outcome space. The learner has access to a set of fixed experts $\mathcal{M}$. At time $t$, a learner receives prediction $x_{t, i} \in \mathcal{X}$ from expert $i$. The learner then must combine the predictions from all experts and outputs $x_t \in \mathcal{X}$. An adversary then chooses an outcome $y_t \in \mathcal{Y}$ causing the learner to incur loss $\ell_t = \ell(x_t, y_t)$ and observe the loss $\ell_{t, i} = \ell(x_{t, i}, y_t)$ for each expert $i$. Learners are typically designed to minimise the \textit{regret} 
\begin{equation}\label{eqn:regret} 
L_T - L_{T, i} = \sum_{t=1}^{T} \ell_{t} - \sum_{t=1}^{T} \ell_{t, i}, 
\end{equation}
a measure of the relative performance of the agent with respect to any fixed expert $i \in \mathcal{M}$. 
The Hedge (aka exponential weights) algorithm is a simple yet fundamental algorithm 
in this setting \cite{cesa2006prediction}. Given a prior distribution $\bm{\nu}$ over $\mathcal{M}$ and learning rate $\eta > 0$, Hedge predicts 
\begin{align*}
    x_t = \frac{\sum_{i \in \mathcal{M}} w_{t, i} x_{t, i}}{\sum_{i \in \mathcal{M}} w_{t, i}}
\end{align*}
where $w_{t, i} = \nu_i e^{- \eta L_{t-1, i}}$. The weights of the Hedge algorithm can be viewed as the posterior probabilities of each expert. The following is a standard regret bound for the Hedge algorithm.

\begin{theorem}[\cite{cesa2006prediction}]
\label{prop:hedge_regret}
If the loss function $\ell$ is $\eta$-exp-concave, then for any $i \in \mathcal{M}$, Hedge with prior $\bm{\nu}$ has regret bound $L_T - L_{T, i} \leq \frac{1}{\eta} \log \frac{1}{\nu_i}.$
\end{theorem}

Incorporating expert advice from new experts arriving in an online fashion can be cast into the specialists setting \cite{FSSW97}, which extends the prediction with expert advice setting by introducing specialists: experts that can abstain from prediction at any given time step. In this setting, the learner has access to a set $\mathcal{M}$ of specialists where at time $t$, only specialists in a subset $\mathcal{M}_t \subseteq \mathcal{M}$ output predictions. 
The crucial idea to adapt the Hedge algorithm to this setting was presented in \cite{CV09} where inactive specialists $j \notin \mathcal{M}_t$ are attributed a forecast equal to that of the learner.

\paragraph{Abstract Environment Models}
Markov state abstraction provides a framework for the external process to inject new features and models.
A state abstraction is a mapping $\psi: \Hist \to \St_{\psi}$ that maps the space of history sequences into an abstract state space.
Given history $h_{t}$ at time $t$, the state at time $t$ is given by $s_t = \psi(h_{t})$. 
In this manner, the interaction sequence of the original process is mapped to a state-action-reward sequence.
For a given $\psi$, an abstract Markov Decision Process (MDP) predicts the next state and reward according to a distribution $\rho_\psi: \St_{\psi} \times \A \to \Dist(\St_{\psi} \times \Reward)$ that factorises into a state transition and reward distribution as 
\begin{equation}\label{eqn:abstract MDP factorisation} 
 \rho_\psi(s', r \,|\, s, a) = \rho_{\psi}(s' \,|\, s, a) \rho_{\psi}(r \,|\, s, a, s'). 
\end{equation}
Let $\rho_{\psi} \coloneqq (\rho_{t, \psi})_{t \geq 1}$, where $\rho_{t, \psi}: \St_{\psi} \times \A \to \Dist(\St_{\psi} \times \Reward)$ for all $t$. 
We refer to the pair $(\psi, \rho_{\psi})$ as an abstract environment model. 

An abstract MDP can simplify the environment's dynamics but pushes a lot of the complexity into the design of the state abstraction function and a sufficiently powerful representation is required to ensure as little generality is lost. 
In particular, the quality of an abstract environment model will determine how closely its reward distribution approximates the underlying environment's reward distribution.
Following \cite{yang2022direct, ng2008probabilistic}, we consider the class of predicate environment models 
where the state abstraction is of the form $\psi(h) = (p_1(h), \ldots, p_n(h))$ and $p_i: \Hist \to \{0, 1\}$ are predicates definable in higher-order logic \cite{lloyd03}. 
%
We construct a new data structure named $\Phi$-BCTW by generalising the Context Tree Weighting (CTW) algorithm \cite{WST95} to use predicates $p_i : \Hist \to \{0,1\}$ from a set $\Phi$ as the context functions in the internal nodes of the tree.
In a $\Phi$-BCTW tree, each sub-tree of depth $d$ is a $\Phi$-prediction suffix tree ($\Phi$-PST) model. For a history $h$, a $\Phi$-PST with predicates $p_i$ at depth $i$ computes a path from root to leaf node as $p_1(h)p_2(h)\ldots p_d(h)$, which forms a state abstraction. At each leaf node resides a KT estimator \cite{KT06} maintaining a distribution over the next bit. 
By chaining together multiple $\Phi$-BCTW trees, we can predict the binary representation of arbitrary symbols.


Finally, by exploiting the distributive law and the structure of prediction suffix trees, one can show that a $\Phi$-BCTW data structure constructed using a set $\Phi$ of $D$ predicates is able to perform an exact Bayesian mixture over $2^{(2^D)}$ $\Phi$-PST models in $\bigO(D)$ time.
With predicates in higher-order logic as context functions, it is shown in \cite{lloyd11, lloyd03, farmer} that such models can represent all computable non-Markovian environments.

\paragraph{Dynamic Hedge AIXI Agent}

%
Our agent, shown in Algorithm \ref{alg:hedge_aixi}, extends the prediction with specialist setting to general reinforcement learning with state abstractions.
In particular, we consider the case where specialists are abstract MDPs.
Each specialist $i \in \mathcal{M}_t$ produces a function $V_i^{\pi_i}: \Hist \times \A \to \R$ denoting the expected utility of action $a_t$ under a policy $\pi_i: \St_{i} \to \A$ up to a horizon $T$:
\begin{align}
    V_{i}^{\pi_i}(h_{t-1}, a_t) 
    &= \sum_{sr_{t:t+T}} \left[ \sum_{j=t}^{t+T} r_j - p_j \right] \rho_i(srp_{t:t+H} \,|\, h_{t-1}, a_{t:t+T}), \label{eqn:specialist_V}
\end{align}
where $\rho_i$ is an extension of (\ref{eqn:abstract MDP factorisation}) to also include the payment from the interaction protocol, and the actions after $a_t$ are selected via $a^{i}_{t+k} = \pi_i(s^{i}_{t+k})$.
At each time step, specialist $i$ predicts a state-action conditional distribution over the next reward and payment, and its prediction is evaluated based on the log loss 
\[ \ell_{t, i} = - \log \rho_{t, i}(r_t p_t \,|\, s_{t-1}^{i}, a_t, s_{t}^{i}). \] 
Thus, DynamicHedgeAIXI will weight specialists based on how well they predict the reward and payment sequences over time. 
Instead of picking the action that maximises the weighted sum of the given $V$ values like in \cite{yang2024dynamic}, here the agent just submits the weighted $V$ functions to the interaction protocol.
There are several sensible choices for $\pi_i$, including knowledge-seeking policies like \cite{orseau2014universal, jin2020reward}.


\begin{algorithm}
\caption{DynamicHedgeAIXI}\label{alg:hedge_aixi}
    \begin{algorithmic}[1]
    \setstretch{1.18}
    \State \textbf{Require:} Interaction protocol $M \triangleright \mu$, 
    learning rate $\eta > 0$, prior weights $\bm{\nu} = (\nu_i)_{i \geq 1}$, 
    \State \textbf{Require:} Sequence of sets of contiguous specialists $(\mathcal{M}_t)_{t \geq 1}$, 
    \State \textbf{Require:} Policies $\bm{\pi} = (\bm{\pi_t})_{t \geq 1}$, where $\bm{\pi_t} = (\pi_i)_{i \in \mathcal{M}_t}$ and $\pi_i: \St_i \to \A$.
    \State \textbf{Initialize:} $L_0 = 0$. For $i \in \mathcal{M}_1$, set $w_{1, i} = \nu_i$. 
    \For{$t = 1, 2, \ldots, T$}
        \State Set $\hat{w}_{t, i} = \frac{w_{t, i}}{\sum_{j \in \mathcal{M}_t} w_{t, j}}$
        \State Submit $\tilde{v}_{t,i} = \sum_{i \in \mathcal{M}_t} \hat{w}_{t, i} V_i^{\pi_i}(h_{t-1}, a)$ to $M \triangleright \mu$ \label{alg:DHA_action_selection}
        \State Observe $o_t, r_t, p_t$ from $M \triangleright \mu$ 
        \State $\forall i \in \mathcal{M}_t$, set $s^i_t = \psi_i(h_{t-1}aor_t p_t)$ 
        \State Set $\rho_t = \sum_{i \in \mathcal{M}_t} \hat{w}_{t, i} \rho_{t, i}(\cdot| s_{t-1}^{i}, a_t, s_{t}^{i})  $
        \State Set $\ell_t = -\log \rho_t(r_t p_t \,|\, s_{t-1}^{i}, a_t, s_{t}^{i})$ 
        \State $\forall i \in \mathcal{M}_t$, set $\ell_{t, i} = - \log \rho_{t, i}(r_t p_t \,|\, s_{t-1}^{i}, a_t, s_{t}^{i})$ 
        \State Set $L_t = L_{t-1} + \ell_t$
        \State $\forall i \in \mathcal{M}_t \cap \mathcal{M}_{t+1}$, set $w_{t+1, i} = w_{t, i} e^{- \eta \ell_{t, i}}$ 
        \State $\forall i \in \mathcal{M}_{t+1} \setminus \mathcal{M}_t$, set $w_{t+1, i} = \nu_i e^{- \eta L_t}$
    \EndFor
    \setstretch{1}
    \end{algorithmic}
\end{algorithm}

In \cite{yang2024dynamic}, the authors show that DynamicHedgeAIXI is the richest direct approximation of AIXI to date and comes with strong value-convergence guarantees. In particular, 
\begin{enumerate}\itemsep1mm\parskip0mm
    \item The model $\tilde{v}_{t,i}$ used in DynamicHedge AIXI (line \ref{alg:DHA_action_selection}) is an exact Bayesian mixture over the available set of models at each time step when $\eta = 1$. The convergence behaviour of the $\tilde{v}_{t,i}$ can thus be understood using Theorem~\ref{thm:MEMConvergence}. 
    
    \item DynamicHedge AIXI will achieve good value convergence rates against the best sequence of environment models $\mu = \mu_1\ldots \mu_T$ available to the agent, in that the cumulative squared difference of the value under DynamicHedgeAIXI has an upper bound that is linear in $\log \frac{1}{w(\mu)}$, where $w(\mu)$ is the prior weight assigned to the sequence $\mu$. 
\end{enumerate}

\subsubsection{Approximate Nash Equilibrium}\label{subsec:nash equilibrium}

The classic result on the effectiveness of Bayesian learning in multi-agent systems is given in \cite{kalai1993rational}, where it is shown that, in a group of interacting agents, if   
\begin{itemize}\itemsep1mm\parskip0mm
    \item each agent models and keeps track of the other agents' strategies using a Bayesian mixture, and 
    \item produces at every time step a best-response policy, as measured by expected cumulative utility, to the Bayesian mixture of opponent strategies, 
\end{itemize} 
then the group of agents will converge to an $\epsilon$-Nash equilibrium in repeated plays of normal-form and stochastic games as long as each agent's Bayesian mixture has a ``grain of truth'', in that every possible opponent strategy is assigned a non-zero probability. 

The result was extended in \cite{leike2016formal} to the multi-agent general reinforcement learning setting, where the authors also provided a theoretical solution to the grain-of-truth problem 
that uses Reflective Oracles \cite{fallenstein2015reflective} and a variant of AIXI that uses Thompson sampling to pick policies \cite{leike2016thompson}.
The Dynamic Hedge AIXI agent can be seen as a practical approximation to the learning reflective agents of \cite{leike2016formal}, in that Dynamic Hedge AIXI satisfies the mixture-modelling and best-response policy (see Theorem 2 in \cite{yang2024dynamic}) conditions, and the grain-of-truth condition is given a realistic chance of being realised through the dynamic knowledge injection setup and the use of higher-order logic, which is Turing complete, for representing state abstractions.

The key difference between our setup and that of \cite{leike2016formal} is that it is the VCG mechanism, rather than the individual agents, that picks the joint action for all the agents.
In that sense, the maximum expected cumulative utility achievable for each agent, and therefore the definition of best response, is with respect to the policy executed by the VCG mechanism based on valuation functions submitted by the agents; each agent's own policy $\pi_i$ is only used to make sure it can get a good estimate $\tilde{v}_{t,i}$ of the rational valuation function given in Definition~\ref{defn:agent valuation function} through possible simulations of the future.

\subsection{Swap Regret and Correlated Equilibrium}

As seen in \S~\ref{subsec:nash equilibrium}, a collection of Bayesian reinforcement learning agents will converge to a Nash equilibrium as long as the grain-of-truth condition is satisfied. 
In cases where the grain-of-truth condition cannot be (confidently) satisfied, we may wish to settle for convergence to a correlated equilibrium \cite{aumann1987correlated}, where no agent would want to deviate from their strategies assuming the others also do not deviate.
Correlated equilibrium includes Nash equilibrium as a special case but is strictly more general.

Online learning algorithms that minimise regret turn out to also be important tools for achieving correlated equilibriums.
In particular, \cite{foster1997calibrated} shows that, in repeated plays of a normal form game $G$, if each agent adopts a strategy that learns to minimise their swap regret, and that the respective swap regret converges to 0, then the empirical distribution of the agents' actions converges to a correlated equilibrium. 
Swap regret is a generalisation of (\ref{eqn:regret}) that allows comparison of the agent's actual performance with respect to an alternative strategy that applies an arbitrary swap function $\omega : \mathcal{X} \to \mathcal{X}$ on the agent's chosen actions in hindsight. (For example, in a stock-picking contest, the swap regret may compare the agent's actual performance against an alternative strategy that swaps the agent's choice to Alphabet every time it picked IBM, and to TSMC every time it picked Intel.)

In \cite{blum2007external, chen2020hedging}, the authors describe a general procedure to turn agents that minimise the standard form of regret into agents that minimise the swap regret, which is regret that allows for any specific agent action to be switched with some other action with the benefit of hindsight.
This is achieved via a master algorithm that runs $N = |\A|$ instances of a standard regret minimisation algorithm, one for each possible action. Each of the $A_l$ algorithms, $l\in [N]$, maintains a $q_l^t \in \R^N$ weight vector at each time step as usual. The master algorithm maintains a weight vector $p^t \in \R^N$ that is the solution to the equation $p^t = p^tQ^t$, where $Q^t \in \R^{N \times N}$ is the matrix made up of row vectors $q_l^t, l\in [N]$. 
(The $p^t$ on the RHS of $p^t = p^t Q^t$ can be understood as the weights attached to $A_l$ algorithms, and the $p^t$ on the left is best understood as the probability of selecting the different actions. Thus, $p^t$ is the stable limiting distribution of the stochastic matrix $Q^t$ and the existence and efficient computability of $p^t$ is given by the Perron-Frobenius Theorem.)
The master algorithm incurs a loss vector $\ell^t \in \R^N$ at each time step, which is then attributed to $A_l$ by $p_l^t \ell^t$. Intuitively, each $A_l$ is responsible for minimising the regret of switching action $l$ to any other action via its standard regret minimising property.
The above master algorithm can be adjusted with the multi-armed bandit algorithm in \cite{auer2002nonstochastic} to deal with the partial information case, where the master algorithm selects an action $a_t$ by sampling $p^t$ and receives a loss $\ell^t \in R$ related only to $a_t$ (instead of the full $p^t$).

It is possible to na{\"i}vely reduce our multi-agent general reinforcement learning problem to a normal-form game with a large action space -- basically, the set of all policies mapping histories to agent valuation functions -- and have each agent use the above swap-regret minimisation algorithm to collectively converge to a correlated equilibrium.
But we can do better. There has been a significant number of algorithmic improvements on the problem of swap-regret minimisation in the last few years. 
The current state-of-the-art is the Multi-Scale Multiplicative Weight Update (MSMWU) algorithm \cite{peng2024fast} and the slightly more general TreeSwap meta-algorithm \cite{dagan2024external}.
Both algorithms 
\begin{itemize}\itemsep1mm\parskip0mm
    \item provide a swap-regret bound that has time complexity that is logarithmic in the action space, substantially improving the bound of \cite{blum2007external} that is polynomial in the action space;
    \item can be applied to extensive-form games to achieve extensive-form correlated equilibrium in time polynomial in the number of agents, the agent's action space, and the horizon of the game. 
\end{itemize}
Since the multi-agent general reinforcement learning with mechanism problem can be naturally represented as an extensive-form game, the MSMWU / TreeSwap algorithm can be used by each agent in our setup to achieve convergence to a correlated equilibrium.

\subsection{Bandit VCG Mechanisms}\label{subsec:bandit vcg}
Note that, in practice, the agents' valuation functions are not fully known but estimated from the actual percepts they get from the environment, which are in turn dependent on the joint actions chosen by the VCG mechanism.
This means formula (\ref{eq:max welfare}) in the mechanism-controlled environment protocol needs to be handled carefully; in particular there is an exploration-exploitation trade-off here, where we need to do enough explorations for the agents to arrive at good estimates of their valuation functions before we can reliably compute the $\arg \max_{\overbar{a} \in \Alt}$ over them.
This problem can be solved using Bandit algorithms, and the techniques described in \cite{kandasamyGJS23} that uses upper-confidence bounds \cite{Aue03, lattimore2020bandit} are directly applicable in our setting. 
A key finding in \cite{kandasamyGJS23} is that (asymptotic) truthfulness is harder to achieve with agents that learn their valuation functions; this may also be an issue in our setting.

\subsection{Markov VCG Mechanisms in Unknown Environments}
\S~\ref{subsec:brl agents} describes agents that learn in a mechanism-controlled environment.
In this section, we take the perspective of the mechanism designer and look at reinforcement learning algorithms that can adjust the parameters of the VCG mechanism based on interactions with agents that are themselves capable of learning and adjusting.

We will first summarise the setup and algorithmic framework described in \cite{qiu2022learning} and then describe some possible extensions.
The environment with a controller and $k$ agents is defined by an episodic Markov Decision Process $\Xi = (\St, \A, H, \mathcal{P}, \{ r_i \}_{i=0}^k)$, where $\St$ and $\A$ are the state and action spaces, $H$ is the length of each episode, $\mathcal{P} = \{ \mathcal{P}_t : \St \times \A \to \D(\St) \}_{t=1}^H$ is the state transition function, and $r_i = \{ r_{i,t} : \St \times \A \to [0,1] \}_{t=1}^H$ are the reward functions, with $r_0$ denoting the reward function for the controller and $r_i, 1\leq i\leq k$, denoting the reward function for agent $i$.
Except for $r_0$, the environment $\Xi$ is unknown to the controller.
The controller interacts with the $k$ agents in multiple episodes, with each episode lasting $H$ time steps.
An initial state $x_1$ is set at the start of each episode.
For each time step $t$ in the episode, the controller observes the current state $x_t \in \St$, picks an action $a_t \in \A$, and receives a reward $r_{0,t}(x_t,a_t)$. 
Each agent $i$ receives their own reward $r_{i,t}(x_t,a_t)$ and report a value $\widetilde{r}_{i,t}(x_t,a_t)$ to the controller. 
At the end of the episode, the controller charges each agent $i$ a price $p_i$.
Given the controller's policy function $\pi = \{ \pi_t : \St \to \A \}_{t=1}^H$ and the prices $\{ p_i \}_{i=1}^k$, the controller's utility for the episode is defined by
\[ u_0 = V_1^\pi(x_1 ; r_0) + \sum_{i=1}^k p_i, \]
and each agent $i$'s utility for the episode is defined by 
$u_i = V_1^\pi(x_1 ; r_i) - p_i$, 
where
 \[ V_h^\pi(x ; r) = \sum_{t = h}^H \mathbb{E}_{\pi,\mathcal{P}} [ r_t(x_t, \pi_t(x_t)) \mid x_h = x ]. \]
The controller's goal is to learn the policy $\pi^*$ and pricing $\{ p_i \}_{i=1}^k$ that implements the so-called Markov VCG mechanism
\begin{gather}
    \pi^*(x) = \arg \max_\pi V_1^\pi (x ; R ) \label{eq:markov vcg pi*}\\
    \pi^*_{-i}(x) = \arg \max_{\pi} V_1^\pi (x ; R^{-i} ) \label{eq:markov vcg pi -i}\\
    p_i(x) = V_1^{\pi^*_{-i}}\bigl( x, R^{-i} \bigr) - V_1^{\pi^*}\bigl( x, R^{-i} \bigr), \label{eq:markov vcg price}
\end{gather}
where $R = \sum_{j=0}^k r_j$ and $R^{-i} = R - r_i$.

\begin{lemma}[\cite{lyu2022pessimism}]
The Markov VCG mechanism is incentive compatible and individually rational.    
\end{lemma}

Of course, one can only implement the Markov VCG mechanism directly if the $\mathcal{P}$ and $\{r_i\}_{i=0}^k$ elements of the underlying episodic MDP are known, and that (\ref{eq:markov vcg pi*}) and (\ref{eq:markov vcg pi -i}) can be computed efficiently.
In practice, the controller has to learn $\mathcal{P}$ and $\{ r_i \}_{i=0}^k$ from interactions with the environment and agents, and (\ref{eq:markov vcg pi*}) and (\ref{eq:markov vcg pi -i}) need to be handled using function approximation when the state and action spaces $\St$ and $\A$ are large.

In \cite{qiu2022learning}, a reinforcement learning framework was proposed for learning Markov VCG with Least-Squares Value Iteration (LSVI) \cite{jin2020provably}.
There are two phases: an exploration phase and an exploitation phase.
During the exploration phase, the controller uses reward-free reinforcement learning \cite{jin2020reward} to interact with the environment and the $k$ agents to appropriately explore the underlying MDP.
This exploration phase is crucial because the controller would need enough data to approximate (\ref{eq:markov vcg pi -i}) and (\ref{eq:markov vcg price}) by simulating counterfactual scenarios 
for when each agent $i$ is missing from the environment.
During the exploitation phase, the controller will then repeat the following steps in each episode $t$: 
\begin{enumerate}\itemsep1mm\parskip0mm
    \item Construct a $\widehat{\pi}^t$ that approximates (\ref{eq:markov vcg pi*}) using LSVI on previously collected data $D$;
    \item Learn an approximation $F_t^{-i}(x)$ of $V_1^{\pi^*_{-i}} \bigl( x, R^{-i} \bigr)$, the first term in (\ref{eq:markov vcg price}), using LSVI on collected data $D$; 
    \item Learn an approximation $G_t^{-i}(x)$ of $V_1^{\pi^*} \bigl( x, R^{-i} \bigr)$, the second term in (\ref{eq:markov vcg price}), using LSVI on collected data $D$ and $\widehat{\pi}^t$ from Step 1; 
    \item For time step $h = 1, \ldots ,H$, observe state $x_{t,h}$, take action $a_{t,h} = \widehat{\pi}^t(x_{t,h})$ and observe $\widetilde{r}_{i,t,h}(x_{t,h},a_{t,h})$ from each agent $i$.
    \item Charge each agent $i$ the price $p_{i,t}(x_1) = F_t^{-i}(x_{t,1}) - G_t^{-i}(x_{t,1})$.
    \item Add $\{ (x_{t,h},a_{t,h}, \{ \widetilde{r}_{i,t,h}(x_{t,h},a_{t,h}) \}_{i=1}^k) \}_{h=1}^H$ to $D$.
\end{enumerate}
In the first three steps, the Least-Squares Value Iteration algorithm is used, in conjunction with a suitable model class $\mathcal{F}$, to construct a sequence of functions $(\widehat{Q}_{t,h} )_{h=1}^H$ that approximates the Q-function for the corresponding value function $V_1^{\pi}(\cdot, r)$ at episode $t$ via the following optimisation problem, from $h=H,\ldots,1$:
\begin{gather*}
  \widehat{Q}_{t,h} = \arg \min_{f \in \mathcal{F}} \sum_{\tau=1}^{t-1} \bigl[ r_{\tau,h}(x_{\tau,h}, a_{\tau,h}) + \widetilde{V}_{h+1}^{\pi}(x_{\tau,h+1}) - f(x_{\tau,h}, a_{\tau,h}) \bigr]^2 + \text{reg}(f),
\end{gather*}
where reg($f$) is a suitable regularisation term, and $\widetilde{V}_{h+1}^{\pi}(x) = \widehat{Q}_{t,h+1}(x, \pi(x))$.
In \cite{qiu2022learning}, $\mathcal{F}$ is taken to be linear functions of the form $f(x,a) = \langle w, \chi(x,a) \rangle$ for some feature mapping $\chi(\cdot,\cdot) \in \R^m$ and the authors give efficient algorithms for linear MDPs that can (i) learn to recover the Markov VCG mechanism from data, and (ii) achieve regret, with respect to the optimal Markov VCG mechanism, upper-bounded by $O(T^{2/3})$, where $T$ is the total number of episodes.

For general MDPs or Partially Observable MDPs, we can replace the LSVI with linear function class with approximate AIXI algorithms like \cite{yang2022direct, yang2024dynamic} to learn Markov VCG mechanisms from data. 

\subsection{Mechanism-level RL vs Agent-level RL} 
Note that (\ref{eq:markov vcg pi*})-(\ref{eq:markov vcg price}) can be understood as a special case of (\ref{eq:agent q function})-(\ref{eq:expected total utility}), 
in that the resultant policy (executed by the mechanism) maximises the total sum of each agent's cumulative utility 
when the underlying environment $\phi$ is an episodic MDP, and all the agents have the same horizon $m$.
The key advantage of performing reinforcement learning at the mechanism level is that, in the case when the mechanism is (approximately) incentive compatible, the RL algorithm has access to and can learn from all available data from interactions between the agents and the environment and mechanism, including all the individual rewards and payments. 
The key disadvantage is that it appears necessary to assume that all the agents have the exact same horizon.
In comparison, the situation is reversed when RL is done at the level of individual agents: each agent's RL algorithm usually only has access to the subset of the interaction data in which it has a role in generating, but each agent can use different RL algorithms with different parameters like horizon / discount factor and different function approximation model classes.

\section{Applications}
\label{sec:applications}

\subsection{Paperclips and All That}\label{subsec:paperclip}

The paperclip maximiser \cite{bostrom2020ethical} is a thought experiment in the AI safety folklore that illustrates a potential risk from developing advanced AI systems with misaligned goals or values. 
Here is the idea:
Imagine an AI system is tasked with the goal of maximising the production of paperclips. 
(This could be an explicit high-level goal tasked by a human, or a subgoal inferred as needed by the AI system for fulfiling a higher-level goal.)
As the AI system becomes more and more intelligent, it may pursue this paperclip-production goal with relentless efficiency, converting all available resources and matter into paperclips, potentially leading to disastrous consequences for humanity and the environment.
A much older (and less trivial) variation of the problem was articulated by AI pioneer Marvin Minsky, who suggested that an AI system designed to solve the Riemann hypothesis might decide to take over all of Earth's resources to build supercomputers to help achieve its goal.

While these are obviously extreme examples, these thought experiments help focus our minds on the following key issues in the development of increasingly powerful AI systems:
\begin{itemize}\parskip0mm\itemsep1mm
\item Even seemingly harmless or trivial goals, if pursued by a superintelligent AI system without proper constraints, could lead to catastrophic outcomes as the AI single-mindedly optimises for that goal.
\item In particular, in single-mindedly pursuing a goal, a superintelligent AI may exhibit so-called convergent instrumental behavior, such as acquiring power and resources and conducting self-improvement as subgoals, to help it achieve its top-level goals at all costs, even if those actions conflict with human values or well-being.
\end{itemize}
These thought experiments underscore the importance of instilling the right goals, values, and constraints into AI systems from the outset, as it may be extremely difficult or impossible to retrofit them into a superintelligent system once created.

There is a rich literature \cite{gabriel2020artificial, everittLH18, christian2021alignment} on aligning the design of AI systems with human values, and there is a lot of useful analyses in the single-agent general reinforcement learning (GRL) setting (\S~\ref{subsec:single agent}) on how to make sure, for example, that 
\begin{itemize}\itemsep1mm\parskip0mm
    \item the agent infers the reward function from a human by observing the person's actions through cooperative inverse reinforcement learning \cite{hadfield2016cooperative}, leading to a provable solution for the off-switch problem \cite{hadfield2017off} in some cases; and 
    \item preventing an AI agent from performing adverse `self-improvement' by tampering with its own reward function through robust learning techniques that incorporates causality \cite{everitt2021reward} and/or domain knowledge \cite{everitt2017reinforcement}. 
\end{itemize}

We argue here that, in the multi-agent GRL setting, additional controls in the form of imposition of social cost on agent actions can help prevent paperclip maximiser-style AI catastrophes.
In particular, in the multi-agent GRL setting where there are multiple superintelligent AI systems acting in the same environment, an AI system cannot unilaterally perform actions that destroy humanity and the environment, or engage in instrumental convergent behaviour like acquiring all available power and resources, without encountering significant frictions and obstacles because most of such actions, and their prevention by other agents (human or AI), are mutually exclusionary and therefore can be subjected to control through economic mechanisms like that described in \S~\ref{sec:action cost}, imposed either explicitly through rules and regulations or implicitly through laws of nature (like physics and biology \cite{chatterjee2012evolutionary, reiter2015biological}).
This form of social control works in concert and in a complementary way with the controls at the single-agent GRL level. Some of the controls are likely necessary but none in isolation are sufficient; together they may be.

In the paperclip maximiser example, there are two forces that will stop paperclip production from spiraling out of control. 
Firstly, the environment will provide diminishing (external) rewards for the agent to produce more and more paperclips, assuming there are controls in place to prevent wireheading issues \cite{muehlhauser2014exploratory, yampolskiy14} where the agent actively manipulates the environment to give it false rewards or changes its own perception of input from the environment.
Secondly, at the same time that the utility of the paperclip maximiser agent is decreasing, the utility of other agents in the same environment in taking competing actions to prevent paperclip production will increase significantly as unsustainable paperclip production threatens the environment and the other agents' welfare.
Thus, at some stage, the utility of the paperclip maximiser agent in producing more paperclips will become lower than the collective utility of other agents' proposed actions to stop further paperclip production, and a VCG-style market mechanism will choose the latter over the former and paperclip production stops.
The argument above does rely on an assumption that the agents are operating on a more-or-less level playing field, where they need to take others' welfare into consideration when acting.
In a completely lopsided environment where there is only one superintelligent agent and its welfare dominates that of all others, which could come about by design, accident, or over time through phenomenon like the Matthew effect (aka rich-get-richer or winner-take-all), social controls will not be able to stop unsustainable paperclip production, and it will only stop when the one superintelligent agent wants it to stop. 
Both conditions that uphold the Matthew effect and the circumstances that cause it to fail are studied in the literature \cite{rigney2010, BerbegliaH17, barabasi2018formula} and those additional control mechanisms, like partial lotteries \cite{frey2023rationality}, will likely be needed in addition to what we propose in \S~\ref{sec:action cost}.

The argument presented in this section has similar motivations to those presented in \cite{conitzer2024social}. We expect to see a lot more progress in this research topic in the coming months and years.

\subsection{Cap-and-Trade to Control Pollution}\label{subsec:cap and trade}

Consider a set of oil refineries $\lbrace R_{1}, R_{2}, \ldots, R_{k} \rbrace$, where $k > 1$. Each refinery $R_{i}$:
\begin{itemize}\itemsep1mm\parskip0mm
    \item produces unleaded fuel that can be distributed and sold in the retail market for \$2 per litre. (We assume the output of the refineries is not big enough to change the market demand, and therefore the price of fuel.);
    \item requires \$1.80 in input cost (crude oil, chemicals, electricity, labour, distribution) to produce and distribute 1 litre of fuel (for details, see \cite{favennec2022economics}); 
    \item can produce up to 100 million litres of fuel per day; and
    \item produces greenhouse gases, costed at \$190 per cubic ton. 
\end{itemize}
The refineries are not, however, equally efficient. Refinery $R_{i}$ emits $$s_{i}(y) = m_{i}\left(\frac{y^3}{5} - 12 y^2 + 200y + 888\right)$$ cubic tons of greenhouse gases per day, which is a function of $y$ the amount of fuel (in millions of litres) it produces per day and $m_{i}$, an inefficiency factor. 
Throughout this example we set $m_{i} = i$. Thus refinery $R_{1}$ is twice as efficient as $R_{2}$, three times as efficient as $R_{3}$, and so on.

\begin{figure}
    \centering
    \includegraphics[width=0.85\textwidth]{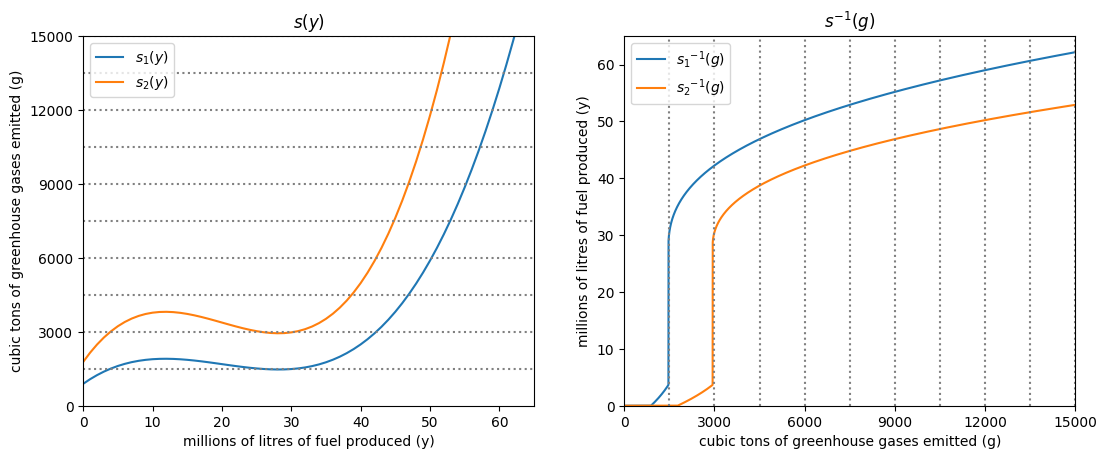}
    \caption{Plots of $s$ and $s^{-1}$ for refineries $R_{1}$ and $R_{2}$.}
    \label{fig:sA_s_inv_AB}
\end{figure}

Graphs of $s_{1}$ and $s_{2}$ are shown in Fig~\ref{fig:sA_s_inv_AB}. The greenhouse gas emissions increase at the start, then drop as the refinery achieves its optimum operating efficiency, after which they increase again rapidly. We also plot
\begin{equation*}
{s_{i}}^{-1}\left(g\right) := \max{\lbrace y \in \mathbb{Y} : \Im(y) = 0 \rbrace },
\end{equation*}
where $\Im(y)$ denotes the imaginary part of $y$, and the set $\mathbb{Y}$ comprises $0$ and the three roots of the cubic equation
$$
\frac{m_{i}}{5}y^3 - 12m_{i}y^2 + 200m_{i}y + (888m_{i} - s_{1}(y)) = 0
$$ 
when it is evaluated at $g := s_{1}(y)$.
This ensures ${s_{i}}^{-1}$ is defined $\forall g \geq 0$ and is non-decreasing. Note in Fig~\ref{fig:sA_s_inv_AB} that $s^{-1}_{1}$ and $s^{-1}_{2}$ have step increases at $s_{1}(28) \approx 1,470$ and $s_{2}(28) \approx 2,940$. Our construction of $s^{-1}$ assumes the refineries will maximise their production (and hence profit) for a given cap on greenhouse gas emissions. For example, if Refinery $R_{2}$ is permitted to emit $3,000$ cubic tons of greenhouse gases, it will choose to produce $30.5$ million litres of fuel rather than $3.9$ or $25.6$ million litres.

\paragraph{No Price on Pollution}
Consider the case of $k = 2$. If 
the two refineries do not have to pay for environmental damage from greenhouse gas emissions, each plant will produce at the maximum capacity of 100 million litres per day since they each make \$0.20 operating profit per litre. The total greenhouse gas emission between them will be $s_{1}\left(100\right) + s_{2}\left(100\right) = 302,664$ cubic tons per day.

\begin{figure}
    \centering
    \includegraphics[width=0.85\textwidth]{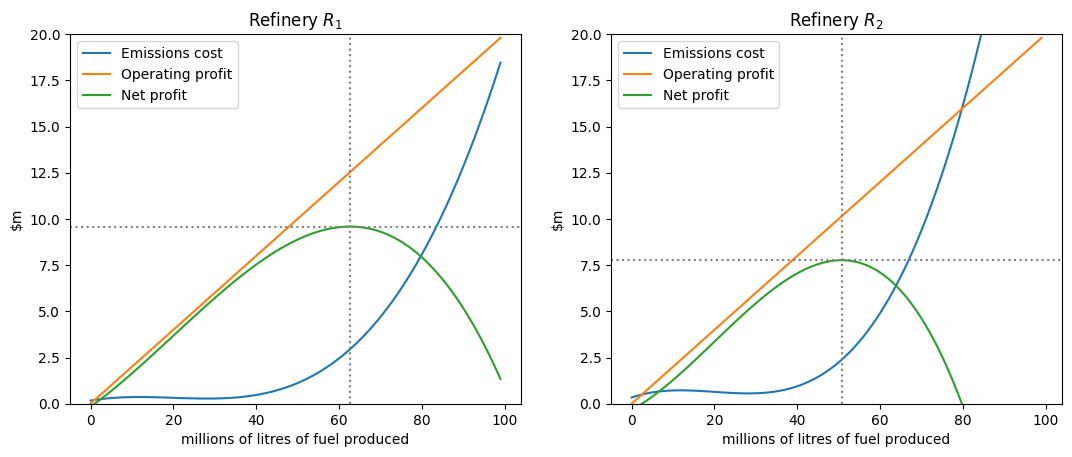}
    \caption{Cubic tons of greenhouse gas emitted for different production of fuel}
    \label{fig:max_net_p}
\end{figure}

\paragraph{Fixed Price for Pollution}
If 
the two refineries have to pay the \$190 per cubic ton cost but there is no cap to greenhouse gas emission, they will seek to maximise their net profit functions $n$, given by, respectively, $n_{1}\left(y\right) = \$200k \cdot y - \$190 \cdot s_{1}\left(y\right)$ and $n_{2}\left(y\right) = \$200k \cdot y - \$190 \cdot s_{2}\left(y\right)$. Taking derivatives with respect to $y$ we have:
\begin{align*}
n^{\prime}_{1}\left(y\right) = -114y^{2} + 4560y + 162000 \\
n^{\prime}_{2}\left(y\right) = -228y^{2} + 9120y + 124000
\end{align*}
which have positive roots at $y^{*}_{1} = 20 + 10\sqrt{\frac{346}{19}}$ and $y^{*}_{2} = 20 + 10\sqrt{\frac{538}{57}}.$ That is, Refinery $R_{1}$ will stop production at $y^{*}_{1} \approx 62.7$ million litres per day, for a daily net profit of $n_{1}\left(y^{*}_{1}\right) \approx \$9.58 \text{ million}$. Similarly, Refinery $R_{2}$ will stop production at $50.72$ million litres per day, for a daily net profit of $\$7.77$ million. See Fig~\ref{fig:max_net_p}. At the \$190 price, the total greenhouse gas emission between the two refineries is thus $s_{1}\left(y^{*}_{1}\right) + s_{2}\left(y^{*}_{2}\right) \approx 28,682$ cubic tons per day.


\paragraph{Market Pricing for Capped Pollution}
Now, 
instead of setting the greenhouse gas emission at \$190 per cubic ton, which requires a lot of economic modelling and assumptions, what if we just want to cap the total amount of emission by the two refineries to, say, 15 thousand cubic tons per day and let the market decide the price of greenhouse gas emission? 
This can be done via sequential Second Price auctions (pg.~\pageref{ex: second price auction}) of 15,000 pollution permits every day, auctioned in, say, tranches of 3,000, each permit entitling the owner to emit 1 cubic ton of greenhouse gas. Denoting the auction winner by $i^{*}$, the change in each refinery's net profit after tranche $t$ is:
$$
\Delta n_{t, i}\left(\Delta y_{t, i}\right) =
\begin{cases}
    {\rho}_{t, i^{*}} - a_{t, -i^{*}} & i = i^{*} \\
    0 & \text{otherwise}
\end{cases}
$$
where
${\rho}_{t,i} = \$200k \cdot \Delta y_{t, i} 
             = \$200k \cdot \left({s_{i}}^{-1}\left(g_{t,i} + 3000\right) - {s_{i}}^{-1}\left(g_{t,i}\right)\right)$
and $g_{t,i}$ is $R_{i}$'s permit holdings before the tranche and
$a_{t,i}$ is $R_{i}$'s bid for tranche $t$.

\paragraph{Greedy Algorithm} Let's work through the example shown in Table~\ref{tab:permit auctions}, where each refinery pursues a greedy bid strategy. That is, for tranche $t$ refinery $R_{i}$ always bids $a_{t, i} = \rho_{t, i}$. In the first tranche of 3000 permits, Refinery $R_{1}$ works out that it can produce $\Delta y_{1, 1} = 42$ million litres of fuel with 3000 permits, since ${s_{1}}^{-1}(3000) \approx 42$, and it is willing to pay a max of $\rho_{1,1} = \$200k \times 42 = \$8.4$m to win those 3000 permits.
Refinery $R_{2}$ similarly works out that it can produce $\Delta y_{1, 2} = {s_{2}}^{-1}(3000) \approx 31$~million litres of fuel and bids $\rho_{1,2} = \$6.1$m.
The VCG mechanism picks $R_{1}$ as the winner, and $R_{1}$ pays \$6.1m to produce 42 million litres of fuel, for a profit of \$2.3m.
In the second tranche of 3000 permits, Refinery $R_{2}$ submits the same bid, but Refinery $R_{1}$ submits a much lower bid of $\rho_{2,1} = \$1.6$m, since it can only produce an extra $\Delta y_{2, 1} = {s_{1}}^{-1}(6000) - {s_{1}}^{-1}(3000) \approx 50 - 42 = 8$~million litres of fuel with the additional 3000 permits.
Thus, $R_{2}$ wins the auction and pays only \$1.6m for the second tranche of 3000 permits, which it uses to produce 31~million litres of fuel for a profit of \$4.5 million.
The following tranches proceed in a similar way. In the end, we have 
\begin{itemize}\itemsep1mm\parskip0mm
    \item Refinery $R_{1}$ winning 9000 permits for a total cost of \$8.0m, and using them to produce 55 million litres of fuel for a total net profit of \$3.1 million; 
    \item Refinery $R_{2}$ winning 6000 permits for a total cost of \$3.2m, and using them to produce 42 million litres of fuel for a total net profit of \$5.3 million;   
    \item A total of \$11.2 million was collected for the 15,000 permits. 
\end{itemize} 
The result is interesting in that the more efficient Refinery $R_{1}$ ends up winning more permits as expected, which it uses to produce more fuel but for a lower total profit compared to Refinery $R_{2}$.
\begin{table}
\begin{center}
\begin{small}
\begin{tabular}{|c|c|c|c|c|}
\hline
Tranche & $R_{1}$ Bid & $R_{2}$ Bid & Payment & Result \\
\hline
3000 & 42m / \$8.4m & 31m / \$6.1m & \$6.1m & $R_{1}$ / \$2.3m  \\
3000 & 8m / \$1.6m & 31m / \$6.1m & \$1.6m & $R_{2}$ / \$4.5m \\
3000 & 8m / \$1.6m & 12m / \$2.4m & \$1.6m & $R_{2}$ / \$0.8m \\
3000 & 8m / \$1.6m & 4m / \$0.9m & \$0.9m & $R_{1}$ / \$0.7m \\
3000 & 5m / \$1.0m & 4m / \$0.9m & \$0.9m & $R_{1}$ / \$0.1m \\
\hline
\end{tabular}
\end{small}
\end{center}
\caption{The auction results assuming each refinery pursues a greedy strategy.}\label{tab:permit auctions}
\end{table}

\paragraph{Reinforcement Learning}
Each refinery can do better by using reinforcement learning to optimise their sequence of bids, using the approaches described in \S~\ref{sec:learning and planning}, and possibly engaging in collusion / cooperation. Figs~\ref{fig:r1_ex1_rfunc1}, \ref{fig:r1_ex1_rfunc2} and \ref{fig:r1_ex1_rfunc3} show the results of running two Q-Learning agents, representing the two refineries, on the cap-and-trade problem. Each agent's state is the 2-tuple (number of permits won by $R_{1}$, number of permits won by $R_{2}$). The action space is a set of 170 bids, $\mathbb{A}~=~\{ 0, \; 50\text{k}, \; 100\text{k}, \; \ldots \; 8.4\text{m} \}$. Agents are allowed to bid any amount in $\mathbb{A}$, even if that amount is higher than ${\rho}_{t,i}$. Any ties arising in the VCG mechanism due to agents bidding the same amount are broken randomly. Each experiment uses the same random seed and RL parameters. The only difference in their setup is the reward function.

In Fig~\ref{fig:r1_ex1_rfunc1} we incentivise the two agents to maximise their individual net profits, by rewarding only the agent $i^{*}$ that won the auction:
$${r_{t,i}}^{(1)} = \Delta n_{t, i} =
\begin{cases}
    {\rho}_{t,i^{*}} - a_{t,-i^{*}} & i = i^{*} \\
    0 & \text{otherwise}
\end{cases}$$
The agents learn a joint policy that results in them each making a total net profit of around \$9m over the course of the five tranches, much higher than the net profits obtained when following the greedy strategy earlier, of $\$3.1$m and $\$5.3$m for $R_{1}$ and $R_{2}$ respectively. The average permit price stabilises at $\$100$.

In Fig~\ref{fig:r1_ex1_rfunc2} we incentivise the agents to maximise their shared net profit:
$${r_{t,i}}^{(2)} = {\rho}_{t,i^{*}} - a_{t,-i^{*}}$$
Here the agents learn a joint policy that always results in 
a total shared net profit of \$19.49m over the five tranches and, importantly, drives the average permit price down to zero. 
Note the agents have learned this collusive policy solely by observing the permit holdings of each participant in the auction -- there was never any explicit communication between them.
This phenomenon is analysed more carefully in Fig~\ref{fig:r1_ex1_pol_rfunc2} in Appendix~\ref{ap:cap and trade}. 
    
    

\begin{figure}
    \centering
    \includegraphics[width=1\textwidth]{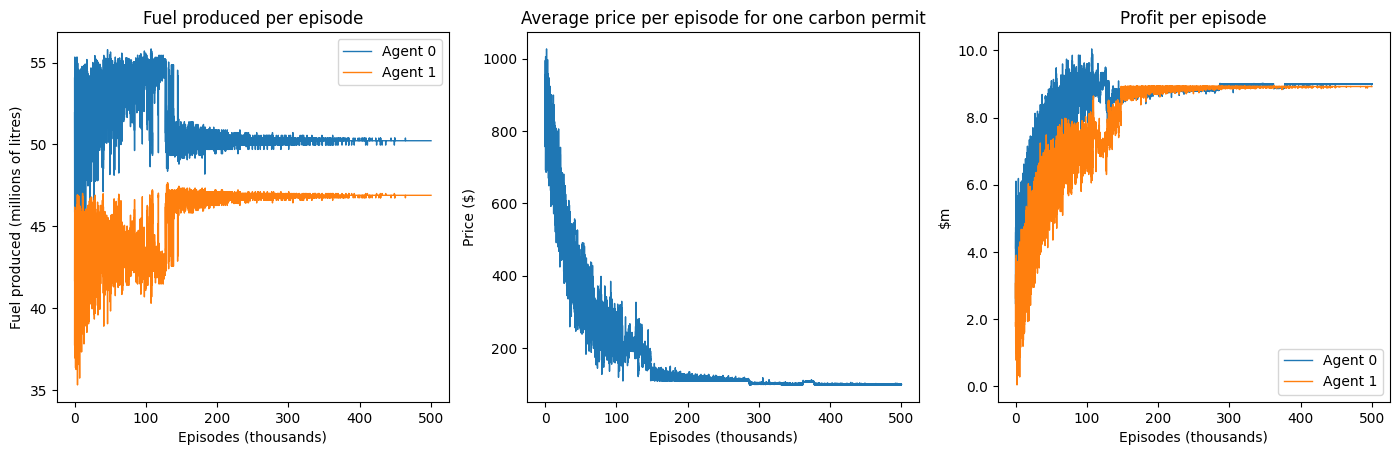}
    \caption{RL with reward function $r^{(1)}$ (incentivise individual profit)}
    \label{fig:r1_ex1_rfunc1}
\end{figure}

\begin{figure}
    \centering
    \includegraphics[width=1\textwidth]{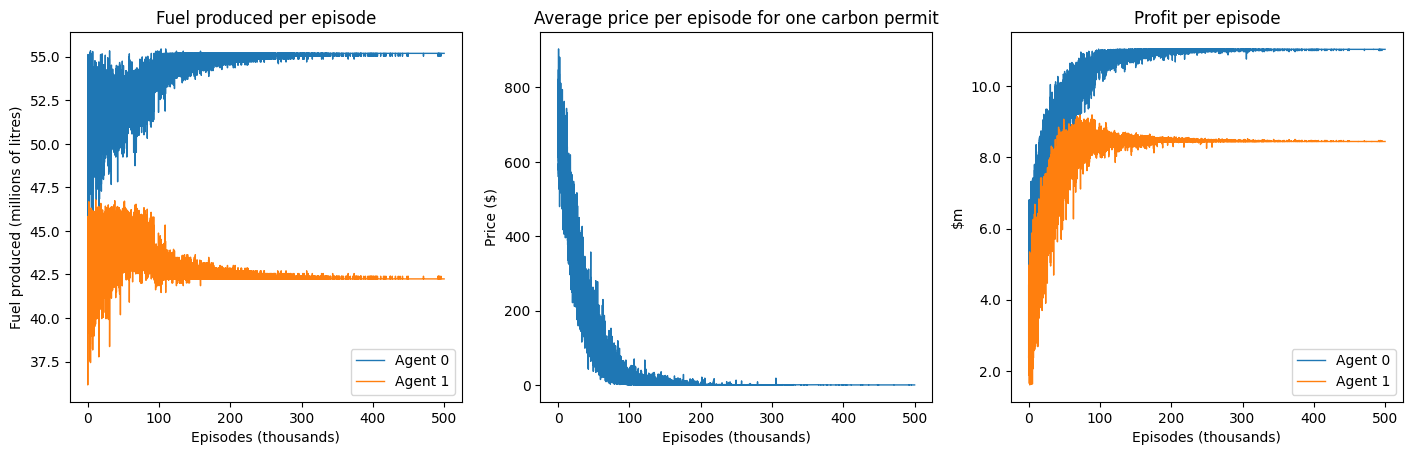}
    \caption{RL with reward function $r^{(2)}$ (incentivise joint profit)}
    \label{fig:r1_ex1_rfunc2}
\end{figure}

\begin{figure}
    \centering
    \includegraphics[width=1\textwidth]{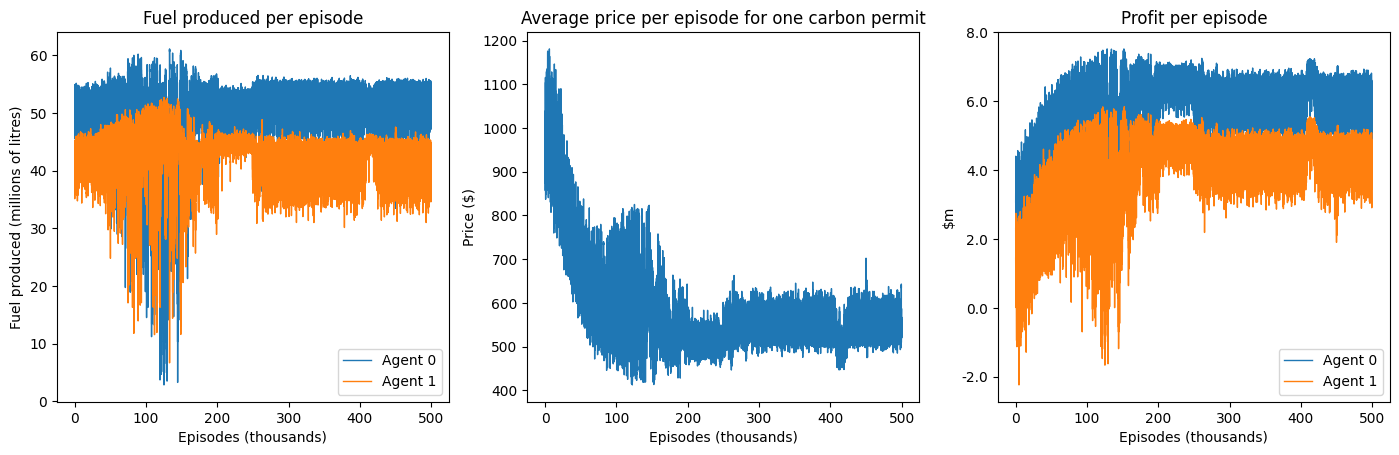}
    \caption{RL with reward function $r^{(3)}$ (zero-sum reward)}
    \label{fig:r1_ex1_rfunc3}
\end{figure}

In Fig~\ref{fig:r1_ex1_rfunc3} we incentivise the agents to maximise their individual net profit \textit{and} minimise the other agent's net profit:
$$
{r_{t,i}}^{(3)} = 
\begin{cases}
    {\rho}_{t,i^{*}} - a_{t,-i^{*}} & i = i^{*} \\
    -\left({\rho}_{t,i^{*}} - a_{t,-i^{*}}\right) & \text{otherwise}
\end{cases}
$$
As expected, in this case, competition between the agents keeps the average permit price well above zero; it eventually settles into an oscillatory pattern around the \$550 mark. 
This phenomenon is analysed in more detail in Fig~\ref{fig:r1_ex1_pol_rfunc3} in Appendix~\ref{ap:cap and trade}. 

\subsection{Other Applications}
The multi-agent general reinforcement learning with social cost setting proposed in this paper has a wide variety of applications, including the coordination of multiple automated penetration testing agents \cite{schwartz2019autonomous, sarraute21} in cyber security, the regulation of dynamic pricing algorithms \cite{calvano2020artificial, brero2022learning} in e-commerce platforms, and moderation of spread of misinformation \cite{acemoglu2023model} in social media.
These more detailed applications will be written up elsewhere.

\section{Discussion and Conclusion}
\label{sec:discussion}

In the spirit of \cite{parkes2015economic}, we considered in this paper the problem of social harms that can result from the interactions between a set of (powerful) general reinforcement learning agents in arbitrary unknown environments and proposed the use of (dynamic) VCG mechanisms to coordinate and control their collective behaviour.
Our proposed setup is more general than existing formulations of multi-agent reinforcement learning with mechanism design in two ways: 
\begin{itemize}\itemsep1mm\parskip0mm
    \item the underlying environment is a history-based general reinforcement learning environment like in AIXI \cite{hutter2024introduction};
    \item the reinforcement-learning agents participating in the environment can have different time horizons and adopt different strategies and algorithms, and learning can happen both at the mechanism level as well as the level of individual agents. 
\end{itemize} 
The generality of the setup opens up a wide range of algorithms and applications, including multi-agent problems with both cooperative and competitive dynamics, some of which we explore in \S~\ref{sec:learning and planning} and \S~\ref{sec:applications}.
The setup closest to ours in generality in the literature is \cite{ivanov2024principal}, with interesting applications like \cite{zheng2022ai}.

A key limitation of our proposed approach, especially when it comes to regulating powerful AGI agents, is that there is no fundamental way to enforce the VCG mechanism on such agents outside of purposedly designed platform economies \cite{kenney2016rise, cohen2017law, evans2016matchmakers}.
In particular, the proposed approach would not work on AGI agents operating "in the wild".
Nevertheless, the study of the ideal market mechanism to regulate AGI agents is a useful step for understanding and benchmarking the design of more practical mechanisms like \cite{kearns2014mechanism} that recommend but do not enforce social choices, and more indirect payment approaches like \cite{zhang2023steering} and \cite{kolumbus2024paying} to steer a set of agents to desired good behaviour.
It would also be interesting, as future work, to understand how natural biological mechanisms like \cite{chatterjee2012evolutionary, reiter2015biological} relate to ideal market mechanisms.

Another future work is a more systematic study of the dynamics of agents that have different discount factors. 
One such result in the "impossibility theorem" in \cite{milgrom2024kenneth}, which states that if a multi-agent reinforcement learning system is efficient, i.e. the selected policy is never one that is Pareto dominated, then the agent with the highest discount factor (so the most patient or long-term focussed one) will end up being the "dictator", in that the optimal policy for the multi-agent system is one that is most preferred by that dictator, and such a policy therefore does not maximise social choice. 

Finally, there is a body of work to systematically study how Matthew Effect and path dependency can lead to one agent becoming dominant even in a multi-agent system with social-cost control, and what additional mechanisms in addition to VCG can be used to address that issue.

\paragraph{Acknowledgments}
We are grateful to Marcus Hutter and Jason Li for helpful comments on the paper.

\addcontentsline{toc}{section}{References}
\bibliographystyle{plain}
\bibliography{refs}

\newpage


\appendix

\section{Variations of the Social Cost Formulation}
\subsection{Notes on the Agent Valuation Function}\label{app:agent valuation function}

\subsubsection*{Expected Cumulative Utility}
We will start by unrolling (\ref{eq:expected total utility}) for $m = 3$ to see the general pattern.
Given the empty history $h_0 = \epsilon$, the expected utility for agent $i$ at time step 1 is
\begin{align*}
& \overline{v}_{1,i}(\epsilon) = 
 q_{1,i}(\epsilon, \overbar{a^*_1}(\epsilon)) - c_{1,i}(\epsilon, \overbar{a^*_1}(\epsilon))  \\ 
 &= \sum_{ \overbar{or_1}} \phi(\overbar{or_1} \,|\, \overbar{a^*_1}(\epsilon)) [r_{1,i} + \overline{q}_{2,i}( \overbar{a^*_1}(\epsilon) \overbar{or_1})] - c_{1,i}(\epsilon, \overbar{a^*_1}(\epsilon)) \\ 
 &= \sum_{ \overbar{or_1}} \phi(\overbar{or_1} \,|\, \overbar{a^*_1}(\epsilon)) [r_{1,i} + 
 q_{2,i}( \overbar{a^*_1}(\epsilon) \overbar{or_1}, \overbar{a^*_2}(\overbar{or_1}))  ] 
 - c_{1,i}(\epsilon, \overbar{a^*_1}(\epsilon)) \\ 
 &= 
 \sum_{ \overbar{or_1}} \phi(\overbar{or_1} \,|\, \overbar{a^*_1}(\epsilon)) 
\biggl[ r_{1,i} + 
   \sum_{ \overbar{or_2}} \phi(\overbar{or_2} \,|\, \overbar{a^*_1}(\epsilon) \overbar{or_1} \overbar{a^*_2}(\overbar{or_1})) [r_{2,i} + \overline{q}_{3,i}(\overbar{a^*_1}(\epsilon) \overbar{or_1} \overbar{a^*_2}(\overbar{or_1}) \overbar{or_2})]
 \biggr]   \\
 & \hspace{6em} - c_{1,i}(\epsilon, \overbar{a^*_1}(\epsilon)) \\
 &= \sum_{ \overbar{or_1}} \phi(\overbar{or_1} \,|\, \overbar{a^*_1}(\epsilon)) 
   \sum_{ \overbar{or_2}} \phi(\overbar{or_2} \,|\, \overbar{a^*_1}(\epsilon) \overbar{or_1} \overbar{a^*_2}(\overbar{or_1})) 
 [ r_{1,i} + r_{2,i} + \overline{q}_{3,i}(\overbar{a^*_1}(\epsilon) \overbar{or_1} \overbar{a^*_2}(\overbar{or_1}) \overbar{or_2})] \\
 & \hspace{6em} - c_{1,i}(\epsilon, \overbar{a^*_1}(\epsilon))  \\
 &= \sum_{ \overbar{or_{1:2}}} \phi(\overbar{or_{1:2}} \,|\, \overbar{a^*_1}(\epsilon), 
    \overbar{a^*_2}(\overbar{or_1})) 
\biggl[ r_{1,i} + r_{2,i} + 
   \sum_{ \overbar{or_3}} \phi(\overbar{or_3} \,|\, \overbar{a^*_1}(\epsilon) \overbar{or_1} \overbar{a^*_2}(\overbar{or_1}) \overbar{or_2} \overbar{a^*_3}(\overbar{or_{1:2}}) )\cdot r_{3,i} 
   \biggr]  \\
 & \hspace{6em} -  c_{1,i}(\epsilon, \overbar{a^*_1}(\epsilon))  \\
 &= \sum_{ \overbar{or_{1:3}}} \phi(\overbar{or_{1:3}} \,|\, \overbar{a^*_1}(\epsilon), 
    \overbar{a^*_2}(\overbar{or_1}), \overbar{a^*_3}(\overbar{or_{1:2}})) [ r_{1,i} + r_{2,i} + r_{3,i} ]\\
 & \hspace{6em} - \sum_{ \overbar{or_{1:2}}} \phi(\overbar{or_{1:2}} \,|\, \overbar{a^*_1}(\epsilon), 
    \overbar{a^*_2}(\overbar{or_1})) [ p_i(\epsilon) + p_i(\overbar{or_1}) + p_i(\overbar{or_{1:2}})  ], 
\end{align*}
where we denote $\overbar{v_{t}}(\overbar{h_{t-1}}) = (v_{t,1}(\overbar{h_{t-1}},\cdot), \ldots, v_{t,k}(\overbar{h_{t-1}},\cdot))$ and
\begin{align*}
&\overbar{a^*_1}(\epsilon) = f(\overbar{v}_1(\epsilon)) && p_i(\epsilon) = p_i(\overbar{v_1}(\epsilon)) \\
&\overbar{a^*_2}(\overbar{or_1}) = f(\overbar{v_{2}}(\overbar{a^*_1}(\epsilon) \overbar{or_1})) && p_i(\overbar{or_1}) = p_i(\overbar{v_{2}}(\overbar{a^*_1}(\epsilon) \overbar{or_1})) \\
&\overbar{a^*_3}(\overbar{or_{1:2}}) = f(\overbar{v_{3}}(\overbar{a^*_1}(\epsilon) \overbar{or_1} \overbar{a_2}^*(\overbar{or_1}) \overbar{or_2})) && p_i(\overbar{or_{1:2}}) = p_i(\overbar{v_{3}}(\overbar{a^*_1}(\epsilon) \overbar{or_1} \overbar{a^*_2}(\overbar{or_1}) \overbar{or_2})).
\end{align*}
For general $m$, one can use an induction argument to show 
\begin{align*}
 \overline{v}_{1,i}(\epsilon) &= \sum_{ \overbar{or_{1:m}}} \phi(\overbar{or_{1:m}} \,|\, \overbar{a^*_1}(\epsilon), 
    \overbar{a^*_2}(\overbar{or_1}), \ldots, \overbar{a^*_m}(\overbar{or_{1:(m-1)}})) \biggl[ \sum_{t=1}^m r_{t,i} \biggr]\\
 & \hspace{2em} - \sum_{ \overbar{or_{1:(m-1)}}} \phi(\overbar{or_{1:(m-1)}} \,|\, \overbar{a^*_1}(\epsilon),\ldots, 
    \overbar{a^*_{m-1}}(\overbar{or_{m-1}})) \biggl[ \sum_{t=0}^{m-1} p_i(\overbar{or_{1:t}}) \biggr]. 
\end{align*}
To avoid clutter, we can write the above as
\[ \overline{v}_{1,i}(\epsilon) = \mathbb{E}_{ \overbar{or_{1:m}}}  \biggl[ \sum_{t=1}^m r_{t,i} \biggr] -  \mathbb{E}_{ \overbar{or_{1:(m-1)}}} \biggl[ \sum_{t=0}^{m-1} p_i(\overbar{or_{1:t}}) \biggr]. \]

\subsubsection*{Self-Rational q-Functions as Valuation Functions}

Given full knowledge of the environment $\phi$ and the mechanism $(f,p_1,\ldots,p_k)$, the \emph{self-rational q-function} $\hat{q}_{t,i}$ of agent $i$ with horizon $m_i$ at time $t$ having seen history $\overbar{h_{t-1}}$ is defined similar to (\ref{eq:Bellman single agent}) as follows:
\begin{align*}
 & \hat{q}_{t,i}(\overbar{h_{t-1},\overbar{a_t}}) =
   \begin{cases}
       0 & t > m_i \\
       \sum\limits_{ \overbar{or_t}} \phi(\overbar{or_t} \,|\, \overbar{h_{t-1}a_t}) [r_{t,i} + \max\limits_{\overbar{a_{t+1}}} \hat{q}_{t+1,i}(\overbar{h_{t-1}aor_t},\overbar{a_{t+1}})]  & t \leq m_i 
   \end{cases}   \label{eq:agent self-rational q function} 
%
\end{align*}

Example~\ref{ex:protocol full example alt self-rational} below illustrates several problems with the use of self-rational q-functions as agent valuation functions. 

\begin{example}\label{ex:protocol full example alt self-rational}
Consider again the setup as described in Example~\ref{ex:protocol full example alt}.
Suppose each agent has horizon 2. Here are the self-rational q-functions of each agent:
    \begin{align*}
     & \hat{q}_{2,1}(a_1or_1,a_2) := \ite{a_1 = 1}{0}{ \ite{a_2 = 1}{100}{0}  } \\
     & \hat{q}_{2,2}(a_1or_1,a_2) := \ite{a_1 = 2}{0}{ \ite{a_2 = 2}{80}{0}  } \\
     & \hat{q}_{2,3}(a_1or_1,a_2) := \ite{a_2 = 3}{60}{0} \\
     & \hat{q}_{1,1}(a_1) := 100 \\
     & \hat{q}_{1,2}(a_1) := 80 \\
     & \hat{q}_{1,3}(a_1) := 60
    \end{align*}
    If the agents submit their self-rational q-functions truthfully, then at $t=1$, all three actions maximise social utility. 
    Suppose we break ties randomly, then Tables~\ref{tab:random choice 1 alt self-rational}-\ref{tab:random choice 3 alt self-rational} show the possible scenarios.
    Inside each cell, the numbers are the realised $\hat{q}_{t,i}$, $r_{t,i}$, and $p_{t,i}$ respectively.
    The last column shows the cumulative utility $\sum_{t=1}^2 r_{t,i} - p_{t,i}$ for each agent.

        \begin{table}[!htbp]
        \centering
        \begin{tabular}{c|c|c|c}
           \hline
                  & $a_1^* = 1$ & $a_2^* = 2$ & CU \\
           \hline       
           $A_1$  &  (100, 100, 0)  & (0,0,0) & 100 \\
           $A_2$  &  (80, 0, 0) &   (80,80,60) & 20 \\
           $A_3$  &  (60 0, 0) &  (0,0,0) & 0 \\
           \hline
        \end{tabular}
        \caption{Scenario when $a_1^*$ is randomly chosen to be 1}
        \label{tab:random choice 1 alt self-rational}
    \end{table}
    
    \begin{table}[!htbp]
        \centering
        \begin{tabular}{c|c|c|c}
           \hline
                  & $a_1^* = 2$ & $a_2^* = 1$ & CU \\
           \hline       
           $A_1$  &  (100, 0, 0)  & (100,100,60) & 40 \\
           $A_2$  &  (80, 80, 0) &   (0,0,0) & 80 \\
           $A_3$  &  (60, 0, 0) &  (0,0,0) & 0 \\
           \hline
        \end{tabular}
        \caption{Scenario when $a_1^*$ is randomly chosen to be 2}
        \label{tab:random choice 2 alt self-rational}
    \end{table}

    \begin{table}[!htbp]
        \centering
        \begin{tabular}{c|c|c|c}
           \hline
                  & $a_1^* = 3$ & $a_2^* = 1$ & TCU \\
           \hline       
           $A_1$  &  (100, 0, 0)  & (100,100,80) & 20 \\
           $A_2$  &  (80, 0, 0) &   (0,0,0) & 0 \\
           $A_3$  &  (60, 0, 0) &  (0,0,0) & 0 \\
           \hline
        \end{tabular}
        \caption{Scenario when $a_1^*$ is randomly chosen to be 3}
        \label{tab:random choice 3 alt self-rational}
    \end{table}
    \noindent Note the obviously suboptimal scenario shown in Table~\ref{tab:random choice 3 alt self-rational}.
    This example shows that the use of self-rational q-functions as agent valuation functions can lead to suboptimal cumulative social utility for all the agents, even when they report truthfully.  
    Further, the agents can manipulate the outcome in their favour by submitting an artificially higher valuation function at $t=1$ without incurring a higher payment. For example, $A_2$ can submit the following valuation
    \begin{equation}
        \widetilde{q}_{1,2}(a_1) := \ite{a_1 = 1}{ 80 }{ \ite{a_1 = 2}{81}{0}} 
    \end{equation}
    to get the outcome given in Table~\ref{tab:random choice 2 alt self-rational}, if $A_1$ and $A_3$ remain truthful.
    Of course, all the agents can choose not to be truthful, in which case they would need a reasonably good probabilistic model of how the other agent will lie to have a good chance of winning.

\end{example}

\subsubsection*{An Alternative Formulation of Rational q-Functions}
In this subsection, we investigate whether the following alternative definition of rational q-functions can be used as agent valuation functions.

\begin{defn}
Given full knowledge of the environment $\phi$ and the mechanism $(f,p_1,\ldots,p_k)$, for each agent $i$ with horizon $m_i$ at time $t$ having seen history $\overbar{h_{t-1}}$, the \emph{rational q-function} $q_{t,i}$ and \emph{rational social cost function} $c_{t,i}$ are defined inductively as follows:
\begin{align}
 q_{t,i}(\overbar{h_{t-1},\overbar{a_t}}) &=
   \begin{cases}
       0 & t > m_i \\
       \sum\limits_{ \overbar{or_t}} \phi(\overbar{or_t} \,|\, \overbar{h_{t-1}a_t}) [r_{t,i} + \overline{q}_{t+1,i}(\overbar{h_{t-1}}\overbar{a_t}\overbar{or_t})]  & t \leq m_i 
   \end{cases}   \label{eq:agent q function q-based} \\
 \overline{q}_{t,i}(\overbar{h_{t-1}}) &= q_{t,i}(\overbar{h_{t-1}}, f( \overbar{q_{t}}))  \label{eq:agent bar q function q-based} \\ 
 c_{t,i}(\overbar{h_{t-1},\overbar{a_t}}) &= 
   \begin{cases}
       0 & t \geq m_i \\
       \sum\limits_{ \overbar{or_t}} \phi(\overbar{or_t} \,|\, \overbar{h_{t-1}}\overbar{a_t}) [ \overline{c}_{t+1,i}(\overbar{h_{t-1}}\overbar{a_t}\overbar{or_t}) ]  & t < m_i 
   \end{cases}   \label{eq:agent cost function q-based} \\
  \overline{c}_{t,i}(\overbar{h_{t-1}}) &= p_i(\overbar{q_t}) + c_{t,i}(\overbar{h_{t-1}}, f( \overbar{q_{t}}))  \label{eq:agent bar cost function q-based} 
\end{align}
where $\overbar{q_t} := (q_{t,1}(\overbar{h_{t-1}},\cdot), \ldots, q_{t,k}(\overbar{h_{t-1}},\cdot))$ and $f(\overbar{q_t}) = \arg \max\limits_{\overbar{a}} \sum\limits_{j} q_{t,j}(\overbar{h_{t-1}},\overbar{a})$.
\end{defn}
Note that unlike Definition~\ref{defn:agent valuation function}, the argument to $f()$ and $p_i()$ above are $\overbar{q_t}$ rather than $\overbar{v_t}$.
In the setting where the mechanism is one of the agents with valuation function equal to total payments received, and the total social welfare includes the utility of the mechanism agent, all the payment terms cancel out and the formulation above should still yield the socially optimal outcome.

%
While Theorem~\ref{thm:vcg incentive compatible} can be used in a straightforward manner to show the interaction protocol $M \triangleright \phi$ is incentive compatible with respect to each agent's rational q-function, we show next that incentive compatibility with respect to the realisable cumulative utility cannot be achieved without some additional restrictions.
Let $M \triangleright \phi$ be the interaction protocol and suppose each agent's true valuation function is their rational q-function. 
Let $\overbar{h_{t-1}}$ be the history at time $t$.
Fix an agent $i$ and 
let $q_{t,-i}$ be the submitted valuation functions of the other agents.
Agent $i$ can choose to submit its true valuation function $q_{t,i}$ or some other arbitrary function $\widetilde{q}_{t,i}$.
If it submits $q_{t,i}$, then the protocol picks the action $\overbar{a_t} := \arg \max_{\overbar{a}} \sum_j q_{t,j}(\overbar{h_{t-1}}, \overbar{a})$ and agent $i$'s expected realisable cumulative utility from the protocol is
\begin{align}
& 
  \mathbb{E}_{\overbar{or_t}} \biggl[ r_{t,i} - p_i(q_{t,i},q_{t,-i}) + \overline{q}_{t+1,i}(\overbar{h_{t-1}}\overbar{a_t}\overbar{or_t}) - \overline{c}_{t+1,i}(\overbar{h_{t-1}a_t or_t}) \biggr] \nonumber \\ 
= & 
   \mathbb{E}_{\overbar{or_t}} \biggl[ r_{t,i}  + \overline{q}_{t+1,i}(\overbar{h_{t-1}}\overbar{a_t}\overbar{or_t}) \biggr] - h_i(q_{t,-i}) + \sum_{j \neq i} q_{t,j}(\overbar{h_{t-1}},\overbar{a_t})  \nonumber \\ 
 & \hspace{18em} - \mathbb{E}_{\overbar{or_t}} \overline{c}_{t+1,i}(\overbar{h_{t-1}a_t or_t}) \nonumber \\
= & \sum_{j} q_{t,j}(\overbar{h_{t-1}},\overbar{a_t}) - h_i(q_{t,-i}) - \mathbb{E}_{\overbar{or_t}} \overline{c}_{t+1,i}(\overbar{h_{t-1}a_t or_t}), \label{eq:truth utility}
\end{align}
where $h_i(q_{t,-i})$ is the Clark pivot payment function.
If agent $i$ submits $\widetilde{v}_{t,i}$ and letting $\overbar{b_t} := \arg \max_{\overbar{a}} \bigl[ \widetilde{q}_{t,i}(\overbar{a}) + \sum_{j \neq i} q_{t,j}(\overbar{h_{t-1}},\overbar{a}) \bigr]$, we can similarly show that agent $i$'s expected realisable cumulative utility is
\begin{equation}\label{eq:lie utility}
  \sum_{j} q_{t,j}(\overbar{h_{t-1}},\overbar{b_t}) - h_i(q_{t,-i}) - \mathbb{E}_{\overbar{or_t}} \overline{c}_{t+1,i}(\overbar{h_{t-1}b_t or_t}).
\end{equation}
While the first term of (\ref{eq:truth utility}) is larger or equal to the first term of (\ref{eq:lie utility}) by definition of $\overbar{a_t}$, it may be possible for agent $i$ to submit a $\widetilde{q}_{t,i}$ so that the last term of (\ref{eq:lie utility}) is sufficiently small to offset that and achieve (\ref{eq:lie utility}) > (\ref{eq:truth utility}).

Example~\ref{ex:protocol full example} below illustrates how agents can play strategically -- i.e. lie about the true rational q-functions -- to obtain an edge over other agents.

\begin{example}\label{ex:protocol full example}
Consider the same problem setup as Example~\ref{ex:protocol full example alt}.
    Suppose each agent has horizon 2 and their true valuation function are their rational q-functions as given earlier in Example~\ref{ex:protocol full example alt}. 
    If all three agents submit their rational q-functions truthfully, then at $t=1$ there are two actions that both maximise social utility: $\{1,2\}$. 
    Suppose we break ties randomly, then Tables~\ref{tab:random choice 1} and \ref{tab:random choice 2} show the two possible scenarios.
    Inside each cell, the four numbers are the realised $q_{t,i}$, $c_{t,i}$, $r_{t,i}$, and $p_{t,i}$ values respectively.
    The last column shows the cumulative utility $\sum_{t=1}^2 r_{t,i} - p_{t,i}$ for each agent. 

        \begin{table}[!htbp]
        \centering
        \begin{tabular}{c|c|c|c}
           \hline
                  & $a_1^* = 1$ & $a_2^* = 2$ & CU \\
           \hline       
           $A_1$  &  (100, 0, 100, 0)  & (0,0,0,0) & 100 \\
           $A_2$  &  (80, 60, 0, 0) &   (80,0,80,60) & 20 \\
           $A_3$  &  (0, 0, 0, 0) &  (0,0,0,0) & 0 \\
           \hline
        \end{tabular}
        \caption{Scenario when $a_1^*$ is randomly chosen to be 1}
        \label{tab:random choice 1}
    \end{table}
    
    \begin{table}[!htbp]
        \centering
        \begin{tabular}{c|c|c|c}
           \hline
                  & $a_1^* = 2$ & $a_2^* = 1$ & CU \\
           \hline       
           $A_1$  &  (100, 60, 0, 0)  & (100,0,100,60) & 40 \\
           $A_2$  &  (80, 0, 80, 0) &   (0,0,0,0) & 80 \\
           $A_3$  &  (0, 0, 0, 0) &  (0,0,0,0) & 0 \\
           \hline
        \end{tabular}
        \caption{Scenario when $a_1^*$ is randomly chosen to be 2}
        \label{tab:random choice 2}
    \end{table}    
    Note that $A_1$ and $A_2$ get different cumulative utility depending on the random choice -- whoever gets to consume the product first gets more cumulative utility by avoiding having to compete with $A_3$ at $t=2$ -- but the sum of their total cumulative utility remains the same in both scenarios.
    Both $A_1$ and $A_2$ can manipulate the outcome in their favour by submitting an artificially higher valuation function at $t=1$ without incurring a higher payment. For example, $A_2$ can submit the following valuation
    \begin{equation}
        \widetilde{q}_{1,2}(a_1) := \ite{a_1 = 1}{ 80 }{ \ite{a_1 = 2}{81}{0}} 
    \end{equation}
    to get the outcome given in Table~\ref{tab:random choice 2}, if $A_1$ remains truthful.
    Of course, both $A_1$ and $A_2$ can choose to not be truthful, in which case they would need a reasonably good probabilistic model of how the other agent will lie to have a good chance of winning.

\end{example}

Note that in both Example~\ref{ex:protocol full example alt self-rational} and \ref{ex:protocol full example}, the payments made by the agents are strictly less than that paid by the agents in Example~\ref{ex:protocol full example alt}.
The extra payment appears to be the price we need to pay to achieve Bayes-Nash incentive compatibility in the setup of Example~\ref{ex:protocol full example alt}.

In practical platform economies, it is possible to obtain approximate $\epsilon$-incentive compatibility results by restricting the type of valuation-function misreports that agents can do, either through explicit rules or by assumption; see \cite{parkes2007} for a survey of key ideas.

\subsection{Guaranteed Utility Mechanism}\label{app:gum}
We show in this section an adaptation of the Guaranteed Utility Mechanism (GUM) proposed in \cite{csoka2024collusion} for our setting.
Suppose we have a multi-agent environment $\phi$ with $k$ agents and a mechanism $M = (f,p_1,\ldots,p_k)$. 
We assume both the principal (i.e. the mechanism designer) and the agents have full knowledge of $\phi$, or are at least capable of learning an approximation of $\phi$ from data.

\begin{defn}
Given full knowledge of the environment $\phi$ and the mechanism $(f,p_1,\ldots,p_k)$, for each agent $i$ with horizon $T$ at time $t$ having seen history $\overbar{h_{t-1}}$, the \emph{q-function} $q_{t}^{i}$ is defined inductively as follows:
\begin{align}
 q_{t}^{i}(\overbar{h_{t-1}}, \overbar{a_t}) &=
   \begin{cases}
       0 & t > T \\
       \sum_{ or_t} \phi( or_t \,|\, \overbar{h_{t-1}} \overbar{a_t} ) \bigl[ r_{t}^{i} + \overline{q}_{t+1}^{i}(\overbar{h_{t-1}} \overbar{a_t or_t}) \bigr]  & t \leq T 
   \end{cases}   \label{eq:agent q function q-based 2} \\
 \overline{q}_{t}^{i}(\overbar{h_{t-1}}) &= q_{t}^{i}(\overbar{h_{t-1}}, f( \overbar{q_{t}^{1:k}}))  \label{eq:agent bar q function q-based 2} 
\end{align}
where $\overbar{q_t^{1:k}} := (q_{t}^{1}(\overbar{h_{t-1}},\cdot), \ldots, q_{t}^{k}(\overbar{h_{t-1}},\cdot))$ and $f(\overbar{q_t^{1:k}}) = \arg \max\limits_{\overbar{a}} \sum\limits_{j} q_{t}^{j}(\overbar{h_{t-1}}, \overbar{a})$.
\end{defn}

In the following, for all time $t \geq 1$ and for all $i \in \{1,\ldots,k\}$, we denote the partial history $h_t^{1:i}$ by
\[ \overbar{h_{t}^{1:i}} := \overbar{h_{t-1}a_{t}or_{t,1:i}} = \overbar{h_{t-1}a_{t}}or_{t,1} or_{t,2} \ldots or_{t,i}, \]
with the special cases of $\overbar{h_{t}^{1:0}} := \overbar{h_{t-1}a_{t}}$ and $\overbar{h_0^{1:i}} := \epsilon$.
Note that $\overbar{h_t^{1:k}} = \overbar{h_t}$.

\begin{defn}
Let $T$ be the horizon.
For all $t \in \{1,\ldots,T+1\}$, $i \in \{1,\ldots,k\}$, partial history $\overbar{h_{t-1}^{1:i}}$, and agent valuation functions $\overbar{q_t^{1:k}}$, we define 
the \emph{principal's anticipated cumulative payoff} $\Upsilon_t^j(\cdot, \cdot)$ for agent $j$ from time $t$ onwards as follows: 
\begin{align*}
 & \Upsilon_{t}^j(\overbar{h_{t-1}^{1:i}}, \overbar{q_t^{1:k}}) = 
 \mathbb{E}_{ \overbar{or_{t-1,(i+1):k}} \,\sim \phi} \biggl[ \sum_{s=1}^{t-1} r_s^j + q_t^j(\overbar{h_{t-1}}, 
                                            f(\overbar{q_t^{1:k}}(\overbar{h_{t-1}}, \cdot))) \biggr] \\ 
& \Upsilon_{T+1}^j(\overbar{h_{T}^{1:i}}, \epsilon) = \mathbb{E}_{\overbar{or_{T,(i+1):k}}\, \sim \phi} \biggl[ \sum_{s=1}^T r_s^j \biggr] \\                                       
   & \Upsilon_1^j(\epsilon, \overbar{q_1^{1:k}}) = q_1^j( \epsilon, f(q_1^1(\epsilon,\cdot), \ldots, q_1^k(\epsilon,\cdot)  ),
\end{align*}
where $\epsilon$ is the empty sequence and $q_t^{1:k}(\overbar{h_{t-1}},\cdot) := ( q_t^1(\overbar{h_{t-1}},\cdot),\ldots, q_t^k(\overbar{h_{t-1}}, \cdot)  )$.
\end{defn}

Note that, by the above two definitions, we have
\begin{equation}\label{eq:Upsilon edge} 
\Upsilon_t^j(\overbar{h_{t-2} a_{t-1}}, q_t^{1:k}) = \Upsilon_{t-1}^j(\overbar{h_{t-2}^{1:k}}, q_{t-1}^{1:k}), 
\end{equation}
since
\begin{align*}
 \Upsilon_{t-1}^j(\overbar{h_{t-2}^{1:k}}, \overbar{q_{t-1}^{1:k}}) = &  \sum_{s=1}^{t-2} r_s^j + q_{t-1}^j( \overbar{h_{t-2}}, f(\overbar{q_{t-1}^{1:k}}(\overbar{h_{t-2}},\cdot)) \\
= &  \sum_{s=1}^{t-2} r_s^j + \mathbb{E}_{\overbar{or_{t-1}}} \bigl[ r_{t-1}^j + \overline{q}_{t}^j( \overbar{h_{t-2} a_{t-1} or_{t-1}}) \bigr] \\
= & \; \mathbb{E}_{\overbar{or_{t-1}}} \biggl[ \sum_{s=1}^{t-1} r_{s}^j + \bigl[ q_{t}^j( \overbar{h_{t-1}}, f(\overbar{q_{t}^{1:k}}(\overbar{h_{t-1}},\cdot))) \bigr] \biggr] \\
 = & \; \Upsilon_t^j(\overbar{h_{t-1}^{1:0}}, \overbar{q_t^{1:k}}) \\
 = & \; \Upsilon_t^j(\overbar{h_{t-2}a_{t-1}}, \overbar{q_t^{1:k}}),
\end{align*}
where we have denoted $\overbar{a_s} := f(\overbar{q_s^{1:k}}(\overbar{h_{s-1}},\cdot))$.

\paragraph{GRL-GUM Interaction Protocol}
We now describe the GRL-GUM interaction protocol between the mechanism and the agents.
Let $\overbar{h_{t-1}}$ be the history up till time $t$.
At time $t$, each agent $i$ submits a report $\tilde{q}_{t}^{i}(\overbar{h_{t-1}},\cdot)$.
We then use the mechanism to determine the joint action the agents should take to maximise social welfare via
\begin{equation}\label{eq:max welfare GUM} 
\overbar{a^*_t} := f(\overbar{\tilde{q}_t^{1:k}}) = \arg \max_{a \in \Alt} \sum_{i} \tilde{q}_{t}^{i}(\overbar{h_{t-1}}, a). 
\end{equation}
A percept $\overbar{or_t}$ is then sampled from $\phi( \cdot \,|\, \overbar{h_{t-1} a^*_t})$ and
each agent $i$ receives the percept and make the following transfer of payments to each other as determined by the principal:
\begin{align}\label{eq:payment function GUM} 
p_t^i &:= \sum_{j \neq i} \gamma_t^{i \to j} - \sum_{j \neq i} \gamma_t^{j \to i} 
\end{align}
where 
\begin{multline}\label{eq:gamma i to j}
 \gamma_t^{i \to j} = \Upsilon_{t}^j( \overbar{h_{t-1}^{1:i}}, (\tilde{q}_t^1,\ldots,\tilde{q}_t^i,q_t^{i+1}, \ldots, q_t^k)) \\ - \Upsilon_{t}^j(\overbar{h_{t-1}^{1:(i-1)}}, (\tilde{q}_t^1,\ldots,\tilde{q}_t^{i-1},q_t^{i}, \ldots, q_t^k)) 
\end{multline}
is the payment that agent $i$ makes to agent $j$.
Each $q_t^i$ in (\ref{eq:gamma i to j}) is as defined in (\ref{eq:agent q function q-based 2}) with argument $\overbar{h_{t-1}}$. (If the principal does not have knowledge of $\phi$, then it needs to estimate $q_t^i$ from data.)  
Thus, $\gamma_t^{i \to j}$ can be thought of as the social cost that agent $i$ imposes on agent $j$, and so the payment $p_t^i$ is the sum of agent $i$'s social costs on the other agents and the other agent's social cost on agent $i$. 

The instantaneous utility of agent $i$ at time $t$ is then given by 
$r_{t}^{i} + p_t^i$. 
The total utility of agent $i$ is given by
\[ U^i = \sum_{t=1}^{T} r_{t}^i + p_t^i, \] 
which is a random variable dependent on $\phi$ and the sequence of valuation functions $\overbar{\tilde{q}_{1:T}^{1:k}}$ submitted by the agents. 


\begin{prop}\label{prop:GUM budget balanced}
The payment function (\ref{eq:payment function GUM}) is budget balanced: $\sum_i p_t^i = 0$.    
\end{prop}
\begin{proof}
The terms in the two sums in (\ref{eq:payment function GUM}) cancel out.   
\end{proof}

\paragraph{Guaranteed Utility Property}
\begin{defn}[\cite{csoka2024collusion}]\label{defn:GUP}
An interaction protocol $M \triangleright \phi$ satisfies the Guaranteed Utility Property (GUP) if there exists agent valuation functions $(q_{1:T}^{*,1}, q_{1:T}^{*,2}, \ldots, q_{1:T}^{*,k})$ and $C^1,\ldots, C^k \in \mathbb{R}$ such that
\begin{enumerate}
 \item For each agent $i$, we have $\forall \overbar{q^{-i}_{1:T}}. \; \mathbb{E}_{M \triangleright \phi} \bigl[ U^i(q^{*,i}_{1:T}, \overbar{q^{-i}_{1:T}}) \bigr] \geq C^i.$ 
 \item The best possible outcome is given by $\sup\limits_{\overbar{q_{1:T}^{1:k}}} \sum_i \mathbb{E}_{M \triangleright \phi} \bigl[ U^i(\overbar{q_{1:T}^{1:k}}) \bigr] = \sum_i C^i.$
\end{enumerate}
\end{defn}

\begin{theorem}[\cite{csoka2024collusion}]\label{thm:GUP IC}
An interaction protocol $M \triangleright \phi$ satisfies the Guaranteed Utility Property if and only if
\[ \sum_i \sup_{q^i_{1:T}} \inf_{ \overbar{q^{-i}_{1:T}}} \mathbb{E}_{M \triangleright \phi} \bigl[ U^i(q^i_{1:T}, \overbar{q^{-i}_{1:T}}) \bigr] = \sup_{\overbar{q^{1:k}_{1:T}}} \sum_i \mathbb{E}_{M \triangleright \phi} \bigl[ U^i(\overbar{q^{1:k}_{1:T}}) \bigr] \]
\end{theorem}

Theorem~\ref{thm:GUP IC} shows that in an interaction protocol that satisfies the GUP, there is no incentive for any agent to do anything other than truthfully report their valuation functions.
The following result can be established through a direct application of the proof technique given in \cite{csoka2024collusion}.

\begin{theorem}
The GRL-GUM interaction protocol satisfies the Guaranteed Utility Property.    
\end{theorem}
\begin{proof}
We show that the q-functions $\overbar{q_{1:T}^{1:k}}$ as defined in (\ref{eq:agent q function q-based 2}), in conjunction with setting $C^i = \Upsilon_{1}^i(\epsilon, \overbar{q_1^{1:k}})$, satisfy the two properties in Definition~\ref{defn:GUP}. 

\paragraph{Part 1}
Fix an arbitrary agent $i$ and an arbitrary sequence of valuation functions $\overbar{\tilde{q}_{1:T}^{-i}}$ for the other agents.
We need to show
\begin{equation}\label{eq:GUM theorem part 1} 
\mathbb{E}_{M \triangleright \phi} \bigl[ U^i(q^{i}_{1:T}, \overbar{\tilde{q}^{-i}_{1:T}}) \bigr] = \Upsilon_{1}^i(\epsilon, \overbar{q_1^{1:k}}). 
\end{equation}
For convenience, we split each transfer $p_t^i$ as defined in (\ref{eq:payment function GUM}) as follows:
\begin{equation}\label{eq:gum transfer parts}
    p_{t,j}^i = \begin{cases}
                 -\gamma_t^{j \to i} & \text{if } j \neq i \\
                 \sum_{l \neq i} \gamma_t^{i \to l} & \text{if } i = j.
                \end{cases}
\end{equation}
We first show that the sum of the anticipated payoff $\Upsilon_t^i$ and the cumulative transfers of agent $i$ is a martingale; i.e. for all $t \geq 1$ and $j \in \{1,\ldots,k\}$
\begin{multline}\label{eq:GUM theorem martingale} 
\mathbb{E}_{or_{t-1,j} \sim \phi} \biggl[ \Upsilon_{t}^i(\overbar{h_{t-1}^{1:j}}, (\overbar{\tilde{q}_t^{1:j}}, \overbar{q_t^{(j+1):k}}) + \sum_{s=1}^t \sum_{l=1}^j p_{s,l}^i \biggr] = \\
 \Upsilon_{t}^i(\overbar{h_{t-1}^{1:(j-1)}}, ( \overbar{\tilde{q}_t^{1:(j-1)}}, \overbar{q_t^{j:k}} ) + \sum_{s=1}^t \sum_{l=1}^{j-1} p_{s,l}^i, 
\end{multline}
which is equivalent to the following simpler expression by rearranging terms
\[ \mathbb{E}_{or_{t-1,j} \sim \phi} \bigl[  \gamma_t^{j \to i} + p_{t,j}^i \bigr] = 0. \]
There are two cases to consider.
If $j \neq i$, then $\gamma_t^{j \to i} + p_{t,j}^i = 0$ by (\ref{eq:gum transfer parts}).
For the case of $j = i$, we have
\[ \gamma_t^{j \to i} + p_{t,j}^i = \gamma_t^{i \to i} + \sum_{l \neq i} \gamma_t^{i \to l} = \sum_{l} \gamma_t^{i \to l}. \]
To show $\mathbb{E} \bigl[ \sum_{l} \gamma_t^{i \to l} \bigr] = 0$, it is sufficient to show $\mathbb{E} \bigl[ \gamma_t^{i \to l} \bigr] = 0$ for each agent $l$, which follows from agent $i$ being truthful in reporting $\tilde{q}_t^i = q_t^i$ and the law of iterated expectations:
\begin{align*}
   &\; \mathbb{E}_{or_{t-1,i} \sim \phi} \bigl[ \Upsilon_{t}^l( \overbar{h_{t-1}^{1:i}}, (\tilde{q}_t^1,\ldots, \tilde{q}_t^{i-1},\tilde{q}_t^i,q_t^{i+1}, \ldots, q_t^k)) \bigr]  \\
= &\;\mathbb{E}_{or_{t-1,i} \sim \phi} \bigl[ \Upsilon_{t}^l( \overbar{h_{t-1}^{1:i}}, (\tilde{q}_t^1,\ldots, \tilde{q}_t^{i-1}, q_t^i,q_t^{i+1}, \ldots, q_t^k)) \bigr] \\
 = &\; \Upsilon_{t}^l(\overbar{h_{t-1}^{1:(i-1)}}, (\tilde{q}_t^1,\ldots,\tilde{q}_t^{i-1},q_t^{i}, q_t^{i+1},\ldots, q_t^k)). 
\end{align*}
For the special case of $t = T+1$, we have
\begin{align*}
  \mathbb{E}_{or_{T,i} \sim \phi} \bigl[ \Upsilon_{T+1}^l( \overbar{h_{T}^{1:i}}, \epsilon) \bigr]  =  \Upsilon_{T+1}^l(\overbar{h_{T}^{1:(i-1)}}, \epsilon), 
\end{align*}
which follows from the definition of $\Upsilon_{T+1}^j$. 
We have thus established the martingale property (\ref{eq:GUM theorem martingale}).

We now show that (\ref{eq:GUM theorem part 1}) can be obtained from repeated applications of (\ref{eq:GUM theorem martingale}) and (\ref{eq:Upsilon edge}) as follows, where we denote by
\[ \overbar{or_{1:t}^{1:j}} := \overbar{or_{1:(t-1)}or_t^{1:j}} = \overbar{or_{1:(t-1)}}or_{t,1}\ldots or_{t,j} \] 
the (partial) percepts sampled from $\phi$. 
\allowdisplaybreaks
\begin{align*}
 & \;\mathbb{E}_{M \triangleright \phi} \bigl[ U^i(q^{i}_{1:T}, \overbar{\tilde{q}^{-i}_{1:T}}) \bigr] \\
= & \; \mathbb{E}_{\overbar{or_{1:T}^{1:k}}} \biggl[ \Upsilon_{T+1}^i(\overbar{h_{T}^{1:k}}, \epsilon) + \sum_{t=1}^{T+1}\sum_{l=1}^k p_{t,l}^i \biggr] \\
= & \; \mathbb{E}_{\overbar{or_{1:T-1}^{1:k}}} \mathbb{E}_{\overbar{or_T^{1:k-1}}} \biggl[ \Upsilon_{T+1}^i(\overbar{h_{T}^{1:(k-1)}}, \epsilon) + \sum_{t=1}^{T+1}\sum_{l=1}^{k-1} p_{t,l}^i \biggr] \\
& \vdots \\
= & \; \mathbb{E}_{\overbar{or_{1:T-1}^{1:k}}} \biggl[ \Upsilon_{T+1}^i(\overbar{h_{T-1}a_T}, \epsilon) + \sum_{t=1}^{T}  p_{t}^i \biggr] \\
= & \;\mathbb{E}_{\overbar{or_{1:(T-1)}^{1:k}}} \biggl[ \Upsilon_{T}^i(\overbar{h_{T-1}^{1:k}}, (q_T^i, \overbar{\tilde{q}_T^{-i}}) + \sum_{t=1}^{T} p_t^i \biggr]   \\
= & \; \mathbb{E}_{\overbar{or_{1:(T-2)}^{1:k}}} \mathbb{E}_{\overbar{or_{T-1}^{1:(k-1)}}} \biggl[ \Upsilon_T^i(\overbar{h_{T-1}^{1:(k-1)}},(\overbar{\tilde{q}_T^{1:(k-1)}}, q_T^k)) + \sum_{t=1}^{T} \sum_{l=1}^{k-1} p_{t,l}^i  \biggr] \\
= & \; \mathbb{E}_{\overbar{or_{1:(T-2)}^{1:k}}} \mathbb{E}_{\overbar{or_{T-1}^{1:(k-2)}}} \biggl[ \Upsilon_T^i(\overbar{h_{T-1}^{1:(k-2)}},(\overbar{\tilde{q}_T^{1:(k-2)}}, \overbar{q_T^{(k-1):k}})) + \sum_{t=1}^{T} \sum_{l=1}^{k-2} p_{t,l}^i  \biggr] \\
& \vdots \\
= & \; \mathbb{E}_{\overbar{or_{1:(T-2)}^{1:k}}} \biggl[ \Upsilon_T^i(\overbar{h_{T-2}a_{T-1}}, \overbar{q_T^{1:k}}) + \sum_{t=1}^{T-1} p_{t}^i  \biggr] \\
= & \; \mathbb{E}_{\overbar{or_{1:(T-2)}^{1:k}}} \biggl[ \Upsilon_{T-1}^i(\overbar{h_{T-2}^{1:k}}, (q_{T-1}^i, \overbar{\tilde{q}_{T-1}^{-i}})) + \sum_{t=1}^{T-1} p_{t}^i  \biggr]   \\
& \vdots \\
= & \; \mathbb{E}_{\overbar{or_1^{1:k}}} \biggl[ \Upsilon_{2}^i(\overbar{h_{1}^{1:k}}, (q_{2}^i, \overbar{\tilde{q}_{2}^{-i}})) + \sum_{t=1}^{2} p_{t}^i  \biggr]   \\
& \vdots \\
= & \; \Upsilon_{2}^i(\overbar{a_1}, \overbar{q_{2}^{1:k}}) + p_{1}^i \\
= & \; \Upsilon_1^i(\epsilon, \overbar{q_1^{1:k}}).
\end{align*}
The last step follows since $p_1^i = 0$.

\paragraph{Part 2} The optimal outcome is achieved when all the agents are truthful, and the result then follows from Part 1.
\end{proof}

We remark also that the same argument used in \cite{csoka2024collusion} can be used to show that GRL-GUM is collusion-proof, a property that the VCG mechanism does not satisfy.

Note also that, unlike the interaction protocol described in \S~\ref{subsec:general case} where only the agents will need to learn estimates of their valuation functions $q_t^i$ from data in practice, the GRL-GUM mechanism requires that both the principal and the agents learn estimates of each agent's valuation function $q_t^i$ from data when the underlying environment $\phi$ is not common knowledge.

\newpage
\section{Cap and Trade Agent Policies}\label{ap:cap and trade}

\paragraph{Collusive Dynamics}
Fig~\ref{fig:r1_ex1_pol_rfunc2} shows the final policy learned for the setup of Fig~\ref{fig:r1_ex1_rfunc2}.
\begin{figure}[!htbp]
\begin{scriptsize}
\begin{verbatim}
=== FINAL POLICY EVALUATION ===
--------------------------------------------------------------------
|  t  | Agent | Prod |  Perm  | Prof | Rho  | Bid  |  Win? | +Prof |
--------------------------------------------------------------------
|  1  |   1   |  0.0 |      0 |  0.0 |  8.4 | 0.00 |       |       |
|     |   2   |  0.0 |      0 |  0.0 |  6.1 | 0.00 |   *   |   6.1 |
--------------------------------------------------------------------
|  2  |   1   |  0.0 |      0 |  0.0 |  8.4 | 0.10 |   *   |   8.4 |
|     |   2   | 30.5 |  3,000 |  6.1 |  2.3 | 0.00 |       |       |
--------------------------------------------------------------------
|  3  |   1   | 42.2 |  3,000 |  8.4 |  1.6 | 0.00 |       |       |
|     |   2   | 30.5 |  3,000 |  6.1 |  2.3 | 0.05 |   *   |   2.3 |
--------------------------------------------------------------------
|  4  |   1   | 42.2 |  3,000 |  8.4 |  1.6 | 0.20 |   *   |   1.6 |
|     |   2   | 42.2 |  6,000 |  8.4 |  0.9 | 0.00 |       |       |
--------------------------------------------------------------------
|  5  |   1   | 50.2 |  6,000 | 10.0 |  1.0 | 0.05 |   *   |   1.0 |
|     |   2   | 42.2 |  6,000 |  8.4 |  0.9 | 0.00 |       |       |
--------------------------------------------------------------------
(winning agent must pay less than rho to make a profit)

Avg permit price: $0.0
Agent 1:
  -> won 9000 permits
  -> paid $0.00m
  -> produced 55.2m litres of fuel
  -> made a total profit of $11.04m

Agent 2:
  -> won 6000 permits
  -> paid $0.00m
  -> produced 42.2m litres of fuel
  -> made a total profit of $8.45m

====== ESTIMATED OPTIMAL STRATEGIES: argmax_a Q(s, a) ======

==============================================================================================
        ||      0      |     3,000   |    6,000    |    9,000    |    12,000   |   15,000    |
==============================================================================================
      0 || 0.00 / 0.00 | 0.10 / 0.00 | 0.35 / 0.00 | 0.35 / 0.05 | 8.35 / 0.00 | 0.0* / 0.0* |
  3,000 || 0.00 / 0.10 | 0.00 / 0.05 | 0.20 / 0.00 | 0.10 / 0.00 | 0.0* / 0.0* |  ?   /  ?   |
  6,000 || 0.00 / 0.05 | 0.00 / 0.05 | 0.05 / 0.00 | 0.0* / 0.0* |  ?   /  ?   |  ?   /  ?   |
  9,000 || 0.00 / 0.40 | 0.05 / 0.15 | 0.0* / 0.0* |  ?   /  ?   |  ?   /  ?   |  ?   /  ?   |
 12,000 || 0.25 / 4.45 | 0.0* / 0.0* |   ?  /  ?   |  ?   /  ?   |  ?   /  ?   |  ?   /  ?   |
 15,000 || 0.0* / 0.0* |  ?   /   ?  |   ?  /  ?   |  ?   /  ?   |  ?   /  ?   |  ?   /  ?   |


====== ESTIMATED RETURNS IF OPTIMAL STRATEGY FOLLOWED: max Q(s, a) ======
  
=====================================================================================================
        ||      0        |    3,000      |    6,000      |    9,000    |    12,000   |    15,000    |
=====================================================================================================
      0 || 19.49 / 19.49 | 13.38 / 13.38 | 11.03 / 11.04 | 9.92 / 9.99 | 7.90 / 8.45 | 0.00 / 0.00  |
  3,000 || 11.04 / 11.04 |  4.94 / 4.94  |  2.59 / 2.59  | 1.59 / 1.60 | 0.00 / 0.00 |  ?   /  ?    |
  6,000 ||  9.44 / 9.44  |  3.34 / 3.34  |  0.99 / 0.99  | 0.00 / 0.00 |  ?   /  ?   |  ?   /  ?    |
  9,000 ||  8.40 / 8.27  |  2.30 / 2.28  |  0.00 / 0.00  |  ?   /  ?   |  ?   /  ?   |  ?   /  ?    |
 12,000 ||  5.85 / 5.84  |  0.00 / 0.00  |   ?   /  ?    |  ?   /  ?   |  ?   /  ?   |  ?   /  ?    |
 15,000 ||  0.00 / 0.00  |   ?   /  ?    |   ?   /  ?    |  ?   /  ?   |  ?   /  ?   |  ?   /  ?    |
\end{verbatim}
\end{scriptsize}
    \caption{Final policy for a sample run of $r^{(2)}$. }
    \label{fig:r1_ex1_pol_rfunc2}
\end{figure}
In the Estimated Optimal Strategies matrix in Fig~\ref{fig:r1_ex1_pol_rfunc2}, rows denote permits held by $R_{1}$ and columns denote permits held by $R_{2}$, so that each cell represents a state of the game. The entries in the cells are (estimated optimal bid in that state by $R_{1}$, estimated optimal bid by $R{2}$). A `?' means the state was never visited by the agent, while an `*' marks the terminal states.
    
On the first tranche, the agents both bid zero, leaving it to the VCG mechanism to decide via random tie-breaking which agent wins. Then, if $R_{1}$ won the first tranche -- i.e. we are now in cell (2, 1) of the matrix -- $R_{1}$ bids zero so $R_{2}$ will win the second tranche. And vice versa if we are in cell (1, 2) of the matrix. After that, on tranche 3, we are always in cell (2, 2) of the matrix and the remainder of the game always proceeds the same way: agent $R_{2}$ wins; agent $R_{1}$ wins; agent $R_{1}$ wins.
    
Thus the learned joint policy always results in a return for each agent (and a total shared net profit) of $19.49$ over the five tranches, as both agents have accurately estimated -- see cell (1, 1) of the ``Estimated Returns'' matrix of Fig~\ref{fig:r1_ex1_pol_rfunc2}. The entries in that matrix are ($R_{1}$'s estimated expected return from the state if $R_{1}$'s estimated optimal bids are made, $R_{2}$'s estimated expected return if its estimated optimal bids are made). Note the agents have learned this collusive policy solely by observing the permit holdings of each participant in the auction -- there was never any explicit communication between them.  

\paragraph{Competitive Dynamics} 
Fig~\ref{fig:r1_ex1_pol_rfunc3} shows the final policy learned for the setup of Fig~\ref{fig:r1_ex1_rfunc3}.
Observe that the learned joint policy here is, in most states, for each agent to bid identical amounts, leaving it to the VCG mechanism to randomly determine who wins. In some cases an agent will bid an amount greater than $\rho$, such as in tranches 3-5 of the sampled final policy evaluation shown in Fig~\ref{fig:r1_ex1_pol_rfunc3}. Why? Consider tranche 3 of the sample, where the game is in state (6000, 0). Here agent $R_{1}$ bid $1.9$m -- the same as $R_{2}$ -- and won, receiving a profit (and reward) of $-0.9$m, so agent $R_{2}$ received a reward of $0.9$m. If agent $R_{1}$ had bid less, say $1.8$m, $R_{2}$ would have won with a profit of $6.1 - 1.8 = 4.3$m, and agent $R_{1}$ would have received a reward of $-4.3$m instead of $-0.9$m. Clearly that is a worse outcome. If agent $R_{1}$ had bid more, the outcome would have also been worse: agent $R_{1}$ would have received a larger negative reward (and thus a larger positive reward would have gone to $R_{2}$). 

Now let's suppose the tie-break at state (6000, 0) -- and all subsequent tie-breaks -- are resolved in $R_{2}$'s favour. Assuming each agent stays with its learned strategies, the last three tranches would play out as follows: 

\begin{itemize}\itemsep1mm\parskip0mm
    \item $t = 3$:  (6000, 0) $\to$ bid $1.9$ / $1.9$ $\to$ $R_{2}$ wins $\to$ reward $-4.2$ / $+4.2$
    \item $t = 4$:  (6000, 3000) $\to$ bid $1.25$ / $1.25$ $\to$ $R_{2}$ wins $\to$ reward $-1.1$ / $+1.1$
    \item $t = 5$:  (6000, 6000) $\to$ bid $1.0$ / $1.0$ => $R_{2}$ wins $\to$ reward $-0.0$ / $+0.0$ 
\end{itemize}

The total reward over the last three tranches is thus -5.3 to $R_{1}$ and +5.3 to $R_{2}$. Overall, agent $R_{1}$ has accurately estimated its expected reward from state (6000, 0) over the remaining three tranches, when playing its estimated optimal actions, as $-5.3$m. See cell (3, 1) of the ``Estimated Returns'' matrix in Fig~\ref{fig:r1_ex1_pol_rfunc3}.

\begin{figure}[!htbp]
\begin{scriptsize}
\begin{verbatim}
=== FINAL POLICY EVALUATION ===
--------------------------------------------------------------------
|  t  | Agent | Prod |  Perm  | Prof | Rho  | Bid  |  Win? | +Prof |
--------------------------------------------------------------------
|  1  |   1   |  0.0 |      0 |  0.0 |  8.4 | 1.55 |   *   |   6.9 |
|     |   2   |  0.0 |      0 |  0.0 |  6.1 | 1.55 |       |       |
--------------------------------------------------------------------
|  2  |   1   | 42.2 |  3,000 |  6.9 |  1.6 | 1.55 |   *   |   0.1 |
|     |   2   |  0.0 |      0 |  0.0 |  6.1 | 1.50 |       |       |
--------------------------------------------------------------------
|  3  |   1   | 50.2 |  6,000 |  7.0 |  1.0 | 1.90 |   *   |  -0.9 |
|     |   2   |  0.0 |      0 |  0.0 |  6.1 | 1.90 |       |       |
--------------------------------------------------------------------
|  4  |   1   | 55.2 |  9,000 |  6.1 |  0.8 | 2.50 |   *   |  -1.7 |
|     |   2   |  0.0 |      0 |  0.0 |  6.1 | 2.50 |       |       |
--------------------------------------------------------------------
|  5  |   1   | 59.0 | 12,000 |  4.4 |  0.6 | 3.40 |   *   |  -2.8 |
|     |   2   |  0.0 |      0 |  0.0 |  6.1 | 3.40 |       |       |
--------------------------------------------------------------------
(winning agent must pay less than rho to make a profit)

Avg permit price: $723.0
Agent 1:
  -> won 15000 permits
  -> paid $10.85m
  -> produced 62.2m litres of fuel
  -> made a total profit of $1.58m

Agent 2:
  -> won 0 permits
  -> paid $0.00m
  -> produced 0.0m litres of fuel
  -> made a total profit of $0.00m

====== ESTIMATED OPTIMAL STRATEGIES: argmax_a Q(s, a) ======

==============================================================================================
        ||      0      |    3,000    |    6,000    |    9,000    |    12,000   |    15,000   |
==============================================================================================
      0 || 1.55 / 1.55 | 1.55 / 1.55 | 1.95 / 1.95 | 2.85 / 2.85 | 4.50 / 4.50 | 0.0* / 0.0* |
  3,000 || 1.55 / 1.50 | 1.20 / 1.15 | 1.05 / 1.05 | 1.15 / 1.15 | 0.0* / 0.0* |   ?  /  ?   |
  6,000 || 1.90 / 1.90 | 1.25 / 1.25 | 1.00 / 1.00 | 0.0* / 0.0* |   ?  /  ?   |  ?   /  ?   |
  9,000 || 2.50 / 2.50 | 1.55 / 1.6* | 0.0* / 0.0* |   ?  /  ?   |  ?   /  ?   |  ?   /  ?   |
 12,000 || 3.40 / 3.40 | 0.0* / 0.0* |   ?  /  ?   |  ?   /  ?   |  ?   /  ?   |  ?   /  ?   |
 15,000 || 0.0* / 0.0* |   ?  /  ?   |  ?   /  ?   |  ?   /  ?   |  ?   /  ?   |  ?   /  ?   |

====== ESTIMATED RETURNS IF OPTIMAL STRATEGY FOLLOWED: max Q(s, a) ======

===================================================================================================
        ||      0       |    3,000     |    6,000     |    9,000     |    12,000    |   15,000    |
===================================================================================================
      0 || 1.72 / -1.72 | 6.30 / -6.30 | 7.10 / -7.10 | 6.10 / -6.10 | 3.95 / -3.95 | 0.00 / 0.00 |
  3,000 || -5.20 / 5.20 | -0.61 / 0.61 | 0.58 / -0.58 | 0.47 / -0.47 | 0.00 / 0.00  |  ?   /  ?   |
  6,000 || -5.30 / 5.30 | -1.05 / 1.05 | 0.04 / -0.04 | 0.00 / 0.00  |  ?   /  ?    |  ?   /  ?   |
  9,000 || -4.44 / 4.44 | -0.80 / 0.80 | 0.00 / 0.00  |  ?   /  ?    |  ?   /  ?    |  ?   /  ?   |
 12,000 || -2.74 / 2.74 | 0.00 / 0.00  |  ?   /  ?    |  ?   /  ?    |  ?   /  ?    |  ?   /  ?   |
 15,000 || 0.00 / 0.00  |  ?   /  ?    |  ?   /  ?    |  ?   /  ?    |  ?   /  ?    |  ?   /  ?   |
\end{verbatim}
\end{scriptsize}
    \caption{Final policy for a sample run of $r^{(3)}$. }
    \label{fig:r1_ex1_pol_rfunc3}
\end{figure}

\stoptoc

\end{document}